%% file: iclr2027_conference.tex
\documentclass{article} 
\usepackage{iclr2027_conference,times}
\iclrfinalcopy
\makeatletter
\def\@noticestring{Preprint. Under review.}
\makeatother

\input{math_commands.tex}

\usepackage{makecell}
\usepackage{hyperref}
\usepackage{url}
\usepackage{booktabs}
\usepackage{amsmath,amssymb,amsfonts,amsthm}
\usepackage{bm}
\usepackage{booktabs}
\usepackage{microtype}
\usepackage{longtable}
\usepackage{graphicx}
\usepackage{subcaption}
\usepackage{booktabs}
\usepackage{comment}
\usepackage{wrapfig}

\newtheorem{theorem}{Theorem}
\newtheorem{proposition}[theorem]{Proposition}
\newtheorem{lemma}[theorem]{Lemma}
\newtheorem{corollary}[theorem]{Corollary}
\newtheorem{definition}[theorem]{Definition}
\newtheorem{remark}[theorem]{Remark}

\usepackage{multirow}
\usepackage{longtable}
\usepackage{booktabs}
\usepackage{amssymb} 
\usepackage{capt-of}
\usepackage{tabularx}

\newcommand{\cM}{\mathcal{M}}
\newcommand{\cU}{\mathcal{U}}
\newcommand{\cW}{\mathcal{W}}
\newcommand{\cV}{\mathcal{V}}

\newcommand{\cL}{\mathcal{L}}
\newcommand{\dist}{\operatorname{dist}}

\hypersetup{
  hidelinks,
  colorlinks=false,
  pdfborder={0 0 0},
  pdftitle={When Does Unsupervised Learning Succeed or Fail?
  A PoS Perspective on Reconstruction-Based Anomaly Detection},
  pdfauthor={Mehmet Yamac, Yagmur Mustu, Muhammad Numan Yousaf,
  Lei Xu, Marcel van Gerven}
}

\title{When Does Unsupervised Learning Succeed or Fail? A PoS Perspective on Reconstruction-Based Anomaly Detection}

\author{
Mehmet Yama\c{c}$^{1}$,
Yagmur Mustu$^{1}$,
Muhammad Numan Yousaf$^{1}$,
Lei Xu$^{1}$,
Marcel van Gerven$^{2}$\\[0.6em]
$^{1}$Tampere University, Tampere, Finland\\
$^{2}$Radboud University, Nijmegen, The Netherlands
}

\begin{document}

\maketitle

\begin{abstract}
Reconstruction-based unsupervised learning can fail in two opposing ways: a model may reconstruct anomalies too accurately or discard valid nominal variation. Using the Pursuit of Subspaces hypothesis, we characterize these failures through the meet, union, and join geometries induced by the nominal components. Excess learned range produces join blindness, while insufficient capacity produces meet preference and loss of nominal fidelity. We show that the compact nominal union is optimal among nominal faithful ranges and generally requires a nonlinear reconstruction map. Based on this geometry, we introduce Dynamic Push and Pull, which learns from controlled perturbations without anomaly labels, and nested manifold carving, which applies the same principle recursively in latent space. Experiments confirm the predicted changes in latent geometry across every tested Push and Pull configuration. The proposed methods improve reconstruction-based anomaly detection across standard benchmarks and unseen image degradations, while also improving pretrained ECG representations for downstream classification. These results connect reconstruction failures to identifiable geometric conditions and provide practical mechanisms for learning compact representations.

\end{abstract}

\section{Introduction}

The recently proposed Pursuit of Subspaces (PoS) hypothesis provides a
postulate-based geometric foundation for representation learning
\citep{pos}. It formalizes recurring properties of successful networks through
four idealized and falsifiable postulates. \textbf{P1: Geometric compaction} states that ideal learning refines
an \(m\)-dimensional candidate data manifold \(\cM_D\subset\R^n\) into a
structured target \(\cM:=\bigcup_{i=1}^{M}\cM_{D_i}\subseteq\cM_D\), where
\(\dim(\cM_{D_i})=k_i<m\). \textbf{P2: Nonlinear orthogonal projection}
specifies the ideal learned operation: within an encoder-decoder
encapsulation, the network acts as
\(P_{\cM}:=f_D\circ f_E:\R^n\rightarrow\cM\), producing
\(\hat{s}=P_{\cM}(s)\in\cM\) and satisfying
\(P_{\cM}^{2}=P_{\cM}\); hence \(P_{\cM}(\hat{s})=\hat{s}\).
\textbf{P3: Normal-space removal} realizes this projection as
\(\hat{s}=P_{\cM}(s)=s-P_{\cM}^{\perp}(s)\), where
\(P_{\cM}^{\perp}(s)\in N_{\hat{s}}\cM_{D_i}\) and
\(P_{\cM}=I-P_{\cM}^{\perp}\). Thus, the identity path preserves \(s\), while
the learned branch estimates the normal-space component to be removed;
\emph{normal space} denotes the geometric object and \emph{residual branch}
its architectural realization. \textbf{P4: Nested projections through depth}
reapplies P1-P3 at every level: the representation at level \(\ell-1\)
becomes the input to
\(P^{(\ell)}:=f_D^{(\ell)}\circ f_E^{(\ell)}
:\R^{n_{\ell-1}}\rightarrow\cM^{(\ell)}\), where
\(\cM^{(\ell)}:=\bigcup_{i=1}^{M_\ell}\cM_{D_i}^{(\ell)}\).
Depth therefore realizes a hierarchy of nested, progressively refined
nonlinear projections.

Remarks~1--3 of PoS~\cite{pos} show that, once the network has learned
\(P_{\cM}\) as a nonlinear orthogonal projector onto a union of submanifolds,
it can be extended while preserving this property. Provided that the added
transformation \(\Phi\) is learned as an isometry,
\(\Phi\circ P_{\cM}\circ\Phi^{-1}\) remains a nonlinear orthogonal projector
onto \(\Phi(\cM)\). These results explain how an established PoS projection can
be preserved during network growth, but not how the network can be encouraged
to learn \(P_{\cM}\) as such a projection in the first place. This paper
addresses precisely that open problem:

\begin{quote}
\textit{Which learning mechanisms encourage neural modules to approximate
nonlinear orthogonal projection onto compact unions of submanifolds?}
\end{quote}

The PoS postulates enable a falsifiable predict-then-measure approach: they
specify both the target representation geometry and how representations should
evolve as compactness increases. We derive two failure predictions and two
corresponding remedies, then test the predicted dynamics at the level of
individual samples and across configurations. We
distinguish properties enforced directly by the proposed objectives from
geometric conclusions that additionally assume the ideal projection in P2. Let \(\cM_D\) contain the component submanifolds
\(\{\cM_{D_i}\}_{i=1}^{M}\), and locally represent them in common coordinates
by tangent subspaces \(S_i\). Define their \emph{meet}
\(\cW:=\bigcap_i S_i\), local nominal \emph{union}
\(\cU:=\bigcup_i S_i\), and \emph{join}
\(\cV:=\sum_i S_i=\operatorname{span}(\bigcup_i S_i)\), the smallest linear
subspace containing every \(S_i\). Since
\(\cW\subseteq\cU\subseteq\cV\), set inclusion reverses distance:
\begin{equation}
    \dist(s,\cV)\leq\dist(s,\cU)\leq\dist(s,\cW).
    \label{eq:score-sandwich-main}
\end{equation}
This relation predicts two limiting failures: a learned range approaching the
join can reduce anomaly scores and cause false negatives, whereas one
approaching the meet assigns positive scores to nominal points in
\(\cU\setminus\cW\), causing false positives. These distances coincide with
reconstruction errors only under the nearest-point projection prescribed by
P2; we do not assume that an arbitrary learned range lies between \(\cW\) and
\(\cV\). Formal definitions and proofs are given in
Appendix~\ref{app:union-interval}.

\begin{figure}[htbp]
    \centering
    \begin{subfigure}[b]{0.23\textwidth}
        \centering
        \includegraphics[trim=12pt 10pt 30pt 10pt, clip, width=\textwidth]{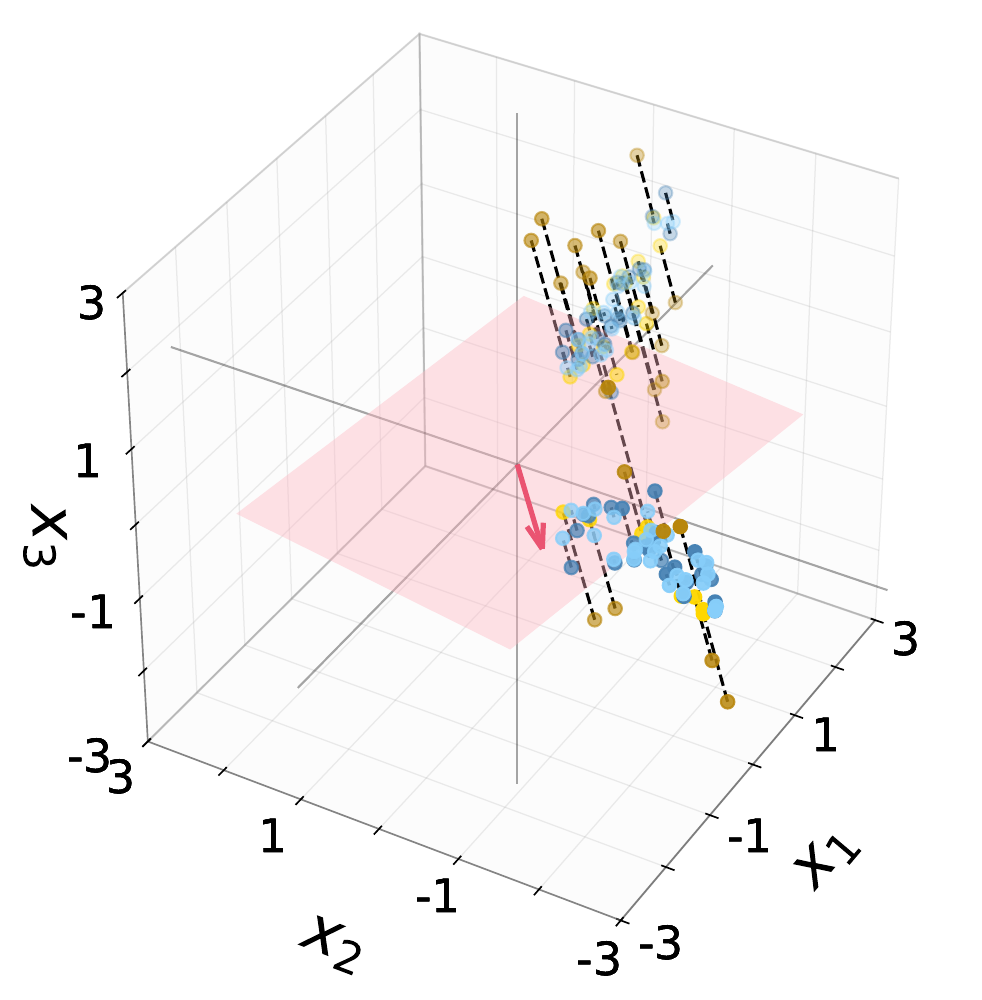}
        \caption{}
        \label{fig:pdf1}
    \end{subfigure}
    \hfill
    \begin{subfigure}[b]{0.23\textwidth}
        \centering
        \includegraphics[trim=12pt 10pt 30pt 10pt, clip, width=\textwidth]{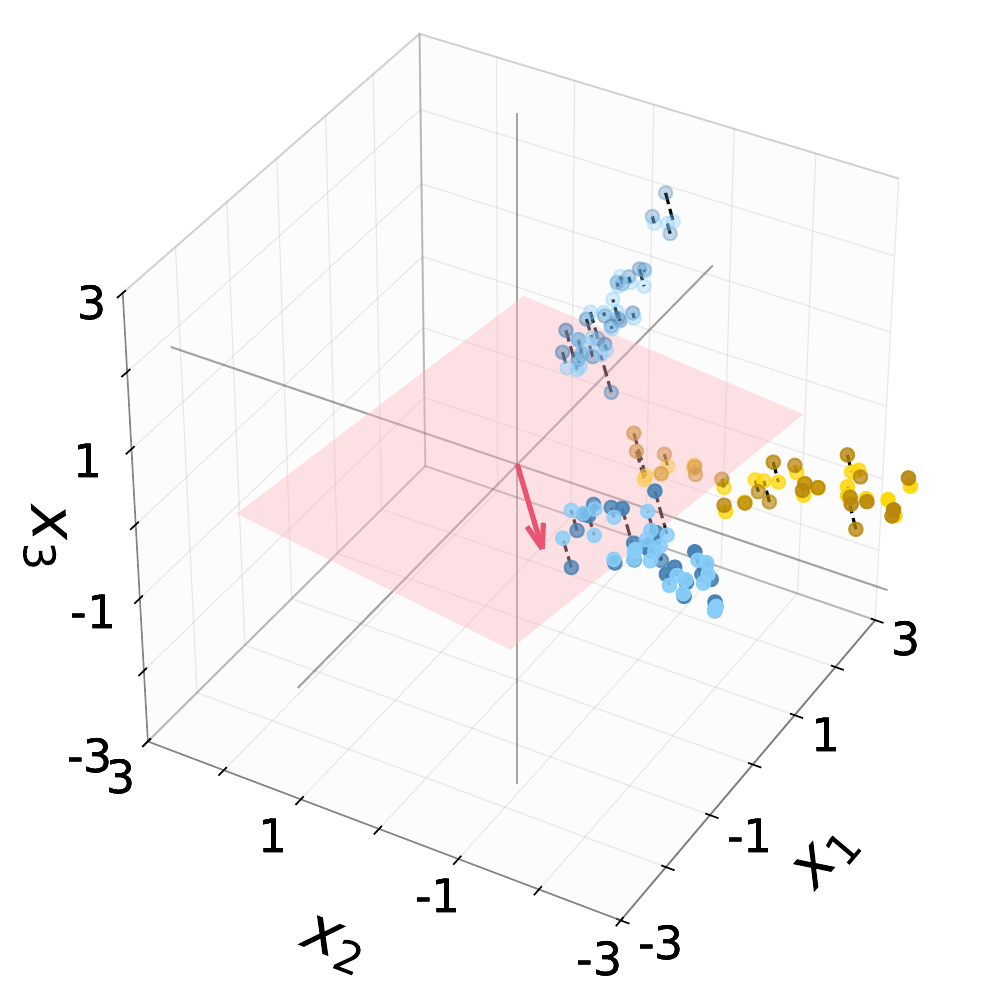}
        \caption{}
        \label{fig:pdf2}
    \end{subfigure}
    \hfill
    \begin{subfigure}[b]{0.23\textwidth}
        \centering
        \includegraphics[trim=12pt 10pt 30pt 10pt, clip, width=\textwidth]{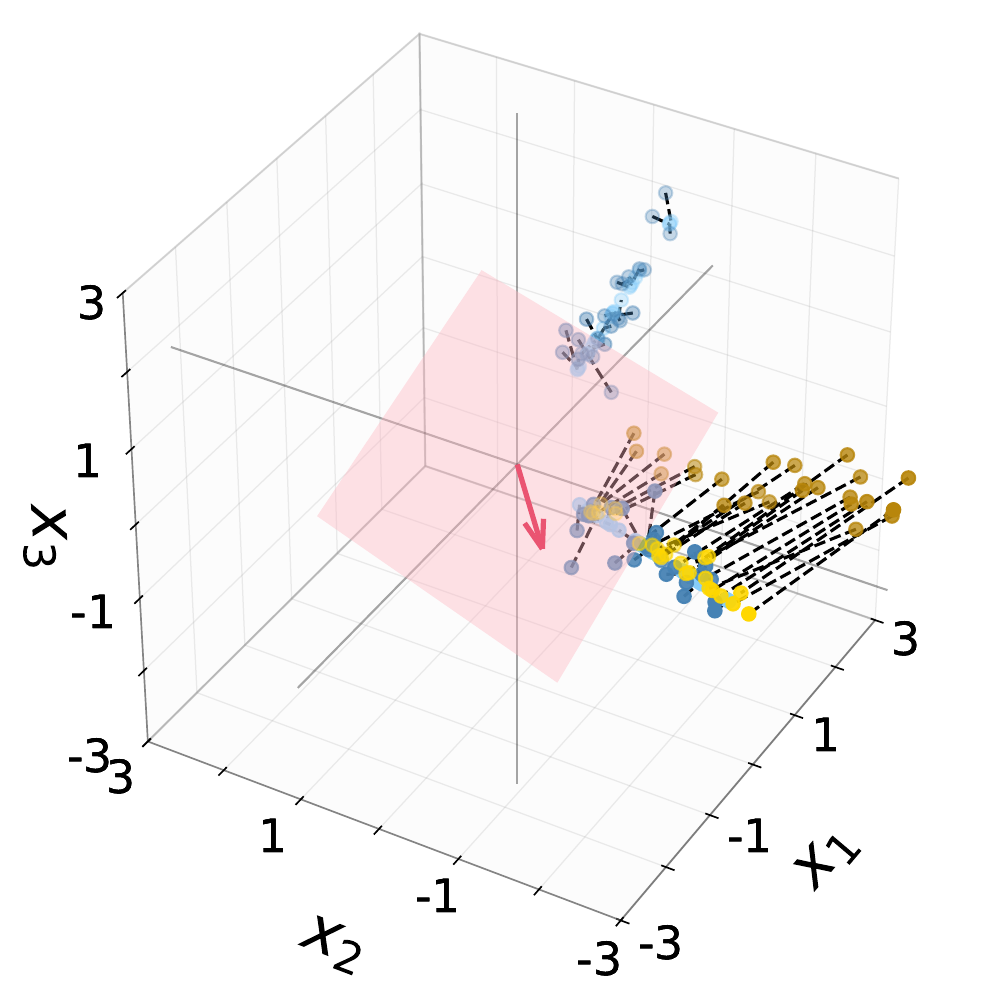}
        \caption{}
        \label{fig:pdf3}
    \end{subfigure}
    \hfill
    \begin{subfigure}[b]{0.23\textwidth}
        \centering
        \includegraphics[trim=12pt 10pt 30pt 10pt, clip, width=\textwidth]{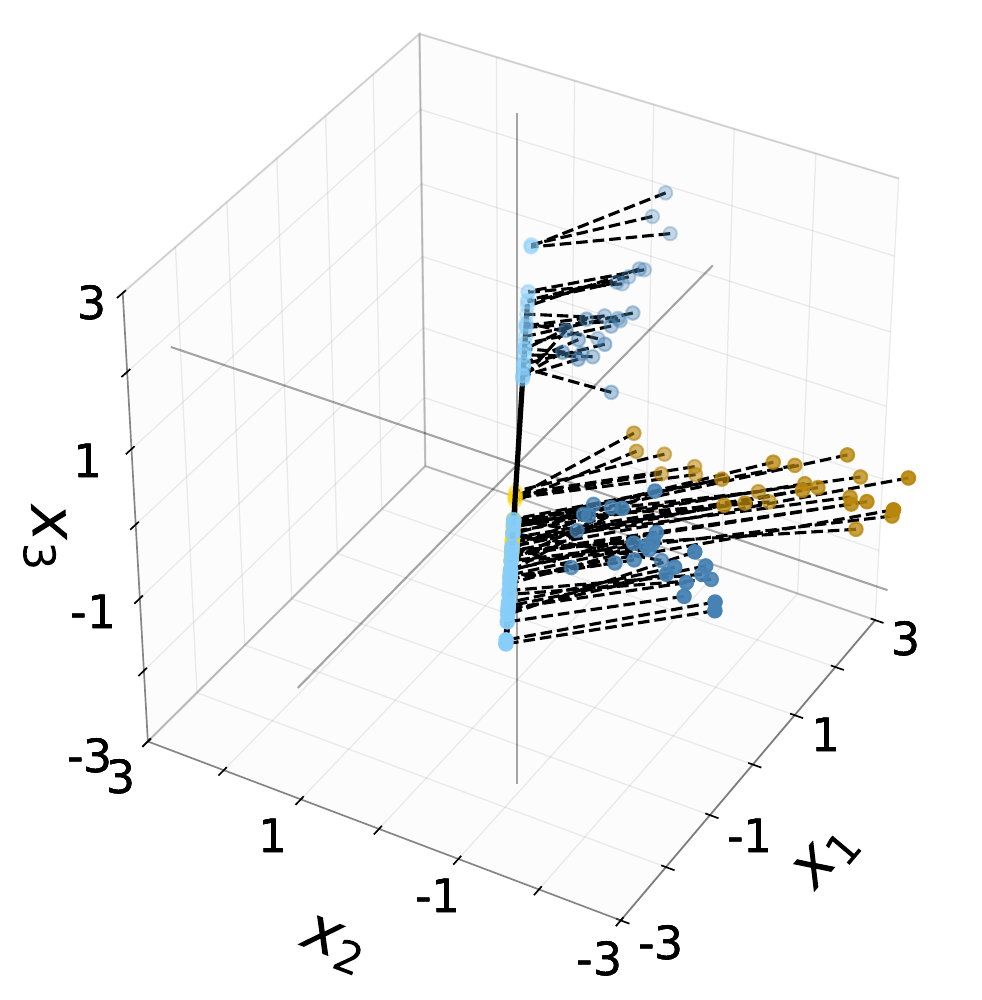}
        \caption{}
        \label{fig:pdf4}
    \end{subfigure}
    \vspace{2pt}

    \noindent
    \makebox[\linewidth][l]{%
        \includegraphics[trim=5pt 10pt 20pt 5pt, clip, width=\linewidth]{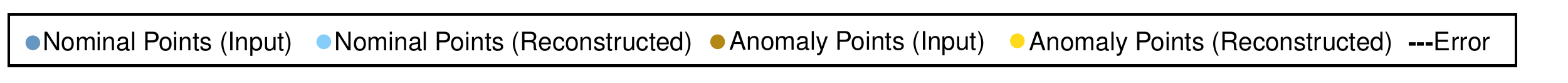}%
    }
    \caption{\textbf{Toy geometry of successful and failed anomaly detection.}
    The shaded plane denotes the join \(\cV\) of the two nominal subspaces;
    dark blue and ochre denote nominal and anomalous inputs, light blue and
    yellow their reconstructions, and dashed segments reconstruction errors.
    \textbf{(a)} A linear network with range \(\cV\) detects anomalies outside
    the join through their nonzero reconstruction errors.
    \textbf{(b)} The same range is blind to anomalies in
    \(\cV\setminus\cU\), which can be reconstructed as accurately as nominal
    data.
    \textbf{(c)} A nonlinear network learns the compact union \(\cU\), preserving
    the nominal components while producing nonzero errors for anomalies both
    inside and outside \(\cV\).
    \textbf{(d)} Rank deficient learning produces a low dimensional compromise range
\(T_{\mathrm{under}}\subset\cV\) that poorly represents both nominal and
anomalous data, inflating both scores and reducing their fidelity to nominal
union membership.}
    \label{fig:toy-failure-geometry}
\end{figure}

\noindent\textbf{PoS Prediction I (join blindness):} A reconstruction objective constrains the model on nominal samples but does not
bound its learned range outside them. The range may therefore expand toward the
join \(\cV\), violating the compaction required by P1. As
Figure~\ref{fig:toy-failure-geometry}(a) illustrates, anomalies outside \(\cV\)
remain detectable through reconstruction error. Under P2, however, an anomaly
\(s_a\in\cV\setminus\cU\) receives zero reconstruction error
(Figure~\ref{fig:toy-failure-geometry}(b)), while even an approximate join can
substantially reduce its score. Join blindness is therefore a loss of
\emph{selectivity}: the model retains directions that should have been
excluded.

\noindent\textbf{PoS Prediction II (meet preference):}
Let \(\cW:=\bigcap_i S_i\) denote the directions shared by all nominal
component tangent spaces. Limited capacity may preferentially retain \(\cW\)
while discarding component specific directions; we call the limiting case
\(T=\cW\) \emph{meet collapse}. Proposition~\ref{prop:balanced-meet}
establishes sufficient conditions under which a rank constrained MSE optimal
linear autoencoder exhibits this preference. Through local tangent
approximations, the same mechanism applies to nonlinear geometry. Under P2,
meet collapse assigns positive errors to nominal points in
\(\cU\setminus\cW\), causing a loss of \emph{fidelity}.
Figure~\ref{fig:toy-failure-geometry}(d) visualizes this consequence through
rank deficient underlearning; its meet is trivial, so it illustrates the
fidelity failure rather than exact meet collapse.

\noindent\textbf{Why the compact union requires nonlinearity:}
P1 identifies the desired learned range as the compact nominal union
\(\cU\). Lemma~\ref{lem:union-optimality} shows that \(\cU\) is optimal among nominal-faithful ranges,
while Proposition~\ref{prop:linear-noncompact} shows that a linear reconstruction map preserving
\(\cU\) must also contain its join \(\cV\). Reconstruction alone does not
constrain the learned range toward \(\cU\): excess capacity can produce join
blindness, whereas restricted capacity can discard component-specific
directions and produce meet preference. Learning the compact union therefore
requires both nonlinearity and an explicit separation signal. Dynamic Push
and Pull supplies this missing signal.
Figure~\ref{fig:toy-failure-geometry}(c)
illustrates the desired case: a nonlinear network learns the compact union,
preserving nominal reconstruction while detecting anomalies both inside and
outside the join. Moreover, any geometrically nominal faithful range
\(T\supseteq\cU\) satisfies
\(
\dist(s,T)\leq\dist(s,\cU),
\)
where \(\cU\) maximizes ideal pointwise separation and, under nominal support
on \(\cU\), expected geometric score gap among such ranges. Formal statements
and proofs are given in Appendices~\ref{app:meet},
\ref{app:linear-noncompact}, and~\ref{app:union-optimality}.

\noindent\textbf{Remedy I (dynamic Push and Pull learning):} P1 to P3 motivate a training objective based on controlled perturbations of
nominal samples, without requiring anomaly examples or class labels.
Reconstructions are pulled toward the corresponding clean samples and pushed
away from their perturbed versions, with a margin adapted to perturbation
severity. This discourages the excess range responsible for join blindness,
as detailed in
(Section~\ref{sec:dynamic-push-pull}).

\noindent\textbf{Remedy II (nested manifold carving):}
P4 motivates applying this compaction recursively in latent space. Each nested
stage is trained to preserve nominal information while removing perturbation
directions, while sufficient latent capacity reduces premature contraction
toward shared directions (Section~\ref{sec:nested-carving}).

Although anomaly detection exposes these failure geometries particularly
clearly, the remedies target unsupervised representation learning more
generally. They learn data structure without semantic labels. Reconstruction
error uses this structure to detect departures from the learned data, while
downstream classifiers can reuse the same representation for label efficient
prediction. Anomaly detection is therefore one application of the learned
geometry, not the scope of the method. 

We evaluate this progression in three stages. First, we compare our strategies
with state of the art unsupervised anomaly detection methods on standard
benchmarks. Second, we introduce a real world nominal only task: learning from
high quality images and detecting unseen degradations without low quality
training examples. Third, we move beyond anomaly detection by pretraining on
unlabelled ECG signals, freezing the encoder, and training a small classifier
with limited labelled data. We additionally test two PoS predictions for
representation dynamics: relative disentanglement of nominal components and
projection of samples outside the union toward the nominal union.

Our contributions are threefold. First, we derive two opposing failure predictions, join blindness and meet preference, and establish their consequences for anomaly scores, the need for nonlinearity, and the optimality of the compact nominal union. Second, we introduce Dynamic Push and Pull and nested manifold carving as complementary PoS-derived strategies for learning compact representations without anomaly examples or class labels. Third, we validate the framework through benchmark anomaly detection, real-world detection of unseen image degradations from high-quality nominal data, and unsupervised ECG pretraining for downstream classification, while testing two predicted dynamics of the learned geometry.
\section{Related Work}
Autoencoders are widely used for unsupervised anomaly detection, where reconstruction error serves as an anomaly score under the assumption that models trained on nominal data reconstruct anomalous inputs less accurately \citep{sakurada2014anomaly, dsebm, zhou2017anomaly, zhao2017spatio}. However, this assumption is not generally reliable, since autoencoders may
also reconstruct unseen anomalies with small error
\citep{dagmm, memae, zaheer2020old}. Existing approaches
regularize the learned representation through latent compactness
\citep{compactAnomaly}, sparse autoencoders \citep{nanaumi2026spg},
denoising objectives \citep{denoisingAnomaly}, or masked reconstruction
\citep{masked1, masked2, masked3}, thereby encouraging the model to capture
stable structure in nominal data rather than learning an unrestricted
identity mapping. Architectural approaches instead constrain the reconstruction process using memory banks of nominal prototypes \citep{memae, memory2}, latent density models \citep{dagmm}, or adversarial and one class components \citep{ocgan}. 

Despite this broad effort, the mechanisms that cause autoencoder based anomaly detection to fail, and the conditions under which it can be reliable, remain insufficiently understood. \citet{bouman2025} show that autoencoders can accurately reconstruct anomalies through interpolation between nominal groups and extrapolation beyond the observed nominal data. We interpret both behaviors as manifestations of an excessively broad learned range and characterize the limiting case as join blindness; we further identify the opposite failure of meet preference, in which nominal directions specific to individual components are lost. Self supervised methods such as CutPaste, DRAEM, NSA, and CSI learn discriminative anomaly criteria from synthetic anomalies or distribution shifted views \citep{li2021cutpaste, zavrtanik2021draem, schluter2022natural, tack2020csi}.
Reconstruction based perturbation based methods train autoencoders using noisy, synthetic, or adversarially perturbed nominal samples \citep{denoisingAnomaly, astrid2021learning, astrid2024exploiting, salehi2021arae}. Methods that reconstruct the clean sample from its perturbation correspond to the Pull component of our formulation, whereas Dynamic Push and Pull additionally imposes an explicit separation margin from the perturbed input to target join blindness. Nested Manifold Carving complements this objective by recursively targeting off union directions retained in latent representations.
\section{PoS-Guided Learning Strategies}
\label{sec:strategies}

\subsection{Dynamic push-pull learning}
\label{sec:dynamic-push-pull}

Let \(d\) be a distance function chosen for the signal modality. For a nominal
observation \(s_n\), we generate a perturbed observation by adding a controlled
perturbation \(e\). Using \([z]_+:=\max\{z,0\}\), define
\begin{equation}
\begin{gathered}
    \widetilde{s}_n:=s_n+e,
    \qquad
    \widehat{s}_n:=P_\theta(\widetilde{s}_n),
    \qquad
    \delta_e:=d(\widetilde{s}_n,s_n),\\
    \cL_{\mathrm{PP}}(\theta;s_n,e)
    :=
    d(\widehat{s}_n,s_n)
    +
    \left[
        \alpha\delta_e
        -
        d(\widehat{s}_n,\widetilde{s}_n)
    \right]_+,
    \qquad
    \alpha>0.
\end{gathered}
\label{eq:push-pull-loss}
\end{equation}
The first term pulls the reconstruction toward the nominal observation, while
the hinge term pushes it away from the perturbed input by the adaptive margin
\(\alpha\delta_e\). In the general noisy setting, \(s_n\) need only be
concentrated near the nominal union. Exact membership is imposed only in the
ideal result below. The experiments do not optimize the push--pull loss in isolation. They retain
the ordinary reconstruction objective on unperturbed nominal observations,
defined by
\begin{equation}
    \cL_{\mathrm{AE}}(\theta)
    :=
    \mathbb E_{s_n}
    \left[
        d\!\left(P_\theta(s_n),s_n\right)
    \right].
    \label{eq:nominal-reconstruction-loss}
\end{equation}
and optimize the complete objective
\begin{equation}
\begin{aligned}
    \cL_{\mathrm{train}}(\theta)
    &:=
    \cL_{\mathrm{AE}}(\theta)
    +
    \lambda_{\mathrm{PP}}\,
    \mathbb{E}_{s_n,e}
    \left[
        \cL_{\mathrm{PP}}(\theta;s_n,e)
    \right] \\
    &\quad
    +
    \lambda_{\mathrm{sp}}\,
    \Omega_{\mathrm{sp}}(\theta),
    \qquad
    \lambda_{\mathrm{PP}}>0,
    \qquad
    \lambda_{\mathrm{sp}}\geq 0,
\end{aligned}
\label{eq:complete-training-loss}
\end{equation}
where, \(\Omega_{\mathrm{sp}}\) is an optional sparsity penalty, with
\(\lambda_{\mathrm{sp}}=0\) when sparsity is not used. Thus,
\(\cL_{\mathrm{AE}}\) anchors the reconstruction of unperturbed nominal
observations, while the push--pull term governs the response to their perturbed
counterparts. The pairwise results below characterize the push-pull component;
the operating point of the shared network is also shaped by the nominal
reconstruction and optional sparsity terms in
Eq.~(\ref{eq:complete-training-loss}).

\begin{figure}[!h]
  \centering
\includegraphics[width=0.38\columnwidth]{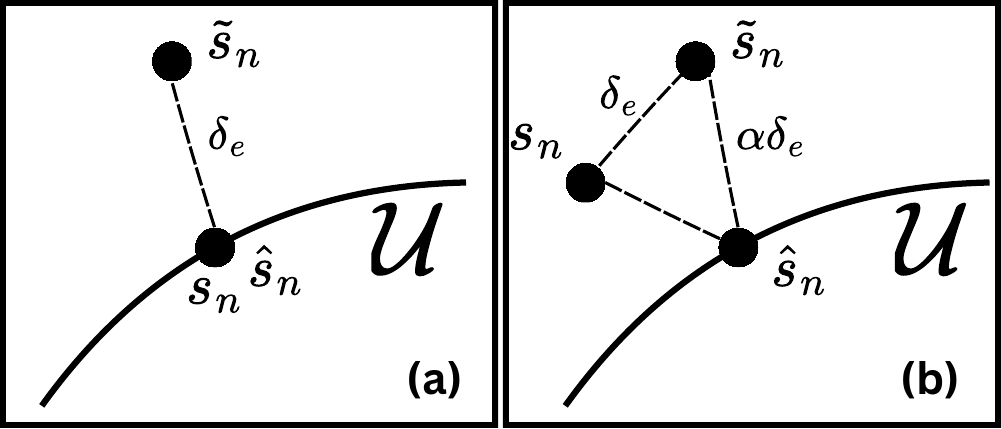}
  \caption{Ideal and realizable regimes of
  Lemma~\ref{lem:pushpull}.
  (a) If \(s_n\in\cU\) and \(e\) is normal to \(\cU\), exact Pull gives
  \(\widehat{s}_n=s_n\). For \(\alpha\leq1\), Push is inactive at this
  optimum but supplies a gradient whenever its margin is violated.
  (b) If \(\dist_d(s_n,\cU)=\eta(s_n)>0\), margins \(\alpha>1\) become
  admissible and impose a constraint unavailable to Pull alone. The
  triangle inequality gives
  \(\alpha\delta_e\leq\delta_e+\eta(s_n)\), recovering the zero-push-error
  case of Eq.~(\ref{eq:alpha-compatibility}).}
  \label{fig:push-pull-geometry}
\end{figure}
\begin{lemma}[Push and pull guarantees]
\label{lem:pushpull}
Let \(d\) be a metric. Write Eq.~(\ref{eq:push-pull-loss}) as
\(\cL_{\mathrm{PP}}=\cL_{\mathrm{pull}}+\cL_{\mathrm{push}}\), and define
\(T:=\operatorname{range}(P_\theta)\),
\(\rho_{\theta,d}(x):=d(x,P_\theta(x))\), and
\(\dist_d(x,T):=\inf_{t\in T}d(x,t)\).
Suppose
\(\eta(s_n):=\dist_d(s_n,\cU)\leq\varepsilon\), where
\(\eta(s_n)=0\) when \(s_n\in\cU\).
If
\(\cL_{\mathrm{pull}}\leq\varepsilon_{\mathrm{pull}}\) and
\(\cL_{\mathrm{push}}\leq\varepsilon_{\mathrm{push}}\), then
\(\dist_d(s_n,T)\leq\varepsilon_{\mathrm{pull}}\) and
\(\rho_{\theta,d}(\widetilde{s}_n)
\geq\alpha\delta_e-\varepsilon_{\mathrm{push}}\).
Under P2, these bounds jointly give
\begin{equation}
\dist_d(\widetilde{s}_n,T)\geq
\max\!\left\{
\delta_e-\varepsilon_{\mathrm{pull}},
\alpha\delta_e-\varepsilon_{\mathrm{push}},
0
\right\}.
\label{eq:push-pull-carving}
\end{equation}
If, in addition, \(T\) provides exact nominal coverage
\(\cU\subseteq T\), the triangle inequality gives
\(\dist_d(\widetilde{s}_n,T)\leq\delta_e+\eta(s_n)\).
Consequently, for \(\delta_e>0\),
\begin{equation}
\alpha\leq
1+\frac{\eta(s_n)+\varepsilon_{\mathrm{push}}}{\delta_e}.
\label{eq:alpha-compatibility}
\end{equation}
At zero push error, this bound permits \(\alpha>1\) precisely when
\(\eta(s_n)>0\).
\end{lemma}

The proof is given in Appendix~\ref{app:push-pull-proof}. The guarantee is local to the sampled perturbations. Suppressing join
blindness throughout \(\cV\setminus\cU\) therefore requires the perturbation
sampler to probe this region sufficiently; otherwise, only the sampled
neighborhoods are certifiably carved. Perturbations with components in
\(S_i^\perp\cap\cV\) directly probe directions inside the join but outside the
active nominal component.

\subsection{Nested manifold carving}
\label{sec:nested-carving}

Following P4, we apply push--pull recursively in latent space. At level
\(\ell\), the previously learned encoder--decoder pairs are frozen, their
encoder stack maps \(s_n\) to \(x_n^{(\ell-1)}\), and a new pair
\((E_\ell,D_\ell)\) is trained on this representation using the push--pull
objective. The resulting reconstruction is then decoded through all preceding
decoders. Each level therefore carves the representation inherited from the
level above, while the pull term preserves nominal information and the push
term prevents the newly added pair from simply reproducing perturbed latent
inputs. The latent width is chosen sufficiently large to retain
component-specific nominal directions and avoid premature meet preference. Let \(\overline T^{(\ell)}\) denote the range of the complete reconstruction
map through level \(\ell\), decoded back to the input space. Under the
compatibility, capacity, and exact-optimization conditions stated in
Appendix~\ref{app:nested-carving}, the construction yields
\begin{equation}
\overline T^{(1)}
\supseteq\overline T^{(2)}
\supseteq\cdots
\supseteq\overline T^{(L)}
\supseteq\cU.
\label{eq:nested-ranges-main}
\end{equation}
Thus, successive levels remove probed off-union regions without sacrificing
nominal coverage. Eq.~(\ref{eq:nested-ranges-main}) describes the ideal
construction target; its capacity, nonlinearity, and perturbation-coverage
conditions are detailed in Appendix~\ref{app:nested-carving}, with the linear
obstruction given in Appendix~\ref{app:linear-noncompact}.

\subsection{PoS-predicted representation dynamics}
\label{sec:geometric-predictions}

PoS predicts how latent geometry changes as a network approaches the ideal
behavior defined by P1 to P4. Let \(x\) and \(\widehat{x}\) denote,
respectively, less compact and more compact representations of the same
samples, where compactness refers to agreement with the PoS postulates rather
than merely smaller variance. From \(x\) to \(\widehat{x}\), P1 and P4 predict increasing relative
separation of nominal components, meaning that their radii contract more than
the distances between their centers. P2 to P4 further predict that off-union
anomalies move toward nominal components. Consequently, although nominal
structure becomes more compact and separated, nominal and anomalous
representations can become less distinguishable. A more compact pretrained
network can therefore provide worse features for a downstream classifier that
must distinguish anomalies from nominal samples. We test these predictions in Section~\ref{sec:representation-dynamics}.
The exact metrics and their computation are provided in
Appendix~\ref{app:Metrics}.

\section{Experiments}
\label{sec:experiments}
\begin{figure*}[t]
    \centering    \includegraphics[width=\textwidth]{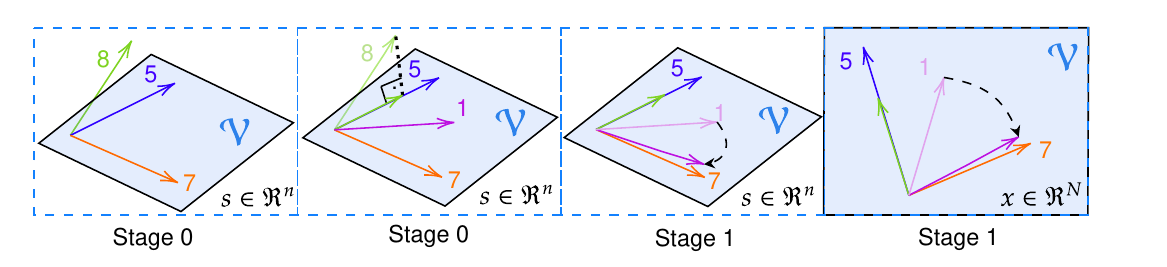}
    \caption{\textbf{Geometric interpretation of the two-stage MNIST
    experiment.}
    At Stage 0, the nominal digits \(4\), \(5\), and \(7\) define a compact
    union \(\cU\) within a coarser learned range \(\cV\). An anomaly such as
    \(8\), lying outside \(\cV\), produces a large Stage-0 reconstruction
    error. By contrast, digit \(1\) can lie inside \(\cV\setminus\cU\) and
    therefore remain undetected at Stage 0. Stage 1 operates in the frozen
    Stage-0 latent space and further carves the retained range, projecting
    such off-union representations toward a nominal component. Anomalous
    information already removed by Stage 0 is unavailable to Stage 1.}
    \label{fig:mnist-stage-geometry}
\end{figure*}

\subsection{A controlled study of the predicted failure modes}
\label{sec:mnist-controlled}

We first use MNIST to examine the complete geometric picture in a controlled
setting. Following~\cite{bouman2025}, digits \(4\), \(5\),
and \(7\) are treated as nominal, while the remaining digits are anomalies.
This setting is particularly informative because the nominal distribution is
multimodal and different anomalous digits interact differently with its learned
geometry.

\noindent\textbf{Baseline and controlled modifications:}
We reproduce the convolutional autoencoder of~\cite{bouman2025} for fair comparison, including its two-dimensional bottleneck and
evaluation protocol. We then retain the same autoencoder backbone but increase
the Stage-0 latent dimension from \(2\) to \(9\), reducing the capacity
restriction that can merge component-specific nominal directions. Starting
from this higher-capacity model, we evaluate the individual and combined
effects of nested carving, dynamic push-pull, masking, and sparsity. For the
nested models, a second encoder--decoder pair is trained in the frozen Stage-0
latent space, whose dimension is expanded to \(60\) before carving. This
controlled construction separates the effect of latent capacity from the
effects of the proposed learning strategies.

\noindent\textbf{Why different levels detect different anomalies:}
Figure~\ref{fig:mnist-stage-geometry} summarizes the mechanism investigated in
this experiment. At Stage 0, anomalies lying outside the learned coarse range
produce large reconstruction errors and are directly detectable. However, an
anomalous digit such as \(1\) may lie close to the range spanning the nominal
digits \(4\), \(5\), and \(7\). It can therefore be reconstructed with a small
Stage-0 error despite not belonging to any nominal component. This is the predicted
join-blindness mechanism. Nested carving addresses this ambiguity in the latent space. After the Stage-0
encoder and decoder are frozen, the second stage is trained to preserve the
latent representations of nominal digits while rejecting their controlled
perturbations. It can therefore remove off-union latent regions that Stage 0
retained and assign a larger error to anomalies such as \(1\). The converse situation is equally important. Some anomalous inputs are already
mapped by Stage 0 onto, or very close to, one of the nominal components. Their
anomalous information is then removed before entering the next latent stage,
so Stage 1 cannot be expected to recover it. These anomalies remain detectable
through the Stage-0 reconstruction error. Stage 0 and Stage 1 are therefore
complementary: the former detects departures removed by the first projection,
whereas the latter detects off-union structure that survives within the
Stage-0 representation.

\noindent\textbf{DAE versus AE-PP:}
Setting \(\alpha=0\) eliminates the Push hinge and reduces Push and Pull to a denoising autoencoder, while \(\alpha>0\) imposes the additional separation constraint characterized by Lemma~1 and Figure~\ref{fig:push-pull-geometry}. DAE provides only a marginal improvement over AE and gains nothing from nesting, whereas AE-PP improves further at Stage 1, indicating that its gain cannot be explained by denoising alone; complete results are reported in Section~\ref{app:mnist-ablation} and Table~\ref{tab:dae-push-pull-ablation}.

\noindent\textbf{Overall ablation results:}
Table~\ref{tab:mnist-main} summarizes the central configurations, with complete results in Tables~\ref{tab:stage-metrics} and~\ref{tab:full-mnist-ablation}. Increasing Stage 0 capacity improves the reproduced baseline, but adding a reconstruction only nested module provides no further gain because it admits an identity solution. Push and Pull supplies the separation signal required for latent carving, improving both the single stage and nested models, while masking and sparsity provide complementary gains in the strongest configuration. Thus, the improvement arises from explicitly carving the learned range rather than from additional capacity, depth, or perturbation alone.

\begin{table}[t]
\centering
\caption{\textbf{Controlled MNIST study with nominal digits \(4,5,7\).}
The baseline reproduces~\cite{bouman2025} with a
two-dimensional bottleneck. The remaining models use the same backbone with
Stage-0 latent dimension \(9\). Nested configurations report the complete
Stage-1 reconstruction score. PP denotes
push--pull. Complete results are provided in
Tables~\ref{tab:stage-metrics} and~\ref{tab:full-mnist-ablation}.}
\label{tab:mnist-main}
\small
\setlength{\tabcolsep}{4pt}
\begin{tabular}{@{}lccc@{}}
\toprule
Configuration & F1 & ROC-AUC & PR-AUC \\
\midrule
Bouman-Heskes baseline
    & 0.693 & 0.866 & 0.936 \\
Higher-capacity AE
    & 0.780 & 0.893 & 0.962 \\
AE + mask
    & 0.766 & 0.881 & 0.957 \\
AE + PP
    & 0.792 & 0.897 & 0.963 \\
AE + nested + PP
    & 0.798 & 0.909 & 0.967 \\
AE + nested + PP + sparsity
    & 0.808 & 0.928 & 0.972 \\
AE + nested + PP + mask + sparsity
    & \textbf{0.815} & \textbf{0.929} & \textbf{0.973} \\
\bottomrule
\end{tabular}
\end{table}

\noindent\textbf{Per digit geometric analysis:}
The image and latent reconstruction errors provide direct evidence for the geometry predicted in Figure~\ref{fig:mnist-stage-geometry}. Digit \(8\) produces a large image error but a small Stage 1 latent error, supporting the prediction that its anomalous direction lies largely outside the coarse range \(\mathcal{V}\), is removed at Stage 0, and leaves a latent representation close to nominal digit \(5\). Conversely, digit \(1\) produces a small image error but a large latent error, supporting the prediction that it lies within \(\mathcal{V}\setminus\mathcal{U}\) and is therefore retained by Stage 0 but differentiated through nested carving. Its Stage 1 reconstructions also become predominantly \(7\)-like, supporting the predicted representation dynamics that increasing compactness moves off-union anomalies toward a nearby nominal component. The complete analysis, flag decomposition, and reconstructed images are provided in Section~\ref{app:latent-score-analysis}, Table~\ref{tab:flag-decomposition}, and Figure~\ref{fig:stages}.

\subsection{Testing the predicted representation dynamics}
\label{sec:representation-dynamics}

We compare the less compact representation \(x\) with the more compact
representation \(\widehat{x}\) across four normal digit sets, two masking
choices, and two sparsity choices. This gives 16 configurations for each
training objective and 32 configurations in total. As summarized in
Table~\ref{tab:geometric-predictions-main}, all 16 configurations trained with
Push and Pull satisfy both PoS predictions. Across all 16 Push and Pull configurations, nominal relative separation increases and all four anomaly proximity indicators decrease. Ordinary AE training increases relative separation in only \(7/16\) configurations and satisfies all four proximity predictions in \(14/16\). Thus, Push and Pull produces the predicted PoS dynamics consistently across every tested normal set, masking choice, and sparsity choice. Complete results and metric
definitions are provided in Appendix~\ref{app:Metrics}.

\begin{table}[ht]
\centering
\small
\caption{Observed PoS representation dynamics. Each objective includes 16
configurations obtained from four normal digit sets, two masking choices, and
two sparsity choices.}
\label{tab:geometric-predictions-main}
\begin{tabular}{p{0.43\linewidth}cc}
\toprule
\textbf{Predicted change} & \textbf{AE} & \textbf{AE + PP} \\
\midrule
Increased nominal relative separation,
\(\Delta S>0\)
& \(7/16\) & \(\mathbf{16/16}\) \\

All four anomaly proximity indicators decrease
& \(14/16\) & \(\mathbf{16/16}\) \\
\bottomrule
\end{tabular}
\end{table}

\input{tables/arrhythmia.tex}
\input{tables/cifar_summary}
\subsection{Standard One-Class Anomaly Detection}
\label{sec:standard-oneclass}

We evaluate one-class anomaly detection on MNIST, Fashion-MNIST, CIFAR-10, CIFAR-100, and Arrhythmia. Training uses only nominal samples. Following the validation protocol of
TOLL~\citep{toll}, checkpoints are selected on a validation set
containing both nominal and anomalous samples. For a controlled comparison, all methods share the Stage~0 autoencoder architecture and optimization settings of TOLL. Our AE baseline is the reconstruction-only model obtained by setting the TOLL compactness coefficient to \(\beta=0\); it therefore uses the same backbone but not the TOLL compactness objective. AE-PP augments this baseline with Dynamic Push and Pull, whereas AE-PoS additionally applies Nested Manifold Carving to the frozen Stage~0 representation.

\begin{wrapfigure}{r}{0.48\columnwidth}
\vspace{-0.6\baselineskip}
\centering
\includegraphics[
  width=\linewidth,
  trim=0 3.5mm 0 0,
  clip
]{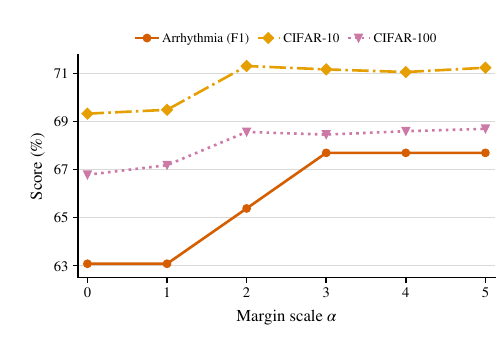}
\caption{Effect of the margin scale \(\alpha\) on Stage-0 AE-PP.
Arrhythmia uses F1; CIFAR-10 and CIFAR-100 use AUC-ROC.}
\label{fig:alpha-sensitivity}
\vspace{-0.8\baselineskip}
\end{wrapfigure}

As shown in Tables~\ref{tab:oneclass-arrhythmia}
and~\ref{tab:oneclass-cifar10-summary}, AE-PoS reaches \(70.77\%\) F1 on Arrhythmia,
compared with \(60.00\%\) for the strongest baseline, and \(71.73\%\)
AUC-ROC on CIFAR-10, compared with \(69.6\%\). Nested Manifold Carving
improves AE-PP from \(67.69\%\) to \(70.77\%\) on Arrhythmia and from
\(71.30\%\) to \(71.73\%\) on CIFAR-10, with gains across all ten CIFAR-10
nominal classes. Together with the controlled multi-component MNIST experiment
in Section~\ref{sec:mnist-controlled}, these results suggest that Stage~1
provides further gains when the Stage-0 representation retains residual multiple components, while largely preserving performance
when Stage~0 is already compact, as in the controlled single-component MNIST
setting and Fashion-MNIST. Complete results for MNIST, Fashion-MNIST, and
CIFAR-100, together with experimental details, baseline definitions,
dataset-specific training strategies, and the AE/AE-PP/AE-PoS ablation, are
provided in Appendix~\ref{apd:oneclass}. The CIFAR-100 results further expose
a limitation of the current reconstruction-error anomaly score, which remains
behind transformation-based methods on several nominal superclasses. Figure~\ref{fig:alpha-sensitivity} shows the predicted saturation: once the Push hinge
remains active for most training pairs, its gradient no longer depends on
\(\alpha\), and performance plateaus at \(\alpha=2\) or \(3\).

\subsection{Detecting Unseen Image Degradations}

\input{tables/hq-degradation-summary}

Our objective is to learn a high-fidelity autoencoder representation of clean, high-quality images that preserves their fine-scale residual components while remaining sensitive to diverse unseen degradations. This contrasts with the VAE, whose probabilistic latent regularization reduces nominal-image fidelity and can preferentially represent low-frequency content, allowing strongly blurred images to be reconstructed more accurately than clean images and thereby reversing the desired anomaly-score ordering. AE-PP and AE-PoS are designed to preserve nominal detail while pushing degraded views away from the learned representation. Table~\ref{tab:hq-degradation-summary} summarizes detection performance across additive noise, blur, rain, and underwater degradations, while Appendix~\ref{apd:hq-full} provides the complete per-setting results, dataset construction, training strategy, clean-reference protocol, and reconstruction-fidelity analysis. As shown in Table~\ref{tab:blur_psnr_main}, the reconstruction PSNR, computed between each input \(s\) and its reconstruction \(\hat{s}\), increases for the VAE from \(27.59\) dB on clean images to \(37.01\) dB under the strongest blur. Thus, blurred anomalies receive lower reconstruction errors than clean inputs, reversing the expected score ordering and explaining the VAE's ROC-AUC below \(0.5\). Complete PSNR results are reported in Table~\ref{tab:reconstruction_psnr_full}. Training perturbations are limited to parametric blur kernels, sensor noise, resampling, and JPEG compression. REDS~\citep{reds} and GoPro~\citep{gopro} therefore test transfer to realistic motion blur constructed through temporal averaging of high-frame-rate video; rain and underwater degradations are absent from training; and Kodak24 Gaussian noise provides an in-family control.

\begin{table}[t]
\centering
\caption{Reconstruction PSNR (dB) under the strongest REDS blur. Higher PSNR
indicates lower reconstruction error. VAE is the only model whose reconstruction
quality improves substantially under blur, reversing the expected anomaly-score
ordering.}
\label{tab:blur_psnr_main}

\small
\setlength{\tabcolsep}{5pt}
\renewcommand{\arraystretch}{1.0}

\begin{tabular}{@{}lrrrr@{}}
\toprule
& AE & VAE & AE-PP & AE-PoS \\
\midrule
Clean
& 63.04 & 27.59 & 43.42 & 43.07 \\
Blur 1.0
& 62.72 & 37.01 & 30.51 & 31.79 \\
\(\Delta\)
& \(-0.32\) & \(\mathbf{+9.42}\) & \(-12.91\) & \(-11.28\) \\
\bottomrule
\end{tabular}
\end{table}

\subsection{Transferable ECG Representations}
\input{tables/ecg_table}
Beyond representation-based anomaly detection, we evaluate nested carving as a general unsupervised representation-learning strategy. Clinical ECG labels require costly expert interpretation, whereas unlabeled recordings are abundant. We therefore learn structure from unlabeled ECGs and train only a linear diagnostic classifier on the frozen representation. Table~\ref{tab:ecg-nested} shows that nested carving improves performance across all five classification tasks; experimental and implementation details are provided in Appendix~\ref{apd:ecg}.

\section{Conclusion and Limitations}

We characterized join blindness and meet preference as opposing failures of reconstruction-based unsupervised learning, and introduced Dynamic Push and Pull and Nested Manifold Carving to better approximate the compact nominal union prescribed by PoS, improving both anomaly detection and transferable representation learning. More importantly, the results demonstrate a predict-then-measure approach: Figure~\ref{fig:mnist-stage-geometry} predicts that anomalies outside the coarse range are removed at Stage~0, those in \(\mathcal{V}\setminus\mathcal{U}\) survive and are exposed by latent carving, and carving moves them toward a nominal component; accordingly, digit~8 shows large reconstruction error but little latent change, whereas digit~1 is detected mainly by Stage~1 and becomes predominantly 7-like. At the population level, PoS predicts how geometry evolves as compactness increases: nominal components disentangle, with their relative distance (inter-centre distance normalized by component radius) increasing, while off-union samples move toward the nominal union. Both changes occur in all 16 Push and Pull configurations. PoS further predicts that increased compactness may weaken downstream discrimination between such samples and nominal data. These representation dynamics offer new design targets for unsupervised learning. Current limitations include metric dependence, manual depth selection, non-unique projections of structured anomalies, and validation of the dynamics only in a controlled multi-component setting, as discussed in Appendix~\ref{apd:limitations}.





\section*{Statements}

\noindent\textbf{AI use:}
We used large language model based generative AI tools to review and polish the
manuscript, including its grammar, clarity, organization, and consistency. All
AI-assisted revisions were critically reviewed and edited by the authors, who
take full responsibility for the final content.

\medskip
\noindent\textbf{Ethics:}
This work uses existing benchmark datasets and involves no new data collection
or interaction with human participants. The clinical ECG datasets are used in
de-identified form under their original licenses and access conditions. The ECG
experiments evaluate representation learning and do not establish clinical
diagnostic validity; clinical deployment would require independent validation
and appropriate regulatory and ethical review.

\medskip
\noindent\textbf{Reproducibility:}
The appendices provide the mathematical assumptions and proofs, architectures,
training objectives, optimization settings, preprocessing steps, perturbation
schemes, and evaluation protocols. Upon acceptance of the conference version,
we will release the source code, pretrained checkpoints, configuration files,
test-set selections, and scripts needed to reproduce the reported tables and
figures. Where dataset licenses prevent redistribution, we will provide file
lists and scripts for reconstructing the evaluation subsets.

\clearpage
\bibliographystyle{iclr2027_conference}
\bibliography{iclr2027_conference}

\clearpage
\appendix
\section*{Appendix}

\input{appendices/Definitions_and_proofs}

\input{appendices/metrics}

\input{appendices/one_class_ablation}

\input{appendices/ECG_Appendix}

\input{appendices/ExtendedLimitations}

\end{document}

%% file: math_commands.tex
\usepackage{amsmath,amsfonts,bm,tikz}

\def\eqref#1{equation~\ref{#1}}

\def\1{\bm{1}}

\DeclareMathAlphabet{\mathsfit}{\encodingdefault}{\sfdefault}{m}{sl}
\SetMathAlphabet{\mathsfit}{bold}{\encodingdefault}{\sfdefault}{bx}{n}

\newcommand{\R}{\mathbb{R}}



%% file: tables/arrhythmia.tex
\begin{table}[h]
\centering
\caption{One-class anomaly detection on
Arrhythmia~\citep{arrhythmia_odds}, measured by F1 score (\%).
AE-PP and AE-PoS results are reported as mean \(\pm\) standard deviation
over 10 seeds. Bold indicates the highest mean.}
\label{tab:oneclass-arrhythmia}

\begingroup
\footnotesize
\setlength{\tabcolsep}{2.5pt}
\renewcommand{\arraystretch}{1.05}

\resizebox{\columnwidth}{!}{%
\begin{tabular}{@{}*{15}{c}@{}}
\toprule
&
\multicolumn{12}{c}{Baselines} &
\multicolumn{2}{c}{Ours} \\
\cmidrule(lr){2-13}
\cmidrule(lr){14-15}
Dataset
& IF
& OC-SVM
& DSEBM-e
& AnoGAN
& DAGMM
& ALAD
& RCGAN
& DSVDD
& GOAD
& MEMAE
& MEMGAN
& TOLL
& AE-PP
& AE-PoS \\
\midrule
Arrhythmia
& 53.03
& 45.18
& 46.01
& 42.42
& 49.83
& 51.52
& 54.14
& 34.79
& 52.00
& 51.13
& 55.72
& 60.00
& \(67.69{\pm}5.75\)
& \(\mathbf{70.77{\pm}5.75}\) \\
\bottomrule
\end{tabular}%
}

\endgroup
\end{table}

%% file: tables/cifar_summary.tex
\begin{table}[t]
\centering
\caption{One-class anomaly detection on CIFAR-10~\citep{cifar}, measured by
AUC-ROC (\%). Values are averaged over the ten nominal-class tasks. Full
class-wise results, including uncertainty over 10 seeds for AE-PP and AE-PoS,
are reported in Table~\ref{tab:oneclass-cifar10}.}
\label{tab:oneclass-cifar10-summary}

\begingroup
\footnotesize
\setlength{\tabcolsep}{2.5pt}
\renewcommand{\arraystretch}{1.05}

\resizebox{\columnwidth}{!}{%
\begin{tabular}{@{}l*{13}{c}@{}}
\toprule
&
\multicolumn{11}{c}{Baselines} &
\multicolumn{2}{c}{Ours} \\
\cmidrule(lr){2-12}
\cmidrule(lr){13-14}
&
OCGAN
& VAE
& DSVDD
& DSEBM
& DAGMM
& IF
& AnoGAN
& ALAD
& MEMAE
& MEMGAN
& TOLL
& AE-PP
& AE-PoS \\
\midrule
Average
& 65.6
& 58.3
& 64.8
& 60.4
& 54.4
& 55.5
& 61.8
& 59.3
& 58.5
& 65.3
& 69.6
& 71.30
& \textbf{71.73} \\
\bottomrule
\end{tabular}%
}

\endgroup
\end{table}

%% file: tables/hq-degradation-summary.tex
\begin{table}[t]
\centering
\caption{Detection of unseen degradations in high-quality images.
Entries are unweighted averages within each degradation family, reported as
AUC-ROC and PR-AUC. Parentheses give the number of evaluation settings; the
overall row averages all 16 settings. Complete results are provided in
Table~\ref{tab:four-way}. Bold indicates the best result in each metric.}
\label{tab:hq-degradation-summary}

\begingroup
\scriptsize
\setlength{\tabcolsep}{2pt}
\renewcommand{\arraystretch}{1.02}

\begin{tabular*}{\columnwidth}{
  @{\extracolsep{\fill}}l*{8}{c}@{}
}
\toprule
&
\multicolumn{2}{c}{AE (0.83M)} &
\multicolumn{2}{c}{VAE (83.7M)} &
\multicolumn{2}{c}{AE-PP (0.83M)} &
\multicolumn{2}{c}{AE-PoS (1.67M)} \\
\cmidrule(lr){2-3}
\cmidrule(lr){4-5}
\cmidrule(lr){6-7}
\cmidrule(lr){8-9}
Family
& ROC & PR
& ROC & PR
& ROC & PR
& ROC & PR \\
\midrule

Additive noise (1)
& \textbf{1.000} & \textbf{1.000}
& 0.998 & 0.998
& \textbf{1.000} & \textbf{1.000}
& \textbf{1.000} & \textbf{1.000} \\

Blur (7)
& 0.568 & 0.536
& 0.133 & 0.334
& 0.889 & 0.881
& \textbf{0.903} & \textbf{0.896} \\

Rain (5)
& 0.440 & 0.462
& 0.791 & 0.687
& 0.946 & 0.931
& \textbf{0.958} & \textbf{0.938} \\

Underwater (3)
& 0.295 & 0.378
& 0.449 & 0.439
& 0.967 & 0.962
& \textbf{0.985} & \textbf{0.981} \\

\midrule
Overall (16)
& 0.504 & 0.512
& 0.452 & 0.506
& 0.929 & 0.919
& \textbf{0.941} & \textbf{0.932} \\

\bottomrule
\end{tabular*}

\endgroup
\end{table}

%% file: tables/ecg_table.tex
\begin{table}[!h]
\centering
\caption{Frozen-encoder AUROC. Stage~0 and Stage~1 denote our
depth-controlled models. Bold indicates the best result in each column.}
\label{tab:ecg-nested}

\begingroup
\small
\setlength{\tabcolsep}{5pt}
\renewcommand{\arraystretch}{1.05}

\begin{tabular}{@{}lccccc@{}}
\toprule
& \multicolumn{4}{c}{PTB-XL} & ICBEB \\
\cmidrule(lr){2-5}
Encoder & Super & Sub & Form & Rhythm & CPSC2018 \\
\midrule
ECG-CPC
& 0.8630 & 0.8860 & 0.8730 & 0.9380 & 0.9020 \\
ECG-FM
& 0.8730 & 0.8660 & 0.7820 & 0.8650 & 0.9060 \\
MERL
& 0.8870 & 0.8950 & 0.8640 & 0.8770 & 0.9140 \\
ST-MEM
& 0.8954 & 0.8935 & 0.8548 & 0.9463 & 0.9559 \\
\midrule
Stage 0
& 0.8969 & 0.8917 & 0.8523 & 0.9449 & 0.9530 \\
Stage 1
& \textbf{0.9136}
& \textbf{0.9277}
& \textbf{0.8825}
& \textbf{0.9482}
& \textbf{0.9623} \\
\bottomrule
\end{tabular}

\endgroup
\end{table}

%% file: appendices/Definitions_and_proofs.tex
\section{Meet--Union--Join Geometry: Definitions and Proofs}
\label{app:union-interval}

This appendix formalizes the tangent-space model used in the main text. In
each local chart, the component submanifolds are represented by tangent
subspaces in common ambient coordinates, where the meet, union, and join
relations hold pointwise. Applied throughout the tangent bundle, these
fiberwise relations provide a global characterization of the nonlinear
geometry without assuming that the manifolds or networks are globally linear.
Whenever \(\dist(s,T)\) is interpreted as a reconstruction score, we assume the
ideal nearest-point projection prescribed by P2. For a generic autoencoder, its
output need not be the nearest point in its learned range, so geometric distance
and reconstruction error are kept distinct.

\subsection{Axiomatic dependence of the results}

\input{tables/axiom-map}

Table~\ref{tab:axiom-result-map} makes the axiomatic dependence explicit.
A checkmark denotes a postulate used directly in the construction or result;
an open circle denotes a postulate needed to transfer a geometric statement to
the corresponding neural-network interpretation. Throughout the paper, every
use of a PoS postulate is stated explicitly; an unmarked step uses only the
displayed hypotheses and standard metric, set, or linear geometry. All definitions and results below apply to any input or latent representation
space. We suppress the level index throughout; at level \(\ell\), the symbols
\(S_i\), \(\cW\), \(\cU\), \(\cV\), and \(T\) are understood as their
level specific counterparts.

\subsection{Meet, union, and join}

Let $S_1,\ldots,S_M$ be linear subspaces of a common finite-dimensional
Euclidean space $\R^n$. In the manifold interpretation, they are tangent spaces
expressed in a common local coordinate system.

\begin{definition}[Meet, union, and join]
Define
\begin{equation}
    \cW:=\bigcap_{i=1}^{M}S_i,
    \qquad
    \cU:=\bigcup_{i=1}^{M}S_i,
    \qquad
    \cV:=\sum_{i=1}^{M}S_i.
    \label{eq:appendix-interval}
\end{equation}
Then $\cW\subseteq\cU\subseteq\cV$. The union $\cU$ is generally not a
subspace, whereas $\cW$ and $\cV$ are subspaces. This chain concerns the three
reference sets only and imposes no ordering on an arbitrary learned range $T$.
\end{definition}

\begin{definition}[Geometric and reconstruction scores]
For a nonempty set \(T\subseteq\mathbb{R}^n\), define the geometric score
\[
    a_{\mathrm{geo}}(s;T):=\dist(s,T)
    =\inf_{t\in T}\lVert s-t\rVert_2.
\]
For an autoencoder \(P_\theta\) with learned range
\(T_\theta:=\operatorname{range}(P_\theta)\), define its reconstruction score
\[
    a_{\mathrm{rec}}(s;\theta)
    :=\lVert s-P_\theta(s)\rVert_2.
\]
Since \(P_\theta(s)\in T_\theta\),
\[
    a_{\mathrm{geo}}(s;T_\theta)
    \leq
    a_{\mathrm{rec}}(s;\theta).
\]
Equality holds if and only if \(P_\theta(s)\) is a nearest point to \(s\) in
\(T_\theta\). Under P2, equality therefore holds wherever the learned nonlinear
orthogonal projector selects a nearest point in its range. Such a nearest point
need not be unique; conditions ensuring uniqueness are outside the scope of
this paper.
\end{definition}

\subsubsection{Linear autoencoders are non-compact learners of subspace unions}
\label{app:linear-noncompact}

The compact and non-compact terminology follows the original PoS
formulation~\citep{pos}. When \(\cU\neq\cV\), a linear autoencoder whose range
contains all nominal components cannot compact its learned range to their union
\(\cU\), because linearity forces the range to contain their entire join
\(\cV\). The following proposition formalizes the contrast between linear and
ReLU networks observed in the PoS toy examples and complements the off-support
reconstruction analysis of \citet{bouman2025}.

\begin{proposition}[Linear non-compaction]
\label{prop:linear-noncompact}
Let \(P=f_D\circ f_E:\R^n\to\R^n\) be a linear autoencoder map with learned
range \(T:=\operatorname{range}(P)\). If \(\cU\subseteq T\), then
\begin{equation}
    \cV=\operatorname{span}(\cU)\subseteq T.
    \label{eq:linear-noncompact}
\end{equation}
Consequently, if \(T\subseteq\cV\), then \(T=\cV\). Whenever
\(\cU\neq\cV\), no linear autoencoder whose range contains \(\cU\) can have
learned range \(T=\cU\).
\end{proposition}

\begin{proof}
The range \(T\) of a linear map is a subspace. Since \(T\) contains \(\cU\),
it must also contain the smallest linear subspace containing \(\cU\), namely
\(\operatorname{span}(\cU)=\sum_i S_i=\cV\). The remaining statements follow
immediately.
\end{proof}

Thus, a linear autoencoder whose range covers \(\cU\) necessarily also contains
the off-union directions \(\cV\setminus\cU\). Under P2, these directions become
zero-score blind regions. Exact carving of a genuine union
\(\cU\neq\cV\) therefore requires nonlinearity somewhere in the complete
reconstruction map.

\begin{remark}[Architectural consequence for nested carving]
The obstruction concerns the complete reconstruction map
\(P=f_D\circ f_E\), not the encoder alone. At any level \(\ell\) where
\(\cU^{(\ell-1)}\neq\cV^{(\ell-1)}\), a linear stage whose pull anchors enforce
\(\cU^{(\ell-1)}\subseteq T^{(\ell)}\) must also satisfy
\(\cV^{(\ell-1)}\subseteq T^{(\ell)}\). It may remove directions outside the
join, but it cannot exclude the interpolating directions
\(\cV^{(\ell-1)}\setminus\cU^{(\ell-1)}\). Exact union carving therefore
requires the complete stage map to be nonlinear. This requirement does not
prescribe where activations or gating must be placed, and it does not apply
when the nominal union is already a subspace.
\end{remark}

\begin{proposition}[Geometric score sandwich]
\label{prop:score-sandwich}
For every \(s\in\R^n\),
\begin{equation}
    \dist(s,\cV)
    \leq
    \dist(s,\cU)
    =
    \min_{1\leq i\leq M}\dist(s,S_i)
    \leq
    \dist(s,\cW).
    \label{eq:score-sandwich}
\end{equation}
\end{proposition}

\begin{proof}
For sets \(A\subseteq B\), distance reverses inclusion:
\(\dist(s,B)\leq\dist(s,A)\). Applying this to
\(\cW\subseteq\cU\subseteq\cV\) gives the two inequalities. Since
\(\cU=\bigcup_{i=1}^{M}S_i\) is a finite union, its distance is the minimum of
the component distances.
\end{proof}

Under P2, this ordering transfers to reconstruction scores when the
corresponding networks realize nearest-point projections. It does not hold
automatically for generic autoencoders.

\subsection{Optimality of the nominal union among geometrically faithful ranges}
\label{app:union-optimality}

Assume that \(p_{\mathrm{nom}}\) is supported on the nominal union \(\cU\).
Since \(\cU\) is a finite union of closed subspaces, it is closed. Define the
closed candidate ranges that geometrically preserve the complete nominal
support by
\begin{equation}
    \mathfrak T_{+}
    :=
    \left\{
        T\subseteq\R^n:
        \cU\subseteq T,\;
        T\ \text{is closed}
    \right\}.
    \label{eq:admissible-ranges}
\end{equation}
For every \(T\in\mathfrak T_{+}\),
\(\dist(s_n,T)=0\) for all \(s_n\in\cU\). Under P2, this corresponds to perfect
nominal reconstruction.

For an anomaly distribution \(p_a\) for which the following expectations
exist, define the expected geometric score gap
\begin{equation}
    \Delta(T)
    :=
    \mathbb E_{s_a\sim p_a}\!\left[\dist(s_a,T)\right]
    -
    \mathbb E_{s_n\sim p_{\mathrm{nom}}}\!\left[\dist(s_n,T)\right].
    \label{eq:expected-gap}
\end{equation}
The second term vanishes for every \(T\in\mathfrak T_{+}\), so
\(\Delta(T)\) equals the expected anomaly distance from \(T\). The question is
therefore which geometrically nominal-faithful range keeps anomalies farthest
away. This analysis transfers to reconstruction error only when the network
realizes the nearest-point projection prescribed by P2. Moreover,
\(\Delta(T)\) is an expected score gap and does not by itself determine AUC.

\begin{lemma}[Nominal union optimality]
\label{lem:union-optimality}
Assume that the nominal distribution is supported on \(\cU\), that is,
\(p_{\mathrm{nom}}(\cU)=1\). For every \(T\in\mathfrak T_{+}\) and every
\(s\in\R^n\),
\begin{equation}
    \dist(s,T)\leq\dist(s,\cU).
    \label{eq:pointwise-optimality}
\end{equation}
Consequently,
\begin{equation}
    \max_{T\in\mathfrak T_{+}}\Delta(T)=\Delta(\cU),
    \label{eq:expected-optimality}
\end{equation}
so the nominal union \(\cU\) is an optimal geometrically faithful range.
\end{lemma}

\begin{proof}
Since \(\cU\subseteq T\), set inclusion gives
\(\dist(s,T)\leq\dist(s,\cU)\). Moreover,
\(p_{\mathrm{nom}}(\cU)=1\) implies that both nominal distance terms in
\(\Delta(T)\) and \(\Delta(\cU)\) vanish. Taking expectations over \(p_a\)
gives \(\Delta(T)\leq\Delta(\cU)\). Since \(\cU\) is closed,
\(\cU\in\mathfrak T_{+}\), so the upper bound is attained at \(T=\cU\).
\end{proof}

\begin{remark}
The lemma identifies \(\cU\) as an optimal geometric target, independently of
the particular model used to represent the learned range. Nominal
reconstruction fidelity preserves \(\cU\), while the push and pull terms use
controlled perturbations of nominal samples to discourage excess range.
Together, these signals encourage the learned range \(T\) to approach the
compact union \(\cU\) without requiring observed anomalies.

The optimum need not be unique for a fixed anomaly distribution. A larger
range \(T\supsetneq\cU\) may attain the same expected score if
\(\dist(s_a,T)=\dist(s_a,\cU)\) for \(p_a\)-almost every anomaly.

If the exact support assumption is relaxed and, for some
\(\varepsilon\geq0\),
\begin{equation}
    \mathbb E_{s_n\sim p_{\mathrm{nom}}}
    \left[\dist(s_n,\cU)\right]
    \leq\varepsilon,
\end{equation}
then every \(T\in\mathfrak T_{+}\) satisfies
\begin{equation}
    \Delta(T)\leq\Delta(\cU)+\varepsilon.
    \label{eq:epsilon-optimality}
\end{equation}
Thus, when nominal observations concentrate near \(\cU\), the nominal union
remains optimal up to \(\varepsilon\) in expected geometric score gap.
\end{remark}

\subsection{Join blindness}

\begin{corollary}[Join blindness]
\label{cor:join-blindness}
Suppose join collapse occurs, so the learned operator has range
$T_\theta:=\operatorname{range}(P_\theta)=\cV$. Under the ideal geometric score,
every sample $s\in\cV$ satisfies $\sigma_{\cV}(s)=0$. Hence nominal data
supported on $\cU$ and any anomaly distribution supported on $\cV$ receive
identical constant scores. Under the standard half-credit convention for ties,
the corresponding AUC is $1/2$.
\end{corollary}

\begin{proof}
The distance from any member of $\cV$ to $\cV$ is zero. Therefore every
nominal and anomalous score is zero.
\end{proof}

\subsection{Meet preference and fidelity loss}
\label{app:meet}

Recall that the meet
\[
\cW:=\bigcap_{i=1}^{M}S_i
\]
is the subspace shared by all component tangent spaces. Limited capacity need
not recover exactly \(\cW\), but may preferentially retain its shared
directions while discarding component specific directions. We call the
limiting case \(T=\cW\) meet collapse. The following proposition gives a
sufficient condition for this preference under rank constrained MSE learning.

\begin{proposition}[Shared direction preference under balanced compression]
\label{prop:balanced-meet}

Let nominal data be sampled uniformly from \(M\) components, with
\(\mathbb E[s\mid i]=0\). Suppose \(\dim\cW>0\), and write
\[
S_i=\cW\oplus C_i,
\qquad
C_i:=S_i\cap\cW^\perp.
\]
Assume that the \(C_i\) are mutually orthogonal and that
\begin{equation}
    \operatorname{Cov}(s\mid i)
    =
    \sigma_{\cW}^{2}P_{\cW}
    +
    \sigma_C^{2}P_{C_i}.
    \label{eq:balanced-component-covariance}
\end{equation}
Then the pooled covariance is
\begin{equation}
    \Sigma
    =
    \sigma_{\cW}^{2}P_{\cW}
    +
    \frac{\sigma_C^{2}}{M}
    \sum_{i=1}^{M}P_{C_i}.
    \label{eq:balanced-pooled-covariance}
\end{equation}
If
\begin{equation}
    \sigma_{\cW}^{2}>\frac{\sigma_C^{2}}{M},
    \label{eq:shared-direction-condition}
\end{equation}
every direction in \(\cW\) has strictly larger covariance eigenvalue than
every component specific direction. Consequently, a globally MSE optimal
linear autoencoder constrained to rank \(r\) retains shared directions first.
In particular, if \(r=\dim\cW\), its learned range is exactly \(\cW\).
\end{proposition}

\begin{proof}
Averaging the conditional covariances over the \(M\) components gives
Eq.~(\ref{eq:balanced-pooled-covariance}). Mutual orthogonality gives
eigenvalue \(\sigma_{\cW}^{2}\) on \(\cW\) and
\(\sigma_C^{2}/M\) on each \(C_i\). Condition
(\ref{eq:shared-direction-condition}) therefore places all shared directions
before all component specific directions in the covariance spectrum. By the linear autoencoder characterization of
\citet{baldi1989neural}, the global MSE optimum projects onto the leading principal
subspace. At rank \(r=\dim\cW\), condition
(\ref{eq:shared-direction-condition}) makes this subspace exactly \(\cW\),
proving the result.
\end{proof}

\begin{remark}
Exact meet collapse occurs at \(r=\dim\cW\). If \(r<\dim\cW\), the learned
range is a proper subspace of \(\cW\). If \(r>\dim\cW\), it retains \(\cW\)
together with selected component specific directions. Thus, the proposition
establishes shared direction preference, not universal convergence to
\(\cW\). Mutual orthogonality of the \(C_i\) is a simplifying sufficient
assumption. For positive variances, condition
(\ref{eq:shared-direction-condition}) is equivalent to
\[
M>\frac{\sigma_C^{2}}{\sigma_{\cW}^{2}},
\]
so shared direction preference strengthens with the number of balanced
components when \(\sigma_{\cW}^{2}\) and \(\sigma_C^{2}\) are fixed.
\end{remark}

\begin{proposition}[Nominal distortion under meet collapse]
\label{prop:meet-distortion}

Suppose meet collapse occurs, so that the learned range is \(T=\cW\). For
every nominal sample \(s\in S_i\),
\begin{equation}
    a_{\mathrm{geo}}(s;\cW)
    =
    \dist(s,\cW)
    =
    \left\lVert
    P_{S_i\cap\cW^\perp}s
    \right\rVert_2.
    \label{eq:meet-distortion}
\end{equation}
Thus, every nominal sample with nonzero component specific tangent content
outside \(\cW\) receives a positive geometric score. Under P2, this score
equals its reconstruction error.
\end{proposition}

\begin{proof}
Since \(\cW\subseteq S_i\),
\[
s
=
P_{\cW}s
+
P_{S_i\cap\cW^\perp}s.
\]
The two terms are orthogonal, and \(P_{\cW}s\) is the nearest point to \(s\)
in \(\cW\). Therefore,
\[
\dist(s,\cW)
=
\left\lVert s-P_{\cW}s\right\rVert_2
=
\left\lVert P_{S_i\cap\cW^\perp}s\right\rVert_2,
\]
proving Eq.~(\ref{eq:meet-distortion}).
\end{proof}

\begin{remark}
Meet collapse guarantees loss of nominal fidelity, but it does not determine
AUC without specifying the anomaly distribution. In particular, chance AUC
requires additional distributional conditions, such as identical nominal and
anomaly score distributions. The geometric prediction is therefore a loss of
fidelity, not a universal claim of chance detection.
\end{remark}

\input{tables/tab_toy_geometries}

\paragraph{Toy data verification:} Table~\ref{tab:toy-geometries} reports reconstruction scores for the toy data
illustrated in Figure~\ref{fig:toy-failure-geometry}. The join model
reconstructs nominal and anomalous samples similarly, demonstrating the loss
of selectivity shown in Figure~\ref{fig:toy-failure-geometry}(b). The rank
deficient model assigns large errors to both groups, demonstrating the loss of
nominal fidelity shown in Figure~\ref{fig:toy-failure-geometry}(d). Because the
two nominal subspaces have a trivial meet, this case is not an instantiation of
Proposition~\ref{prop:balanced-meet}. Only the compact union model preserves
low nominal error while assigning substantially larger scores to anomalies, as
shown in Figure~\ref{fig:toy-failure-geometry}(c).

\subsection{Proof of the separate pull and push guarantees}
\label{app:push-pull-proof}

\begin{proof}
Since \(P_\theta(\widetilde{s}_n)\in T\), the pull bound implies
\[
    \dist_d(s_n,T)
    \le d(P_\theta(\widetilde{s}_n),s_n)
    \le \varepsilon_{\mathrm{pull}}.
\]
Moreover, \([z]_+\ge z\) gives
\[
    \rho_{\theta,d}(\widetilde{s}_n)
    \ge \alpha\delta_e-\varepsilon_{\mathrm{push}}.
\]
Under P2,
\(\rho_{\theta,d}(\widetilde{s}_n)
=\dist_d(\widetilde{s}_n,T)\).
The triangle inequality also gives
\[
    \delta_e
    \le \dist_d(\widetilde{s}_n,T)
       +d(P_\theta(\widetilde{s}_n),s_n)
    \le \dist_d(\widetilde{s}_n,T)
       +\varepsilon_{\mathrm{pull}}.
\]
Together with nonnegativity, these inequalities prove
Eq.~(\ref{eq:push-pull-carving}).

Finally, if \(\cU\subseteq T\), then
\[
    \dist_d(\widetilde{s}_n,T)
    \le \dist_d(\widetilde{s}_n,\cU)
    \le \delta_e+\eta(s_n).
\]
Combining this upper bound with
\(\dist_d(\widetilde{s}_n,T)
\ge\alpha\delta_e-\varepsilon_{\mathrm{push}}\)
and dividing by \(\delta_e>0\) proves
Eq.~(\ref{eq:alpha-compatibility}).
\end{proof}

\section{Detailed construction of nested manifold carving}
\label{app:nested-carving}

This appendix gives the stagewise construction and capacity conditions behind
Section~\ref{sec:nested-carving}. Following P4, nested carving applies dynamic
push--pull recursively in learned latent spaces. Let \(x^{(0)}=s\). At stage
\(\ell\), the previously learned encoder--decoder pairs are frozen, and a new
pair is trained on the current nominal representation \(x^{(\ell-1)}\):
\begin{equation}
x^{(\ell)}=f_E^{(\ell)}(x^{(\ell-1)}),
\qquad
\widehat{x}^{(\ell-1)}
=f_D^{(\ell)}(x^{(\ell)}).
\label{eq:nested-module}
\end{equation}

Before training, the new latent dimension is chosen to preserve the current
nominal union. In the local linearized model, write
\(\cU^{(\ell-1)}=\bigcup_i S_i^{(\ell-1)}\) and let
\(f_E^{(\ell)}:\mathbb{R}^{n_{\ell-1}}\to\mathbb{R}^{n_\ell}\).
As reviewed in PoS~\cite{pos} and classical union-of-subspaces
theory~\cite{blumensath2009sampling, lu2008theory}, a linear encoder is injective on
\(\cU^{(\ell-1)}\) precisely when
\begin{equation}
\ker f_E^{(\ell)}
\cap
\bigl(S_i^{(\ell-1)}+S_j^{(\ell-1)}\bigr)
={0}
\quad\text{for all }i,j.
\label{eq:latent-injectivity}
\end{equation}
For a finite family, this condition is generically feasible when
\begin{equation}
n_\ell\geq k_{\max}^{(\ell-1)}
:=
\max_{i,j}
\dim\bigl(S_i^{(\ell-1)}+S_j^{(\ell-1)}\bigr).
\label{eq:latent-capacity}
\end{equation}
Below this width, some pairwise difference is necessarily annihilated, so
distinct nominal points become indistinguishable and exact reconstruction of
the union is impossible.

Equations~(\ref{eq:latent-injectivity})--(\ref{eq:latent-capacity}) preserve
nominal distinctions but do not themselves ensure compact-union carving. If
the complete stage map is linear and faithful on a non-subspace union,
Proposition~\ref{app:linear-noncompact} forces its range to contain the
corresponding join. Exact union carving therefore additionally requires
nonlinearity and informative perturbations.

When Eq.~(\ref{eq:latent-injectivity}) holds and

$$
S_i^{(\ell)}
:=f_E^{(\ell)}(S_i^{(\ell-1)}),
$$

the dimensions of all pairwise sums are preserved, giving
\(k_{\max}^{(\ell)}=k_{\max}^{(\ell-1)}\). Hence the capacity floor in
Eq.~(\ref{eq:latent-capacity}) persists across faithful linearized levels. A
simple sufficient architectural choice is \(n_\ell\geq n_{\ell-1}\), for
which a generic full-column-rank encoder has zero kernel. These are local
linearized capacity conditions; nominal fidelity must still be learned by the
pull objective.

Finally, a perturbation \(e^{(\ell)}\) is applied to
\(x^{(\ell-1)}\), and Eq.~(\ref{eq:push-pull-loss}) is used at level \(\ell\)
with
\begin{equation}
s_n=x^{(\ell-1)},\qquad
e=e^{(\ell)},\qquad
P_\theta=P_{\theta_\ell}
:=f_D^{(\ell)}\circ f_E^{(\ell)}.
\label{eq:level-push-pull}
\end{equation}
Thus, the pull term preserves the nominal latent representation, while the
push term discourages the new stage from reproducing its perturbed counterpart.
After optimization, the new pair is frozen and the procedure is repeated at
the next latent level. For compact notation, define the encoder and decoder stacks
$$
F_\ell
:=f_E^{(\ell)}\circ\cdots\circ f_E^{(1)},
\qquad
G_\ell
:=f_D^{(1)}\circ\cdots\circ f_D^{(\ell)},
$$

with \(F_0=G_0=I\). The complete reconstruction map through level \(\ell\)
is therefore
\begin{equation}
\overline P^{(\ell)}
:=G_{\ell-1}\circ P_{\theta_\ell}\circ F_{\ell-1}
=G_\ell\circ F_\ell,
\qquad
\overline T^{(\ell)}
:=\operatorname{range}\bigl(\overline P^{(\ell)}\bigr),
\label{eq:decoded-stage-range}
\end{equation}
where, \(\overline T^{(\ell)}\) is the range of the level-\(\ell\)
reconstruction after decoding it back to the input space.

Range nesting requires the new stage to remain compatible with the
representation inherited from the preceding encoder stack:
\begin{equation}
P_{\theta_\ell}\left(\operatorname{range}(F_{\ell-1})\right)
\subseteq \operatorname{range}(F_{\ell-1}).
\label{eq:nested-compatibility}
\end{equation}
Indeed, Eq.~(\ref{eq:nested-compatibility}) gives $
\overline T^{(\ell)}
\subseteq
G_{\ell-1}\left(\operatorname{range}(F_{\ell-1})\right)
=
\overline T^{(\ell-1)}.
$ Nominal coverage is preserved if the preceding encoder--decoder stack and the
new stage are faithful on \(\cU\), namely,

$$
G_{\ell-1}(F_{\ell-1}(u))=u,
\qquad
P_{\theta_\ell}(F_{\ell-1}(u))=F_{\ell-1}(u),
\qquad u\in\cU.
$$

These conditions imply $
\overline P^{(\ell)}(u)
=
G_{\ell-1}\left(
P_{\theta_\ell}(F_{\ell-1}(u))
\right)
=u,
$ and hence \(\cU\subseteq\overline T^{(\ell)}\). Applying compatibility and
nominal fidelity at every level yields
\begin{equation}
\overline T^{(1)}
\supseteq\overline T^{(2)}
\supseteq\cdots
\supseteq\overline T^{(L)}
\supseteq\cU.
\label{eq:nested-ranges-appendix}
\end{equation}

%% file: tables/axiom-map.tex
\begin{table}[!h]
\centering
\caption{Dependence of the main results and constructions on the PoS
postulates. A checkmark denotes direct use, while an open circle denotes that
the postulate is required only to transfer a geometric statement to its
neural-network interpretation. Blank entries rely only on the stated
hypotheses and standard geometry.}
\label{tab:axiom-result-map}

\begingroup
\small
\setlength{\tabcolsep}{4pt}
\renewcommand{\arraystretch}{1.08}

\begin{tabular}{@{}p{0.68\columnwidth}cccc@{}}
\toprule
Result or construction & P1 & P2 & P3 & P4 \\
\midrule
Score ordering, union optimality, and linear obstruction
& & \(\circ\) & & \\

Join blindness and meet preference as PoS failures
& \(\checkmark\) & \(\circ\) & & \\

Dynamic Push and Pull construction
& \(\checkmark\) & \(\checkmark\) & \(\checkmark\) & \\

Push and Pull geometric guarantees
& & \(\checkmark\) & & \\

Nested Manifold Carving and the nested-range relation
& \(\checkmark\) & \(\checkmark\) & \(\checkmark\) & \(\checkmark\) \\

Increasing relative separation of nominal components
& \(\checkmark\) & & & \(\checkmark\) \\

Movement of off-union samples toward nominal components
& & \(\checkmark\) & \(\checkmark\) & \(\checkmark\) \\
\bottomrule
\end{tabular}

\endgroup
\end{table}

%% file: tables/tab_toy_geometries.tex
\begin{table}[t]
\centering
\scriptsize

\caption{Toy verification of the learned geometries and their score
consequences. The linear two dimensional model learns the join, the ReLU model
approximates the compact union, and the one dimensional linear model exhibits
rank deficient underlearning. The meet is trivial in this toy, so the final
case illustrates nominal fidelity loss rather than exact meet collapse.
Distances report the mean and standard deviation of reconstruction scores.}
\label{tab:toy-geometries}
\begin{tabular*}{\linewidth}{@{\extracolsep{\fill}}lccccc@{}}
\toprule
Learned geometry
& \(d_{\mathrm{nom,train}}\)
& \(d_{\mathrm{nom,test}}\)
& \(d_{\mathrm{anom}}\)
& \(\frac{d_{\mathrm{anom}}}{d_{\mathrm{nom,test}}}\)
& AUROC \\
\midrule
Join (linear AE, \(2\mathrm{D}\))
& \(0.134\pm0.11\)
& \(0.143\pm0.10\)
& \(0.136\pm0.12\)
& \(0.95\)
& \(0.456\) \\
Compact union (ReLU AE, \(2\mathrm{D}\))
& \(0.226\pm0.12\)
& \(0.244\pm0.11\)
& \(1.624\pm0.56\)
& \(6.64\)
& \(1.000\) \\
Rank deficient underlearning (linear AE, \(1\mathrm{D}\))
& \(1.364\pm0.51\)
& \(1.386\pm0.44\)
& \(1.435\pm0.90\)
& \(1.04\)
& \(0.422\) \\
\bottomrule
\end{tabular*}
\end{table}

%% file: appendices/metrics.tex
\section{Detailed MNIST ablation study}
\label{app:mnist-ablation}

This appendix provides the complete results supporting the controlled MNIST
study in Section~\ref{sec:mnist-controlled}. We consider four nominal-digit
sets and treat all remaining digits as anomalies. All models use the same
ConvAE backbone, with Stage-0 latent dimension \(9\) and Stage-1 expanded
dimension \(60\). The ordinary autoencoder objective uses MSE, whereas
push--pull uses the Euclidean metric. For every stage \(\ell\), the anomaly
score is always computed in the input space as

$$
\rho^{(\ell)}(s)
=
\left\|s-\widehat{s}^{(\ell)}\right\|_2,
$$

where \(\widehat{s}^{(\ell)}\) is the complete reconstruction through that
stage. Stage-0 and Stage-1 scores are reported separately; no latent-space or
multistage score fusion is used. Results are averaged over seeds \(42\)--\(46\),
and F1 uses the \(95\)th percentile of the fit-split nominal scores as its
threshold.

\paragraph{Per-stage results:}
Table~\ref{tab:stage-metrics} reports the mean and sample standard deviation
of ROC-AUC and F1 at both stages for every nominal-digit set and training
configuration.


\paragraph{When nested carving contributes:}
The benefit of Stage 1 depends on which anomalous directions remain unresolved
after Stage 0. The clearest improvement occurs for the nominal set
\(\{0,2,4,5,6,8,9\}\), for which digit \(1\) is anomalous. Without masking or
sparsity, AE+PP increases from \(0.722\) to \(0.816\) ROC-AUC and from
\(0.436\) to \(0.489\) F1 between Stage 0 and Stage 1. This behavior is
consistent with the geometry in Figure~\ref{fig:mnist-stage-geometry}: digit
\(1\) can remain close to the Stage-0 join while lying outside the nominal
union, and therefore benefits from additional carving in the latent space.

When digit \(1\) is itself nominal, as in the set
\(\{1,2,5,7,8\}\), the same configuration changes only from \(0.921\) to
\(0.922\) ROC-AUC and from \(0.748\) to \(0.749\) F1. Here, Stage 0 already
separates the remaining anomalies effectively, leaving little unresolved
off-union structure for Stage 1 to remove. Intermediate gains appear for the
other two settings in which digit \(1\) is anomalous. With sparsity and no
masking, Stage 1 raises ROC-AUC from \(0.848\) to \(0.875\) for nominal digits
\(\{2,3,6,9\}\), and from \(0.903\) to \(0.928\) for nominal digits
\(\{4,5,7\}\). These results indicate that nesting is most useful when the
preceding stage retains anomaly directions within its learned coarse range,
rather than when the first projection has already resolved them.

\paragraph{Masked reconstruction versus push--pull:}
Masking alone provides no consistent carving advantage. For nominal digits
\(\{4,5,7\}\), the unmasked Stage-0 AE obtains \(0.893\) ROC-AUC and \(0.780\)
F1, while masked reconstruction decreases these values to \(0.881\) and
\(0.766\). In contrast, the unmasked AE+PP reaches \(0.897\) ROC-AUC and
\(0.792\) F1 at Stage 0, and nested push--pull further increases them to
\(0.909\) and \(0.798\) at Stage 1. The same general pattern appears across
the other nominal-digit sets: push--pull improves over the corresponding AE
objective more consistently than masking alone, and its gain can increase
when applied recursively.

Masking can nevertheless complement push--pull when combined with an
appropriate nested objective. For nominal digits \(\{4,5,7\}\), masked
AE+PP with sparsity reaches \(0.929\) ROC-AUC and \(0.815\) F1 at Stage 1,
the best result for this setting. Thus, the results do not show that masking
is intrinsically ineffective. Rather, masking by itself does not specify which
perturbation directions should be excluded, whereas push--pull explicitly
separates the reconstruction from the perturbed input while anchoring it to
the nominal observation.

\paragraph{Complete structural ablation:}
Table~\ref{tab:full-mnist-ablation} reorganizes the
\(\{4,5,7\}\) results to isolate the effects of hierarchy, push--pull,
masking, and sparsity. A reconstruction-only nested module leaves performance
essentially unchanged because the expanded latent autoencoder admits the
identity map and receives no explicit carving signal. Push--pull supplies this
signal, while sparsity and masking act as optional complementary
regularizers. Their effects are not uniformly additive: sparsity strengthens
nested push--pull in the \(\{4,5,7\}\) experiment but can reduce performance
for other nominal sets. The principal consistent gain therefore comes from
the push--pull objective, with nesting contributing when Stage 0 retains
off-union anomaly directions.

\input{tables/MNIST_and_metric_tables/different_normaldigits_2stage}
\input{tables/MNIST_and_metric_tables/different_normaldigits_ticked}

\begin{table}[h]
\centering
\small
\caption{Matched comparison of reconstruction, denoising, and Push and Pull
on MNIST with nominal digits \(4,5,7\). Masking and sparsity are disabled
to isolate the training objective. Stage 1 uses the complete nested
reconstruction decoded to input space, without score fusion.}
\label{tab:dae-push-pull-ablation}
\begin{tabular}{lcccc}
\toprule
Objective & Stage & F1 & ROC-AUC & PR-AUC \\
\midrule
AE
& 0 & 0.780 & 0.893 & 0.962 \\
AE
& 1 & 0.780 & 0.893 & 0.962 \\
\midrule
DAE
& 0 & 0.783 & 0.895 & 0.962 \\
DAE
& 1 & 0.783 & 0.895 & 0.962 \\
\midrule
AE with Push and Pull
& 0 & 0.792 & 0.897 & 0.963 \\
AE with Push and Pull
& 1 & \textbf{0.798} & \textbf{0.909} & \textbf{0.967} \\
\bottomrule
\end{tabular}
\end{table}

\paragraph{Denoising versus Push and Pull:}
A denoising autoencoder (DAE) receives a perturbed nominal sample
\(\widetilde{s}_n=s_n+e\) and learns to reconstruct its clean counterpart
\(s_n\). This corresponds to retaining only the Pull term. Indeed, setting
\(\alpha=0\) gives
\[
\mathcal{L}_{\mathrm{PP}}
=
d(\widehat{s}_n,s_n)
+
\bigl[-d(\widehat{s}_n,\widetilde{s}_n)\bigr]_+
=
d(\widehat{s}_n,s_n),
\]
which is the DAE objective, up to whether the reconstruction discrepancy is
implemented as Euclidean distance or its square. The comparison in
Table~\ref{tab:dae-push-pull-ablation} therefore isolates the contribution of
the positive Push margin from ordinary denoising. DAE changes the Stage 0
results only marginally relative to AE, from \(0.780\) to \(0.783\) F1 and
from \(0.893\) to \(0.895\) ROC-AUC, and its Stage 1 results remain unchanged.
In contrast, Push and Pull reaches \(0.792\) F1 and \(0.897\) ROC-AUC at
Stage 0, followed by \(0.798\) F1 and \(0.909\) ROC-AUC after nested carving.
This behavior agrees with Lemma~1 and
Figure~\ref{fig:push-pull-geometry}. In the ideal case
\(s_n\in\mathcal{U}\), Pull can recover \(s_n\), and for
\(\alpha\leq1\) the Push hinge is inactive at that optimum, although it can
still supply a gradient while the margin is violated during learning. For
realistic nominal data lying near rather than exactly on \(\mathcal{U}\),
Pull alone only attracts the reconstruction toward \(s_n\), whereas a
positive Push margin additionally requires separation from
\(\widetilde{s}_n\). The table consequently indicates that exposure to
perturbations through denoising is insufficient to explain the observed
gain; the explicit separation constraint contributes additional carving,
particularly when it is applied at the nested stage.

\subsection{Geometric contribution of nested latent carving}
\label{app:latent-score-analysis}

Aggregate detection metrics show whether Stage 1 improves performance, but not
which geometric cases produce that improvement. We therefore examine the image
and latent reconstruction errors separately for each digit. This provides a
direct empirical test of the geometry proposed in
Figure~\ref{fig:mnist-stage-geometry} and of the representation dynamics
described in Section~\ref{sec:representation-dynamics}.

Let \(T_{\mathrm{in}}\) denote the frozen Stage 0 encoder, let the nested
module map \(x\) to \(\widehat{x}\), and let \(T_{\mathrm{out}}\) denote the
complete decoder. For an input image \(s\),
\[
x=T_{\mathrm{in}}(s), \qquad
\widehat{x}=T_{\mathrm{nested}}(x), \qquad
\widehat{s}=T_{\mathrm{out}}(\widehat{x}).
\]
We measure the change made by the complete model in the image domain and the
change made specifically by Stage 1 in the Stage 0 latent space:
\begin{align}
e^{\mathrm{rec}}(s)
&=
\frac{1}{|s|}
\left\|\widehat{s}-s\right\|_2^2,
&
\text{image reconstruction error},
\\
e^{\mathrm{sub}}(s)
&=
\frac{1}{|x|}
\left\|\widehat{x}-x\right\|_2^2,
&
\text{latent reconstruction error}.
\end{align}
Each quantity is computed per image by averaging over its elements. No
averaging across images or batches is used. The latent error is introduced
only for this diagnostic analysis. The principal results in the main paper
continue to use the complete input space reconstruction score without score
fusion.

The two errors distinguish the geometric cases illustrated in
Figure~\ref{fig:mnist-stage-geometry}. If the anomalous component of \(s\)
lies largely outside the coarse Stage 0 range \(\mathcal{V}\), Stage 0 removes
it before \(x\) is formed. The resulting latent representation can therefore
already resemble a nominal representation, giving a large image error but a
small Stage 1 latent error. Digit \(8\) is the clearest example of this case:
its anomalous direction lies predominantly outside the learned Stage 0 range,
and its retained representation becomes close to that of nominal digit \(5\).
By contrast, an anomaly lying inside
\(\mathcal{V}\setminus\mathcal{U}\) can be reconstructed well at Stage 0 and
remain difficult to detect in the image domain. Its retained off-union
direction is still present in \(x\), however, allowing Stage 1 to carve it and
produce a large discrepancy between \(x\) and \(\widehat{x}\). Figure
\ref{fig:mnist-stage-geometry} predicts digit \(1\) as the representative case.

To quantify this complementarity, we calibrate one threshold for each score
using a held out set
\(\mathcal{D}_{\mathrm{fit}}\) of \(n_{\mathrm{fit}}=435\) nominal images that
are used neither for training nor for evaluation. For
\(c\in\{\mathrm{rec},\mathrm{sub}\}\), let \(t_c(q)\) be the \(q\)th
percentile of score \(c\) on \(\mathcal{D}_{\mathrm{fit}}\). Define
\[
A(s)
=
\left[
e^{\mathrm{rec}}(s)>t_{\mathrm{rec}}(q)
\right],
\qquad
B(s)
=
\left[
e^{\mathrm{sub}}(s)>t_{\mathrm{sub}}(q)
\right].
\]
For this auxiliary analysis, an image is flagged when either channel fires:
\begin{equation}
\operatorname{flag}(s)=A(s)\lor B(s).
\label{eq:or-rule}
\end{equation}
Because Eq.~(\ref{eq:or-rule}) combines two tests, assigning the \(95\)th
percentile independently to both channels would produce a joint false positive
rate above \(5\%\). We therefore select a common percentile \(q^\star\) on
\(\mathcal{D}_{\mathrm{fit}}\) such that the combined rule has an approximately
\(5\%\) false positive rate. Across the evaluated runs,
\(q^\star\in[97.0,97.5]\), producing realized rates between \(0.046\) and
\(0.051\).

For each digit, Table~\ref{tab:flag-decomposition} reports the marginal firing
probabilities \(\Pr(A)\) and \(\Pr(B)\), together with the disjoint outcomes
\[
A\land\neg B,\qquad
\neg A\land B,\qquad
\neg A\land\neg B.
\]
The unreported intersection can be recovered as
\[
\Pr(A\land B)
=
\Pr(A)-\Pr(A\land\neg B)
=
\Pr(B)-\Pr(\neg A\land B).
\]
We additionally report
\begin{equation}
r_{\mathrm{latent}}
=
\frac{\Pr(\neg A\land B)}
     {\Pr(A\lor B)},
\label{eq:latent-only-share}
\end{equation}
which measures the fraction of detected samples that would be lost if only
the image reconstruction score were retained.

\input{tables/MNIST_and_metric_tables/subspace_detection}

The per digit decomposition closely matches the proposed geometry. For digit
\(8\), the image channel fires for \(87.6\%\) of samples, whereas the latent
channel fires for only \(12.4\%\), and only \(0.6\%\) of its detections depend
exclusively on the latent score. Thus, most of the anomalous information in
digit \(8\) is removed before Stage 1, leaving a Stage 0 representation that
the nested module changes little. This supports the first case in
Figure~\ref{fig:mnist-stage-geometry}: an anomaly outside the coarse range is
detected through its large image reconstruction error, but its removed
directions are unavailable to deeper latent processing.

Digit \(1\) exhibits the opposite behavior. The image channel fires for only
\(12.3\%\) of its samples, consistent with digit \(1\) lying close to the
coarse range \(\mathcal{V}\) and therefore remaining a Stage 0 blind case. The
latent channel fires for \(57.6\%\), with \(48.1\%\) of all digit \(1\)
samples detected exclusively in the latent space. Consequently, \(80.8\%\) of
the detected digit \(1\) samples would be missed without the Stage 1 latent
score. This result supports the second geometric case: Stage 0 retains the
digit \(1\) direction inside its coarse range, while Stage 1 differentiates it
from the compact nominal union.

The reconstructions in Figure~\ref{fig:stages} provide complementary visual
evidence. After nested carving, digit \(1\) no longer preserves its original
structure and frequently becomes \(7\)-like, with some reconstructions
remaining visually ambiguous between \(7\) and \(9\). The \(7\)-like cases
are consistent with the predicted representation dynamics: as the learned
range becomes more compact, an off-union anomaly is moved toward a nearby
nominal component. The experiment therefore does more than show an aggregate
benefit from nesting. It identifies the anomalies removed before the latent
stage, the anomalies exposed by latent carving, and the loss of anomalous
information after projection. This agreement between the predicted geometry,
the latent discrepancies, and the reconstructed images demonstrates the
explanatory value of the PoS based geometric framework.

\begin{figure}[t]
  \centering
  \includegraphics[
    trim=0 15 40 10,
    clip,
    width=0.55\linewidth
  ]{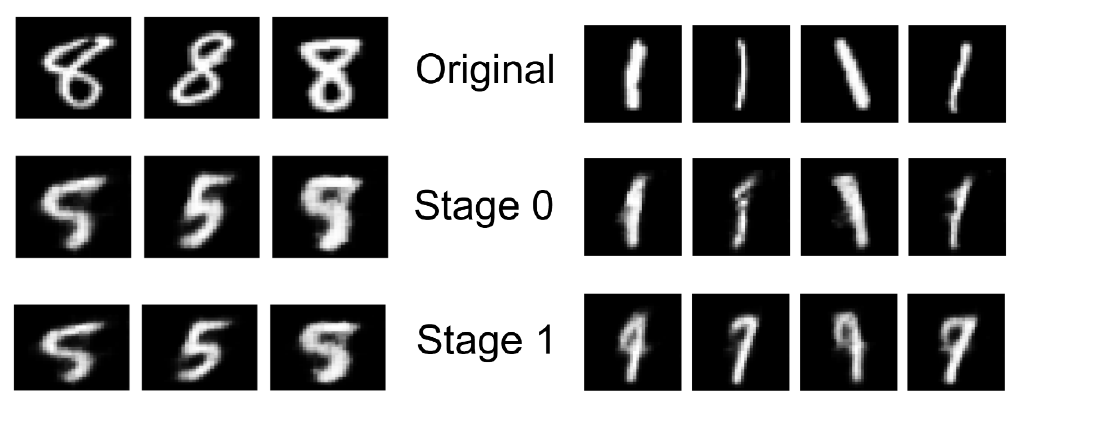}
  \caption{Representative reconstructions at successive stages. Anomalous
  directions removed at Stage 0 cannot be recovered from the Stage 0 latent
  representation. Digit \(1\), which remains inside the coarse Stage 0 range,
  is modified by nested carving and frequently becomes \(7\)-like, consistent
  with projection toward a nearby nominal component.}
  \label{fig:stages}
\end{figure}

\section{Metrics}
\label{app:Metrics}
The following metrics in this section are mainly introduced to provide additional quantitative information regarding the claims that have been asserted throughout the paper. It should be noted that all introduced metrics are always taken between $x$ and $\hat x$ which are the outputs of the frozen encoder and the outputs of the subspace, namely the network in between the frozen encoder decoder pair, respectively. Before any metric is computed, both spaces are standardised per dimension using the mean and the standard deviation of the normal $x$ samples (with $\varepsilon=10^{-8}$ added to the standard deviation and this value is used everywhere), so that the same map is applied to $x$ and $\hat x$. All metrics are computed in the full latent space and since the results are a comparison of the results taken from two latent spaces having the same dimension, their difference is calculated using the relative change, for any scalar quantity $Q$, as following,
\[
  \Delta Q \;=\; 100 \cdot \frac{Q_{\hat x}-Q_{x}}{Q_{x}+\varepsilon} \quad [\%]
\]

\input{tables/MNIST_and_metric_tables/custom_metric_normal}
\input{tables/MNIST_and_metric_tables/custom_metric_anomaly}

Additionally, there exists some commonly used quantities in all the metrics, namely the center of the normal clusters and their radii. Let $z \in \{x, \hat x\}$ index the two spaces mentioned above, and write $z_i$ for the representation of point $i$ in that space. The normal classes are $c \in \mathcal{C}$ with $K=|\mathcal{C}|$ and index sets $\mathcal{H}_c$, which are the same in both spaces, and the anomalies are $a \in \mathcal{A}$. The anomaly classes are $b \in \mathcal{B}$ with $M=|\mathcal{B}|$ and index sets $\mathcal{A}_b$, again the same in both spaces. For each normal class $c \in \mathcal{C}$, we, first, get the center points and the radii of the formed clusters in the latent space according to their labels. The centers of the normal clusters, denoted as $\mu_c^{z}$, are calculated by taking the means of these points, and the radii, denoted as $R_c^{z}$, are calculated by taking the distance of each point in the class to the center and averaging over the class, i.e.,

\begin{equation}
    \mu_c^{z} = \frac{1}{|\mathcal{H}_c|}\sum_{i\in\mathcal{H}_c} z_i,
\qquad
R_c^{z} = \frac{1}{|\mathcal{H}_c|}\sum_{i\in\mathcal{H}_c} \bigl\|z_i-\mu_c^{z}\bigr\|
\label{eq:center-radius}
\end{equation}
where $\|\cdot\|=\|\cdot\|_2$ for all cases.

\subsection{Separation Ratio}

To quantify the possible separation in between the normal clusters, we have used a metric we call the Separation Ratio, $S^{z}$. This metric is calculated by taking the means of the inter-distances of the normal clusters $\bar D_{\text{inter}}^{z}$ and dividing the result by the mean of the radii of all normal clusters $\bar R^{z}$. The relevant equations are as follows

\begin{equation}
    \bar D_{\text{inter}}^{z} = \binom{K}{2}^{-1}\!\!\sum_{\substack{c,c'\in\mathcal{C}\\ c<c'}}\!\bigl\|\mu_c^{z}-\mu_{c'}^{z}\bigr\|,
\qquad
\bar R^{z} = \frac{1}{K}\sum_{c\in\mathcal{C}} R_c^{z},
\qquad
S^{z} = \frac{\bar D_{\text{inter}}^{z}}{\bar R^{z}+\varepsilon},
\label{eq:separation-ratio}
\end{equation}
where $\varepsilon=10^{-8}$. The division by $\bar R^{z}$ makes $S^{z}$ invariant to a uniform rescaling of the latent space. What it does \emph{not} do is isolate separation from compactness. $S^{z}$ rises whenever $\bar D_{\text{inter}}^{z}$ grows relative to $\bar R^{z}$, which happens both when the normal classes move apart and when they contract in place. So we report $\bar D_{\text{inter}}^{z}$ and $\bar R^{z}$ separately in the results, and read $S^{z}$ as a statement about the relative movement of these two terms. For observing shrinkage in normal clusters, we expect the $\bar R^{z}$ to be negative, and for normal clusters moving out, we expect $\bar D_{\text{inter}}^{z}$ to be positive, and as the combination of both of these metrics we expect absolute of $S^{z}$ to be a high number indicating a separation.

\subsection{Proximity Ratio}

As for showing the predicted behavior of the anomalous data approach to the normal clusters in latent space, we have used another metric called Proximity Ratio denoted as $\rho^{z}$. For an anomalous point $a$, let $c^{*}=\arg\min_{c\in\mathcal{C}}\|a-\mu_c^{z}\|$ be its closest normal cluster and $d^{z}(a)$ the distance to that cluster's center. The score is then the median, over the anomalous points, of $d^{z}(a)$ divided by the radius of that closest cluster:
\begin{equation}
d^{z}(a) = \bigl\|a-\mu_{c^{*}}^{z}\bigr\|,
\qquad
    \rho^{z} = \operatorname*{med}_{a\in \mathcal{A}} \frac{d^{z}(a)}{R^{z}_{c^{*}}+\varepsilon}.
\end{equation}
Note that the denominator here is not a single constant per space: $R^{z}_{c^{*}}$ depends on which cluster is closest and therefore varies from one anomalous point to the next. This metric basically tells us whether the anomalies move faster toward the normal clusters than the rate of shrinkage of their closest normal clusters. This, along with $\Delta d$ columns can tell us whether the anomaly points are actually approaching to the normal cluster, and gives the rate of the approaching behavior.

But since this metric is susceptible to the behavior of the normal cluster as well -a shrink in the normal cluster can be interpreted as the anomalies moving away, inaccurately- we have devised some other ways to understand the behavior of the anomalous data. Firstly, we report the relative change of the median of $d^{z}$ itself, taken over all anomalous points at once, indicated as $\Delta d_{\text{samp}}$. Since the anomaly classes are not equally sized, we then take the median within each anomaly class and the median across those class values, indicated as $\Delta d_{\text{pool}}$, so that no single class dominates the number,

\begin{equation}
    d^{z}_{\text{samp}} = \operatorname*{med}_{a\in \mathcal{A}} d^{z}(a),
\qquad
    d^{z}_{\text{pool}} = \operatorname*{med}_{b\in \mathcal{B}}\;\Bigl(\operatorname*{med}_{a\in \mathcal{A}_b} d^{z}(a)\Bigr).
\end{equation}
What the pair of them diagnoses is the effect of the class sizes: $d^{z}_{\text{samp}}$ lets the larger anomaly classes weigh more, while $d^{z}_{\text{pool}}$ does not. It should be noted that with $M$ small and odd the outer median selects one class value rather than combining them, so $d^{z}_{\text{pool}}$ can coincide with a single anomaly class throughout. 

And lastly, we report a second normalised score, in which $d^{z}(a)$ is divided by the mean inter-normal-centroid distance $\bar D_{\text{inter}}^{z}$ rather than by a cluster radius. Unlike the denominator of $\rho^{z}$, this one is a single constant per space and is independent of which cluster is closest, so it acts as a scale reference: $\mathrm{IC}^{z}$ is unchanged under a uniform rescaling of the latent space and asks whether the anomalies approached the normal classes faster than the space as a whole contracted.

\begin{equation}
    \mathrm{IC}^{z} =\operatorname*{med}_{a\in \mathcal{A}} \frac{d^{z}(a)}{\bar D_{\text{inter}}^{z}+\varepsilon}.
\end{equation}

The expectation from $\rho^{z}$ and $\mathrm{IC}^{z}$ is for them to drop to show an approaching behavior of the anomalous data to their closest overall normal cluster, and thus the $\Delta \rho$ and $\Delta \mathrm{IC}$ results are expected to be negative.

%% file: tables/MNIST_and_metric_tables/different_normaldigits_2stage.tex
\begingroup
\setlength{\tabcolsep}{3.5pt}
\renewcommand{\arraystretch}{1.05}
\scriptsize

\setlength{\LTleft}{0pt}
\setlength{\LTright}{0pt}
\setlength{\LTcapwidth}{\linewidth}

\begin{longtable}{@{\extracolsep{\fill}}lccc|cc|cc@{}}
\caption{\textbf{Per-stage MNIST anomaly-detection performance.}
At each stage \(\ell\), the anomaly score is the input-space reconstruction
error \(d(s,\widehat{s}^{(\ell)})\), where
\(\widehat{s}^{(\ell)}\) is the complete reconstruction through that stage.
No latent-space or multistage score fusion is used. Each entry reports the
mean \(\pm\) sample standard deviation over seeds \(42\)--\(46\). AE denotes
the nominal reconstruction objective \(\cL_{\mathrm{AE}}\), while AE+PP
additionally includes \(\lambda_{\mathrm{PP}}\cL_{\mathrm{PP}}\). The best
result within each nominal-digit set and stage is shown in \textbf{bold}.}
\label{tab:stage-metrics}\\
\toprule
\textbf{Normal digits}
& \textbf{Aug.}
& \textbf{Objective}
& \textbf{Sparse}
& \multicolumn{2}{c|}{\textbf{Stage 0}}
& \multicolumn{2}{c}{\textbf{Stage 1}} \\
& & & & \textbf{ROC-AUC} & \textbf{F1}
& \textbf{ROC-AUC} & \textbf{F1} \\
\midrule
\endfirsthead

\multicolumn{8}{c}{\tablename\ \thetable{} -- continued} \\
\toprule
\textbf{Normal digits}
& \textbf{Aug.}
& \textbf{Objective}
& \textbf{Sparse}
& \multicolumn{2}{c|}{\textbf{Stage 0}}
& \multicolumn{2}{c}{\textbf{Stage 1}} \\
& & & & \textbf{ROC-AUC} & \textbf{F1}
& \textbf{ROC-AUC} & \textbf{F1} \\
\midrule
\endhead

\midrule
\multicolumn{8}{r}{\textit{Continued on next page}} \\
\endfoot

\bottomrule
\endlastfoot

0,2,4,5,6,8,9 & mask & AE & -- &
$0.700 \pm 0.015$ & $0.401 \pm 0.038$ &
$0.702 \pm 0.017$ & $0.402 \pm 0.040$ \\

0,2,4,5,6,8,9 & mask & AE & yes &
$0.685 \pm 0.008$ & $0.373 \pm 0.045$ &
$0.688 \pm 0.012$ & $0.382 \pm 0.038$ \\

0,2,4,5,6,8,9 & mask & AE+PP & -- &
$0.711 \pm 0.011$ & $0.427 \pm 0.029$ &
$\mathbf{0.822 \pm 0.016}$ & $0.493 \pm 0.028$ \\

0,2,4,5,6,8,9 & mask & AE+PP & yes &
$0.708 \pm 0.020$ & $0.400 \pm 0.030$ &
$0.782 \pm 0.024$ & $0.472 \pm 0.048$ \\

0,2,4,5,6,8,9 & none & AE & -- &
$0.704 \pm 0.016$ & $0.410 \pm 0.032$ &
$0.710 \pm 0.017$ & $0.411 \pm 0.035$ \\

0,2,4,5,6,8,9 & none & AE & yes &
$0.698 \pm 0.023$ & $0.394 \pm 0.043$ &
$0.704 \pm 0.023$ & $0.401 \pm 0.042$ \\

0,2,4,5,6,8,9 & none & AE+PP & -- &
$\mathbf{0.722 \pm 0.015}$ & $\mathbf{0.436 \pm 0.050}$ &
$0.816 \pm 0.017$ & $0.489 \pm 0.025$ \\

0,2,4,5,6,8,9 & none & AE+PP & yes &
$0.711 \pm 0.013$ & $0.417 \pm 0.049$ &
$0.782 \pm 0.038$ & $\mathbf{0.494 \pm 0.058}$ \\

\midrule

1,2,5,7,8 & mask & AE & -- &
$0.900 \pm 0.002$ & $0.664 \pm 0.028$ &
$0.899 \pm 0.002$ & $0.660 \pm 0.031$ \\

1,2,5,7,8 & mask & AE & yes &
$0.894 \pm 0.012$ & $0.653 \pm 0.049$ &
$0.895 \pm 0.011$ & $0.657 \pm 0.048$ \\

1,2,5,7,8 & mask & AE+PP & -- &
$0.919 \pm 0.002$ & $0.731 \pm 0.012$ &
$0.918 \pm 0.004$ & $0.724 \pm 0.026$ \\

1,2,5,7,8 & mask & AE+PP & yes &
$0.919 \pm 0.005$ & $0.710 \pm 0.038$ &
$0.918 \pm 0.006$ & $0.714 \pm 0.033$ \\

1,2,5,7,8 & none & AE & -- &
$0.910 \pm 0.008$ & $0.689 \pm 0.036$ &
$0.911 \pm 0.008$ & $0.690 \pm 0.036$ \\

1,2,5,7,8 & none & AE & yes &
$0.903 \pm 0.006$ & $0.672 \pm 0.041$ &
$0.904 \pm 0.007$ & $0.674 \pm 0.040$ \\

1,2,5,7,8 & none & AE+PP & -- &
$\mathbf{0.921 \pm 0.004}$ & $\mathbf{0.748 \pm 0.025}$ &
$\mathbf{0.922 \pm 0.003}$ & $\mathbf{0.749 \pm 0.025}$ \\

1,2,5,7,8 & none & AE+PP & yes &
$0.919 \pm 0.003$ & $0.727 \pm 0.021$ &
$0.919 \pm 0.004$ & $0.738 \pm 0.026$ \\

\midrule

2,3,6,9 & mask & AE & -- &
$0.814 \pm 0.004$ & $0.557 \pm 0.023$ &
$0.815 \pm 0.003$ & $0.561 \pm 0.017$ \\

2,3,6,9 & mask & AE & yes &
$0.823 \pm 0.003$ & $0.564 \pm 0.036$ &
$0.820 \pm 0.004$ & $0.572 \pm 0.036$ \\

2,3,6,9 & mask & AE+PP & -- &
$0.845 \pm 0.010$ & $0.613 \pm 0.025$ &
$0.847 \pm 0.006$ & $0.633 \pm 0.024$ \\

2,3,6,9 & mask & AE+PP & yes &
$0.841 \pm 0.006$ & $0.600 \pm 0.034$ &
$0.869 \pm 0.008$ & $0.646 \pm 0.055$ \\

2,3,6,9 & none & AE & -- &
$0.829 \pm 0.009$ & $0.601 \pm 0.031$ &
$0.828 \pm 0.010$ & $0.603 \pm 0.029$ \\

2,3,6,9 & none & AE & yes &
$0.830 \pm 0.009$ & $0.580 \pm 0.033$ &
$0.830 \pm 0.011$ & $0.584 \pm 0.033$ \\

2,3,6,9 & none & AE+PP & -- &
$0.842 \pm 0.005$ & $0.614 \pm 0.035$ &
$0.850 \pm 0.008$ & $0.626 \pm 0.022$ \\

2,3,6,9 & none & AE+PP & yes &
$\mathbf{0.848 \pm 0.003}$ & $\mathbf{0.626 \pm 0.027}$ &
$\mathbf{0.875 \pm 0.010}$ & $\mathbf{0.665 \pm 0.036}$ \\

\midrule

4,5,7 & mask & AE & -- &
$0.881 \pm 0.005$ & $0.766 \pm 0.017$ &
$0.882 \pm 0.003$ & $0.766 \pm 0.018$ \\

4,5,7 & mask & AE & yes &
$0.884 \pm 0.005$ & $0.769 \pm 0.009$ &
$0.886 \pm 0.005$ & $0.770 \pm 0.008$ \\

4,5,7 & mask & AE+PP & -- &
$0.902 \pm 0.003$ & $0.790 \pm 0.016$ &
$0.913 \pm 0.002$ & $0.794 \pm 0.019$ \\

4,5,7 & mask & AE+PP & yes &
$0.899 \pm 0.006$ & $0.788 \pm 0.018$ &
$\mathbf{0.929 \pm 0.007}$ & $\mathbf{0.815 \pm 0.006}$ \\

4,5,7 & none & AE & -- &
$0.893 \pm 0.008$ & $0.780 \pm 0.012$ &
$0.893 \pm 0.008$ & $0.780 \pm 0.011$ \\

4,5,7 & none & AE & yes &
$0.890 \pm 0.011$ & $0.772 \pm 0.012$ &
$0.889 \pm 0.011$ & $0.772 \pm 0.009$ \\

4,5,7 & none & AE+PP & -- &
$0.897 \pm 0.004$ & $0.792 \pm 0.008$ &
$0.909 \pm 0.008$ & $0.798 \pm 0.010$ \\

4,5,7 & none & AE+PP & yes &
$\mathbf{0.903 \pm 0.004}$ & $\mathbf{0.793 \pm 0.011}$ &
$0.928 \pm 0.004$ & $0.808 \pm 0.029$ \\

\end{longtable}
\endgroup

%% file: tables/MNIST_and_metric_tables/different_normaldigits_ticked.tex
\begingroup
\footnotesize
\setlength{\tabcolsep}{3.5pt}
\renewcommand{\arraystretch}{1.05}

\begin{longtable}{@{}lccccccc@{}}
\caption{\textbf{Full factorial ablation on MNIST with nominal digits
\(4,5,7\).} The first row reproduces the two-dimensional autoencoder baseline
of~\cite{bouman2025}. Rows without a hierarchy tick use the
Stage-0 reconstruction score; ticked rows use the complete Stage-1
reconstruction decoded to the input space. Values are
means over different seeds. Best results are shown in \textbf{bold} and
second-best results are \underline{underlined}.}
\label{tab:full-mnist-ablation}\\
\toprule
\textbf{Normal digits}
& \textbf{Hierarchy}
& \textbf{Push}
& \textbf{Mask}
& \textbf{Sparsity}
& \textbf{F1}
& \textbf{ROC-AUC}
& \textbf{PR-AUC} \\
& \textbf{level}
& \textbf{pull}
& & & & & \\
\midrule
\endfirsthead

\multicolumn{8}{c}{\tablename\ \thetable{} -- continued} \\
\toprule
\textbf{Normal digits}
& \textbf{Hierarchy}
& \textbf{Push}
& \textbf{Mask}
& \textbf{Sparsity}
& \textbf{F1}
& \textbf{ROC-AUC}
& \textbf{PR-AUC} \\
& \textbf{level}
& \textbf{pull}
& & & & & \\
\midrule
\endhead

\midrule
\multicolumn{8}{r}{\textit{Continued on next page}} \\
\endfoot

\bottomrule
\endlastfoot
 Baseline (4,5,7) &  &  &  &  & 0.693 & 0.866 & 0.936 \\
\midrule
4,5,7 &  &  &  &  & 0.780 & 0.893 & 0.962 \\
4,5,7 & \checkmark &  &  &  & 0.780 & 0.893 & 0.962 \\
4,5,7 &  &  &  & \checkmark & 0.772 & 0.890 & 0.960 \\
4,5,7 & \checkmark &  &  & \checkmark & 0.772 & 0.889 & 0.960 \\
4,5,7 &  & \checkmark &  &  & 0.792 & 0.897 & 0.963 \\
4,5,7 & \checkmark & \checkmark &  &  & 0.798 & 0.909 & 0.967 \\
4,5,7 &  & \checkmark &  & \checkmark & 0.793 & 0.903 & 0.965 \\
4,5,7 & \checkmark & \checkmark &  & \checkmark & \underline{0.808} & \underline{0.928} & \underline{0.972} \\
4,5,7 &  &  & \checkmark &  & 0.766 & 0.881 & 0.957 \\
4,5,7 & \checkmark &  & \checkmark &  & 0.766 & 0.882 & 0.957 \\
4,5,7 &  &  & \checkmark & \checkmark & 0.769 & 0.884 & 0.958 \\
4,5,7 & \checkmark &  & \checkmark & \checkmark & 0.770 & 0.886 & 0.959 \\
4,5,7 &  & \checkmark & \checkmark &  & 0.790 & 0.902 & 0.965 \\
4,5,7 & \checkmark & \checkmark & \checkmark &  & 0.794 & 0.913 & 0.968 \\
4,5,7 &  & \checkmark & \checkmark & \checkmark & 0.788 & 0.899 & 0.964 \\
4,5,7 & \checkmark & \checkmark & \checkmark & \checkmark & \textbf{0.815} & \textbf{0.929} & \textbf{0.973} \\
\end{longtable}
\endgroup

%% file: tables/MNIST_and_metric_tables/subspace_detection.tex
\begin{table}[t]
\centering
\caption{Per-image flag decomposition of the two-channel or-rule.
Entries are fractions of test images, with thresholds calibrated on held-out
nominal images to a joint \(5\%\) false-positive rate (\(q\approx97\)).
Results are mean \(\pm\) standard deviation over 12 runs. Nominal digits are
the negative class; hence, their detection share is undefined.}
\label{tab:flag-decomposition}

\begingroup
\scriptsize
\setlength{\tabcolsep}{2.5pt}
\renewcommand{\arraystretch}{0.92}

\resizebox{\columnwidth}{!}{%
\begin{tabular}{@{}llcccccc@{}}
\toprule
& & \multicolumn{2}{c}{Channel fires}
& \multicolumn{3}{c}{Disjoint outcomes} & \\
\cmidrule(lr){3-4}
\cmidrule(lr){5-7}
Digit & Role
& Recon.
& Subsp.
& \makecell{Recon.\\only}
& \makecell{Subsp.\\only}
& Neither
& \makecell{Subsp.-only\\share of det.} \\
\midrule
0
& anomaly
& \(0.952{\pm}0.044\)
& \(0.325{\pm}0.113\)
& \(0.632{\pm}0.127\)
& \(0.006{\pm}0.008\)
& \(0.042{\pm}0.038\)
& \(0.6\%\) \\

\textbf{1}
& anomaly
& \(\mathbf{0.123{\pm}0.094}\)
& \(\mathbf{0.576{\pm}0.137}\)
& \(0.028{\pm}0.015\)
& \(\mathbf{0.481{\pm}0.101}\)
& \(0.396{\pm}0.126\)
& \(\mathbf{80.8\%}\) \\

2
& anomaly
& \(0.965{\pm}0.008\)
& \(0.444{\pm}0.098\)
& \(0.531{\pm}0.099\)
& \(0.010{\pm}0.005\)
& \(0.024{\pm}0.005\)
& \(1.1\%\) \\

3
& anomaly
& \(0.827{\pm}0.033\)
& \(0.129{\pm}0.077\)
& \(0.705{\pm}0.093\)
& \(0.008{\pm}0.004\)
& \(0.165{\pm}0.030\)
& \(0.9\%\) \\

6
& anomaly
& \(0.581{\pm}0.062\)
& \(0.388{\pm}0.080\)
& \(0.249{\pm}0.081\)
& \(0.056{\pm}0.031\)
& \(0.364{\pm}0.050\)
& \(8.8\%\) \\

8
& anomaly
& \(0.876{\pm}0.031\)
& \(0.124{\pm}0.105\)
& \(0.758{\pm}0.106\)
& \(0.006{\pm}0.008\)
& \(0.117{\pm}0.032\)
& \(0.7\%\) \\

9
& anomaly
& \(0.101{\pm}0.014\)
& \(0.049{\pm}0.026\)
& \(0.081{\pm}0.020\)
& \(0.030{\pm}0.017\)
& \(0.870{\pm}0.023\)
& \(21.6\%\) \\

\midrule
4
& nominal
& \(0.017{\pm}0.004\)
& \(0.026{\pm}0.008\)
& \(0.012{\pm}0.003\)
& \(0.021{\pm}0.008\)
& \(0.962{\pm}0.009\)
& --- \\

5
& nominal
& \(0.050{\pm}0.011\)
& \(0.049{\pm}0.013\)
& \(0.034{\pm}0.009\)
& \(0.033{\pm}0.007\)
& \(0.917{\pm}0.015\)
& --- \\

7
& nominal
& \(0.023{\pm}0.005\)
& \(0.034{\pm}0.008\)
& \(0.016{\pm}0.002\)
& \(0.027{\pm}0.006\)
& \(0.950{\pm}0.007\)
& --- \\
\bottomrule
\end{tabular}%
}

\endgroup
\end{table}

%% file: tables/MNIST_and_metric_tables/custom_metric_normal.tex
\begin{table}[h]\centering\small
\caption{\textbf{Do the normal classes separate from one another?}
Every column is a relative change from $x$ to $\hat x$.
$\bar D_{\text{inter}}$ is how far apart the classes sit;
$\bar S_{\text{intra}}$ is how wide the
classes are;
$S=\bar D_{\text{inter}}/\bar S_{\text{intra}}$ combines the two into a single
separation score, so it can rise either because the classes move apart or because
they contract in place. Under PushPull the rise in $S$ is driven by $\bar S_{\text{intra}}$ falling by
roughly $8$--$10\%$ while $\bar D_{\text{inter}}$ also falls, only more slowly. MSE leaves all three unchanged.
The arrow marks the direction of improvement, the best value is
\textbf{bold} and the second best is \underline{underlined}.}
 
\label{tab:normal-full}
\begin{tabular}{llcc rrr}
\toprule
normal set & mask & loss & sparsity & $\Delta S\ \uparrow$ & $\Delta \bar D_{\text{inter}}$ & $\Delta \bar S_{\text{intra}}$\\
 & & & & (\%) & (\%) & (\%) \\
 
\midrule
$\{4,5,7\}$ & none & MSE & -- & +0.1 & +0.0 & -0.0 \\
 & none & MSE & yes & -0.2 & -0.1 & +0.1 \\
 & mask & MSE & -- & -0.1 & +0.0 & +0.1 \\
 & mask & MSE & yes & -0.8 & -0.9 & -0.1 \\
 & none & PP & -- & \underline{+7.6} & -3.1 & -10.0 \\
 & none & PP & yes & +2.6 & -0.4 & -3.0 \\
 & mask & PP & -- & \textbf{+7.7} & -2.0 & -9.0 \\
 & mask & PP & yes & +3.5 & -0.9 & -4.3 \\
\midrule
$\{2,3,6,9\}$ & none & MSE & -- & +0.2 & +0.0 & -0.2 \\
 & none & MSE & yes & -0.4 & +0.4 & +0.8 \\
 & mask & MSE & -- & -0.3 & -0.2 & +0.1 \\
 & mask & MSE & yes & -0.8 & -1.1 & -0.3 \\
 & none & PP & -- & \underline{+5.2} & -4.0 & -8.7 \\
 & none & PP & yes & +1.9 & -2.8 & -4.6 \\
 & mask & PP & -- & \textbf{+7.2} & -1.9 & -8.5 \\
 & mask & PP & yes & +1.8 & -1.2 & -2.9 \\
\midrule
$\{1,2,5,7,8\}$ & none & MSE & -- & -0.1 & +0.2 & +0.3 \\
 & none & MSE & yes & +0.0 & +0.8 & +0.8 \\
 & mask & MSE & -- & +0.1 & -0.2 & -0.4 \\
 & mask & MSE & yes & -0.1 & -1.6 & -1.5 \\
 & none & PP & -- & \textbf{+5.4} & -2.9 & -7.9 \\
 & none & PP & yes & +1.9 & -0.9 & -2.7 \\
 & mask & PP & -- & \underline{+5.3} & -2.7 & -7.6 \\
 & mask & PP & yes & +1.4 & -0.1 & -1.5 \\
\midrule
$\{0,2,4,5,6,8,9\}$ & none & MSE & -- & +0.3 & -0.0 & -0.3 \\
 & none & MSE & yes & +0.7 & -0.0 & -0.7 \\
 & mask & MSE & -- & +0.1 & -0.0 & -0.1 \\
 & mask & MSE & yes & +0.1 & -1.3 & -1.5 \\
 & none & PP & -- & \underline{+6.5} & -2.0 & -8.0 \\
 & none & PP & yes & +1.6 & -0.7 & -2.3 \\
 & mask & PP & -- & \textbf{+7.3} & -1.7 & -8.3 \\
 & mask & PP & yes & +0.7 & -0.4 & -1.1 \\
\bottomrule
\end{tabular}
\end{table}

%% file: tables/MNIST_and_metric_tables/custom_metric_anomaly.tex
\begin{table}[h!]\centering\small
\caption{\textbf{Do the anomalies move into the normal clusters?} All columns are relative changes from $x$ to $\hat x$, and all are built from the same quantity
$d(a)$, the distance from an anomalous point to its nearest normal centroid.
$\Delta d_{\text{samp}}$ is the median of $d(a)$ over all anomalous points, while $\Delta d_{\text{pool}}$ takes the
median within each anomaly class and then the median across classes.
$\Delta\mathrm{IC}$ reports whether the anomalies approached
faster than the space itself contracted.
$\Delta\rho$ divides instead by the radius of each point's own nearest cluster.
The arrows mark the direction of improvement, and the best value in each column is set in bold and the second best is \underline{underlined}.}
\label{tab:anom-full}
\begin{tabular}{llcc rrrr}
\toprule
normal set & mask & loss & sparsity & $\Delta d_{\text{samp}}\ \downarrow$ & $\Delta d_{\text{pool}}\ \downarrow$ & $\Delta \mathrm{IC}\ \downarrow$ & $\Delta\rho\ \downarrow$\\
 & & & & (\%) & (\%) & (\%) & (\%) \\
\midrule
$\{4,5,7\}$ & none & MSE & -- & -0.4 & -0.1 & -0.5 & -0.1 \\
 & none & MSE & yes & -0.5 & -0.3 & -0.4 & -0.6 \\
 & mask & MSE & -- & -1.0 & -0.8 & -1.0 & -0.6 \\
 & mask & MSE & yes & -1.1 & -1.7 & -0.2 & -1.0 \\
 & none & PP & -- & \underline{-12.4} & \textbf{-15.3} & -9.6 & -2.2 \\
 & none & PP & yes & -12.1 & -13.4 & \underline{-11.8} & \underline{-9.2} \\
 & mask & PP & -- & -11.9 & -12.7 & -10.1 & -3.3 \\
 & mask & PP & yes & \textbf{-15.5} & \underline{-14.8} & \textbf{-14.7} & \textbf{-10.4} \\
\midrule
$\{2,3,6,9\}$ & none & MSE & -- & -0.7 & -1.1 & -0.8 & -0.7 \\
 & none & MSE & yes & -1.0 & -1.4 & -1.4 & -2.3 \\
 & mask & MSE & -- & -0.8 & -1.6 & -0.5 & -0.9 \\
 & mask & MSE & yes & -2.4 & -2.5 & -1.3 & -2.6 \\
 & none & PP & -- & \underline{-15.3} & \underline{-19.7} & -11.8 & -7.2 \\
 & none & PP & yes & \textbf{-18.1} & \textbf{-20.6} & \textbf{-15.7} & \textbf{-13.0} \\
 & mask & PP & -- & -12.7 & -15.7 & -11.1 & -3.9 \\
 & mask & PP & yes & -13.6 & -14.1 & \underline{-12.6} & \underline{-8.7} \\
\midrule
$\{1,2,5,7,8\}$ & none & MSE & -- & -0.1 & -0.2 & -0.3 & -0.3 \\
 & none & MSE & yes & -0.3 & -1.0 & -1.1 & -1.2 \\
 & mask & MSE & -- & -0.3 & -0.4 & -0.1 & -0.3 \\
 & mask & MSE & yes & -2.7 & -2.5 & -1.1 & -2.3 \\
 & none & PP & -- & \textbf{-11.7} & \textbf{-12.0} & \textbf{-9.0} & \textbf{-6.3} \\
 & none & PP & yes & -5.2 & -2.3 & -4.3 & \underline{-3.5} \\
 & mask & PP & -- & \underline{-9.9} & \underline{-11.1} & \underline{-7.4} & -3.1 \\
 & mask & PP & yes & -2.4 & -1.2 & -2.2 & -2.0 \\
\midrule
$\{0,2,4,5,6,8,9\}$ & none & MSE & -- & -0.4 & -0.1 & -0.3 & -0.6 \\
 & none & MSE & yes & +1.0 & +1.9 & +1.0 & +0.8 \\
 & mask & MSE & -- & -0.9 & +0.1 & -0.8 & -0.8 \\
 & mask & MSE & yes & -2.4 & -1.1 & -1.1 & -1.2 \\
 & none & PP & -- & \textbf{-21.0} & \textbf{-17.1} & \textbf{-19.4} & \underline{-8.8} \\
 & none & PP & yes & -10.8 & -13.2 & -10.1 & -5.7 \\
 & mask & PP & -- & \underline{-20.0} & \underline{-15.1} & \underline{-18.6} & \textbf{-11.1} \\
 & mask & PP & yes & -8.7 & -9.3 & -8.3 & -6.9 \\
\bottomrule
\end{tabular}
\end{table}

%% file: appendices/one_class_ablation.tex
\section{Standard One-Class Anomaly Detection}
\label{apd:oneclass}

This section gives the protocol, competing methods, implementation details,
and ablations for the experiments in
Section~\ref{sec:standard-oneclass}. It also reports the class-wise results
omitted from the main paper.

\subsection{Benchmarks and protocol}\label{apd:protocol}

We use MNIST~\citep{mnist}, Fashion-MNIST~\citep{fashion_mnist},
CIFAR-10 and CIFAR-100~\citep{cifar}, and
Arrhythmia~\citep{arrhythmia_odds}. MNIST, Fashion-MNIST, and CIFAR-10
define ten one-class problems, one for each nominal class. CIFAR-100 defines
twenty problems using its coarse superclasses. Arrhythmia defines one problem;
following \citet{toll}, classes 3, 4, 5, 7, 8, 9, 14, and 15 are anomalous and
the remaining classes are nominal.

For each image problem, the training split is restricted to the nominal class.
The validation and test splits retain all classes and are relabelled as nominal
or anomalous. We use ten seeds for each nominal class. A seed controls model
initialisation and stochastic operations but does not change the data split.
Validation is performed every 20 iterations, and the reported value is the test
score at the checkpoint selected on the corresponding validation metric. We use
AUC-ROC for the image datasets. For Arrhythmia, the 15\% of test samples with
the highest anomaly scores are labelled anomalous before computing $F_1$, as in
\citet{toll}. Class-wise results are means and standard deviations over seeds;
dataset averages are unweighted means over nominal classes.

MNIST and Fashion-MNIST inputs are standardised to zero mean and unit variance.
CIFAR-10 and CIFAR-100 use global contrast normalisation with the $\ell_1$ norm,
followed by min--max scaling to $[0,1]$, as in \citet{toll}. No anomalous sample
is used for training. Table~\ref{tab:dataset-stats} summarises the data used
before restricting the training pool to a nominal class.

\begin{table}[t]
\centering
\caption{One-class benchmark statistics. For the image datasets, the training
count denotes the full pool before restriction to the nominal class. All our
variants and seeds use the same fixed splits.}
\label{tab:dataset-stats}

\begingroup
\footnotesize
\setlength{\tabcolsep}{2.5pt}
\renewcommand{\arraystretch}{1.05}

\begin{tabular*}{\columnwidth}{@{\extracolsep{\fill}}lccccc@{}}
\toprule
Dataset
& \shortstack{Training\\pool}
& Validation
& Test
& \shortstack{Input\\shape}
& Problems \\
\midrule
MNIST
& 50\,000
& 10\,000
& 10\,000
& \(1{\times}28{\times}28\)
& 10 \\

Fashion-MNIST
& 50\,000
& 10\,000
& 10\,000
& \(1{\times}28{\times}28\)
& 10 \\

CIFAR-10
& 40\,000
& 10\,000
& 10\,000
& \(3{\times}32{\times}32\)
& 10 \\

CIFAR-100
& 40\,000
& 10\,000
& 10\,000
& \(3{\times}32{\times}32\)
& 20 \\

Arrhythmia
& \multicolumn{3}{c}{452 records (fixed stratified split)}
& 274 features
& 1 \\
\bottomrule
\end{tabular*}

\endgroup
\end{table}

\subsection{Competing methods}

We compare against the following competing methods: OC-SVM~\citep{ocsvm},
KDE~\citep{kde}, Isolation Forest~\citep{iforest}, Deep
SVDD~\citep{dsvdd}, the winner-take-all autoencoder
(DCAE)~\citep{dcae}, VAE~\citep{vae}, MemAE~\citep{memae},
MemGAN~\citep{memgan}, deep structured energy-based models
(DSEBM and DSEBM-e)~\citep{dsebm}, DAGMM~\citep{dagmm},
AnoGAN~\citep{anogan}, ALAD~\citep{alad}, ADGAN~\citep{adgan},
OCGAN~\citep{ocgan}, GANomaly~\citep{ganomaly}, RCGAN~\citep{rcgan},
GOAD~\citep{goad}, geometric-transformation classification
(GeoTrans)~\citep{geotrans}, the attribute-restoration framework
(ARNet)~\citep{arnet}, and targeted collapse (TOLL)~\citep{toll}.

TOLL adds a latent-norm compactness term to this autoencoder family; its
training loss and anomaly score include $\beta\lVert x\rVert_2$. The TOLL
column denotes that regularised method, whereas AE denotes the $\beta=0$
reconstruction-only baseline. AE-PP and AE-PoS also set $\beta=0$: AE-PP adds
Dynamic Push and Pull to the common backbone, and AE-PoS further applies Nested
Manifold Carving to the frozen Stage-0 representation. Thus, the methods share
TOLL's architecture and optimisation settings, but not its compactness
objective.

\subsection{Models and training}\label{apd:impl}

\paragraph{Common backbone:}
AE, AE-PP, and the Stage-0 component of AE-PoS use the autoencoder architecture,
latent width, optimiser, learning rate, and batch size specified by
\citet{toll}. This holds model capacity and the basic optimisation budget fixed
while changing the objective. Stage~0 trains for 2,000 iterations with Adam and
no weight decay. Stage~1 freezes the Stage-0 encoder and decoder from the same
nominal class and seed, then trains a nested MLP encoder--decoder for 6,000
iterations. The nested MLP uses ReLU, with LayerNorm on MNIST and
Fashion-MNIST and BatchNorm on Arrhythmia, CIFAR-10, and CIFAR-100. Its
dimensions and dataset-specific settings are given in Tables~\ref{tab:hparams}
and~\ref{tab:hparams-l1}.

\paragraph{Objectives and score:}
AE minimises clean reconstruction error. AE-PP retains that term and adds the
Dynamic Push and Pull loss using the negative
views in Table~\ref{tab:degview}. AE-PoS trains Stage~1 inside the frozen
Stage-0 model. All reconstruction and push-pull distances use the $\ell_2$ norm.
At test time, AE, AE-PP, and AE-PoS use reconstruction error as the anomaly
score.

\input{tables/headline}

\begin{table}[t]
  \centering
  \caption{Stage-0 AE-PP settings. Latent width, learning rate, and batch size
  follow \citet{toll}; scale and weight specify Dynamic Push and Pull. All runs
  use Adam without weight decay for 2,000 iterations and validate every 20
  iterations.}
  \label{tab:hparams}
  \begingroup
  \small
  \setlength{\tabcolsep}{3.5pt}
  \renewcommand{\arraystretch}{1.05}
  \begin{tabular}{@{}lccccc@{}}
  \toprule
  & & & & \multicolumn{2}{c}{Push-pull} \\
  \cmidrule(lr){5-6}
  Dataset & Latent & LR & Batch & Scale & Weight \\
  \midrule
  Arrhythmia    &  16 & $10^{-3}$ &  100 & 3.00 & 1.00 \\
  MNIST         & 128 & $10^{-4}$ & 1000 & 3.00 & 1.00 \\
  Fashion-MNIST & 256 & $10^{-4}$ &  500 & 3.00 & 1.00 \\
  CIFAR-10      &  64 & $10^{-2}$ &   50 & 2.00 & 1.00 \\
  CIFAR-100     &  64 & $10^{-3}$ &  100 & 2.00 & 1.00 \\
  \bottomrule
  \end{tabular}
  \endgroup
\end{table}

\begin{table}[t]
\centering
\caption{Stage-1 AE-PoS settings. Stage~1 is trained within the frozen
Stage~0 model for each nominal class and seed. The decoder mirrors the listed
MLP encoder path.}
\label{tab:hparams-l1}

\begingroup
\footnotesize
\setlength{\tabcolsep}{2.5pt}
\renewcommand{\arraystretch}{1.05}

\begin{tabular*}{\columnwidth}{@{\extracolsep{\fill}}lcccccc@{}}
\toprule
& & & &
\multicolumn{3}{c}{Push-pull} \\
\cmidrule(lr){5-7}
Dataset
& Encoder path
& LR
& Batch
& Space
& Scale
& Weight \\
\midrule
Arrhythmia
& \(16\!\to\!8\!\to\!2\)
& \(10^{-3}\)
& 100
& latent
& 1.00
& 1.00 \\

MNIST
& \(128\!\to\!256\!\to\!256\)
& \(10^{-4}\)
& 1000
& pixel
& 2.00
& 1.00 \\

Fashion-MNIST
& \(256\!\to\!384\!\to\!384\)
& \(10^{-4}\)
& 500
& pixel
& 2.00
& 4.00 \\

CIFAR-10
& \(64\!\to\!32\!\to\!32\)
& \(10^{-2}\)
& 50
& latent
& 2.00
& 1.00 \\

CIFAR-100
& \(64\!\to\!96\!\to\!96\)
& \(10^{-3}\)
& 100
& latent
& 2.00
& 1.00 \\

\bottomrule
\end{tabular*}

\endgroup
\end{table}

\subsection{Push-pull negative views}

The negative view is generated from a nominal training sample and is used only
to construct the push-pull objective; it is not labelled anomaly data. At
Stage~1, MNIST, Fashion-MNIST, and Arrhythmia reuse their Stage-0
transformation, and the CIFAR models perturb the frozen Stage-0 latent rather
than an image. Table~\ref{tab:degview} gives the complete
configuration.

\begin{table*}[t]
\centering
\caption{Negative views used by Dynamic Push and Pull. \emph{OneOf} samples
one transformation per input with weights \(3{:}1{:}1{:}1\) in the listed
order.}
\label{tab:degview}

\begingroup
\footnotesize
\setlength{\tabcolsep}{3pt}
\renewcommand{\arraystretch}{1.02}

\begin{tabularx}{\textwidth}{
  @{}
  l
  >{\raggedright\arraybackslash}X
  >{\raggedright\arraybackslash}X
  @{}
}
\toprule
Dataset & Stage 0 & Stage 1 \\
\midrule

MNIST
& CutPaste with three rotated patches
& Same CutPaste configuration \\

Fashion-MNIST
& \emph{OneOf}: CutPaste (area \(0.15\)--\(0.4\), two patches),
  patch shuffle (grid 2), \(90^\circ\) rotation, and phase scramble
  (\(0.3\))
& Same \emph{OneOf} configuration \\

CIFAR-10
& \emph{OneOf}: phase scramble (\(0.3\)), patch shuffle (grid 2),
  CutPaste, and channel shuffle
& Gaussian noise on the Stage-0 latent, with relative
  \(\sigma=0.03\) and \(p=0.8\) \\

CIFAR-100
& \emph{OneOf}: phase scramble (\(0.3\)), patch shuffle (grid 2),
  CutPaste, and channel shuffle
& Gaussian noise on the Stage-0 latent, with relative
  \(\sigma=0.03\) and \(p=0.8\) \\

Arrhythmia
& Feature shuffle with fraction \(0.2\)
& Same feature-shuffle configuration \\

\bottomrule
\end{tabularx}

\endgroup
\end{table*}

\subsection{Ablation and complete results}

The controlled ablation uses the same backbone, data splits, and training
budget for AE, AE-PP, and AE-PoS. AE isolates reconstruction, AE-PP adds Dynamic
Push and Pull, and AE-PoS further adds Stage~1.


Tables~\ref{tab:oneclass-mnist}, \ref{tab:oneclass-fmnist},
\ref{tab:oneclass-cifar10}, and~\ref{tab:oneclass-cifar100} report the complete
class-wise results for the image datasets. In particular,
Table~\ref{tab:oneclass-cifar10} expands the average CIFAR-10 comparison
reported in Table~\ref{tab:oneclass-cifar10-summary} of the main paper. Arrhythmia,
which comprises a single one-class task, is reported directly in
Table~\ref{tab:oneclass-arrhythmia} of the main paper.

\input{tables/mnist.tex}

\input{tables/fmnist.tex}
\input{tables/cifar10.tex}
\input{tables/cifar100.tex}

\subsection{Push--Pull Sensitivity}
\label{apd:ppcal}

Table~\ref{tab:pp-alpha-sweep} evaluates the sensitivity of AE-PP to the
margin scale \(\alpha\), with the Push--Pull weight fixed at \(w=1\). The AE
column is the reconstruction-only baseline obtained using the common TOLL
backbone with its compactness coefficient set to \(\beta=0\), and is therefore
distinct from the AE-PP configurations. When \(\alpha=0\), the Push term vanishes, leaving only the Pull objective, which maps a perturbed nominal view to its clean counterpart. This is exactly a denoising autoencoder when the perturbation is additive Gaussian noise, as in our controlled multi-component MNIST experiment, and more generally subsumes perturbation-to-clean autoencoders based on CutPaste, shuffling, rotations, and related transformations used in the other settings. For \(\alpha>0\), the same dataset-specific perturbations additionally activate Push and impose an explicit separation margin. The consistent improvement
from \(\alpha=0\) to \(\alpha>0\) isolates the benefit of the proposed Push
term beyond denoising and supports its intended role in separating
overlapping nominal submanifolds. Most gains are obtained by \(\alpha=2\) or
\(3\), after which performance remains broadly stable. Validation selects
\(\alpha=3\) for Arrhythmia, MNIST, and Fashion-MNIST, and \(\alpha=2\) for
CIFAR-10 and CIFAR-100, indicating limited sensitivity once a sufficient Push
margin is imposed. Arrhythmia is evaluated using F1, while all other datasets
use AUC-ROC.

\input{tables/sensitivity.tex}

Under the conditions of Lemma~\ref{lem:pushpull}, Eq.~(\ref{eq:alpha-compatibility}) also predicts saturation with
respect to \(\alpha\). Once \(\alpha\) exceeds the zero-Push-error bound
\(1+\eta(s_n)/\delta_e\) for most training pairs, the hinge remains active,
and its gradient,
\(-\nabla_\theta d(\hat{s}_n,\tilde{s}_n)\), no longer depends on
\(\alpha\). Consistent with this prediction, performance plateaus at
\(\alpha=2\) or \(3\); on Arrhythmia, all settings
\(\alpha\in\{3,4,5,6\}\) produce identical results.

\subsection{Detecting Unseen Degradations in High-Quality Images}
\label{apd:hq-full}

\input{tables/image}
\input{tables/image_psnr_full}

\subsubsection{Dataset Construction}\label{sec:image_data}

\noindent\textbf{Training data:}\quad
All models are trained on the same set of 8{,}693 clean images: 800 from
DIV2K~\citep{div2k}, 2{,}649 from Flickr2K~\citep{flickr2k}, 4{,}744 from the
Waterloo Exploration Database~\citep{wed}, and 500 from
BSDS500~\citep{bsds500}.

\noindent\textbf{Synthetic degraded views during training:}\quad
During training, each \(256\times256\) clean crop is paired with a degraded
view generated online for the Push--Pull loss. Defocus blur
(\(r\in[1,6]\)), motion blur (\(\ell\in[3,25]\)), and Gaussian blur
(\(\sigma\in[0.2,1.0]\)) are selected independently with probabilities
\(0.55\), \(0.40\), and \(0.30\), respectively. When \(n\) blur types are
selected, their strengths are scaled by \(1/\sqrt{n}\). Their kernels are
combined with an optional, unattenuated sinc kernel (\(p=0.10\)) and applied
once in linear-light space using \(\gamma=2.2\). Spatially varying defocus is
used for \(20\%\) of the defocus samples. The blurred image is then
downsampled by a random factor of at most \(2\) and resized to its original
resolution with probability \(0.5\). Gaussian sensor noise
(\(\sigma\in[0,6]\)) is always added, followed by JPEG compression with
quality \(70\)--\(95\) with probability \(0.8\).

\noindent\textbf{REDS blur set:}\quad
The REDS blur set contains five aligned degraded versions of each ground-truth
video frame, with blur levels \(0.2\), \(0.4\), \(0.6\), \(0.8\), and \(1.0\).
For evaluation, we use the same provided subset of 200 frames at every level,
with one frame selected from each video clip. The degradation increases
monotonically with the level: reconstruction error relative to the ground truth
rises while image sharpness decreases. The blur follows local scene motion,
leaving static regions nearly unchanged while increasingly blurring moving
regions, rather than applying a single spatially invariant kernel to the whole
image. Empirically, the mean squared deviation from the ground truth grows
approximately quadratically with the blur level.

\noindent\textbf{Underwater degradation:}\quad
We use underwater images from the LSUI dataset~\citep{lsui}, grouped by
degradation severity following~\citet{xopgan}. The groups are defined using
PSNR between each underwater image and its paired LSUI reference: high
degradation corresponds to \(8.79\)-\(16.57\) dB, medium degradation to
\(16.57\)-\(19.79\) dB, and low degradation to \(19.79\)-\(31.53\) dB.
Our evaluation contains 144, 143, and 143 images from these groups,
respectively. The paired LSUI references are used only to determine degradation
severity; clean DIV8K images define the nominal class in our anomaly-detection
evaluation.

\noindent\textbf{Anomalous test data:}\quad
The evaluation covers 16 settings across five degradation families. Additive
noise uses the 24 Kodak24 images~\citep{kodak24} corrupted with Gaussian noise
at \(\sigma=25\). The REDS blur set~\citep{reds} provides five aligned severity
levels from \(0.2\) to \(1.0\), evaluated on the same 200 frames at every
level. GoPro~\citep{gopro} contributes 200 motion-blurred images and is
evaluated under two nominal-reference protocols. Rain comprises Test100
(98 images)~\citep{rain_test100}, Rain100L and Rain100H (100 each)
\citep{rain100}, Test1200 (200)~\citep{rain_test1200}, and Test2800
(200)~\citep{rain_test2800}. Finally, LSUI~\citep{lsui} contributes high-,
medium-, and low-degradation underwater groups containing 144, 143, and 143
images, respectively, following the PSNR-based partition of
\citet{xopgan}.

\noindent\textbf{Nominal test data:}\quad
For Kodak24, REDS, and the paired GoPro protocol, we use the clean counterpart
of each degraded image, so the two classes differ only in degradation. For
rain, underwater, and the second GoPro protocol, the nominal class is drawn
from DIV8K~\citep{div8k}. These nominal sets are matched to the corresponding
anomalous sets in size and centre-cropped at native resolution to the largest
image dimensions in each test set. The DIV8K protocols are therefore
corpus-level rather than scene-paired comparisons and may also reflect
differences in image content and native resolution; in particular, crops from
8K images generally contain finer per-pixel detail than equally sized crops
from lower-resolution images.
\input{figs/fig_underwater}
\input{figs/fig_blurreds}
\subsubsection{Compared Models}
\label{sec:image_models}

\noindent\textbf{Autoencoder architecture (AE, AE-PP, AE-PoS):}\quad
The Stage-0 encoder maps the RGB input to 128 channels using a stride-\(2\)
patch convolution, applies four residual blocks, and projects the result to a
9-channel latent representation at half the input resolution. This gives a
latent-to-input dimensionality ratio of
\(9(H/2)(W/2)/(3HW)=0.75\). The decoder mirrors the encoder and restores the
original resolution using pixel shuffle. Each residual block consists of
LayerNorm, a \(1\times1\) expansion to twice the channel width, a
\(3\times3\) depthwise convolution, a ReLU-gated elementwise product between
the two channel halves, and a \(1\times1\) projection back to the original
width. All convolutions are bias-free, and the network uses neither attention
nor encoder--decoder skip connections. This is the architecture used in the AE
and AE-PP baselines. Stage~1 applies the same design to the
frozen 9-channel Stage-0 latent, using an overcomplete \(9\to27\to9\) mapping
without changing the spatial resolution.

\noindent\textbf{Optimisation:}\quad
Stage~0 is trained for \(250\mathrm{k}\) iterations using AdamW with a learning
rate of \(10^{-3}\), zero weight decay, a batch size of \(4\), and gradient
clipping at \(1.0\). We use a two-cycle cosine schedule with periods
\([150\mathrm{k},100\mathrm{k}]\) and relative weights \([1,0.5]\).
Stage~1 follows the same training schedule with a learning rate of \(10^{-4}\).

\noindent\textbf{VAE:}\quad
As a pretrained baseline, we use the frozen Stable Diffusion KL-\(f8\)
autoencoder~\citep{ldm} with the MSE-fine-tuned decoder released as
\texttt{sd-vae-ft-mse}~\citep{sdvaeftmse}. It is evaluated deterministically
through its 4-channel latent representation, without training on our data.
Because its decoder was fine-tuned using mean squared error, the checkpoint is
well matched to our reconstruction-error score. The VAE contains
\(83{,}653{,}863\) parameters, compared with \(830{,}720\) for AE and AE-PP
and \(1{,}665{,}280\) for the complete two-stage AE-PoS. It is therefore
approximately \(101\times\) larger than AE and \(50\times\) larger than
AE-PoS.

\subsection{Results}

Tables~\ref{tab:four-way} and~\ref{tab:reconstruction_psnr_full} show that
AE-PP and AE-PoS combine reliable detection of unseen degradations with high
nominal reconstruction fidelity. The visual comparisons in
Figures~\ref{fig:reds-blur-ramp}, \ref{fig:image-anomaly},
and~\ref{fig:underwater} explain the contrasting behaviours. Plain AE
preserves fine image details but also reconstructs degraded inputs, resulting
in weak anomaly separation. VAE has substantially lower fidelity on nominal
images and increasingly represents the low-frequency blurred component as
blur strengthens; under strong blur, it even reconstructs the degraded input
more accurately than its clean counterpart, reversing the desired anomaly-score
ordering. AE-PoS instead preserves nominal content while excluding unseen
degradations from its learned target set, causing them to appear as strong
reconstruction residuals.

This combination is particularly useful as a lightweight pre-detection stage
in mobile imaging pipelines. When most captured images are already of
sufficient quality, only those identified as degraded need to be passed to
computationally expensive restoration networks, reducing the expected latency
and energy consumption.

\input{figs/fig_mix_images}

%% file: tables/headline.tex
\begin{table}[t]
  \caption{Plain AE, AE-PP and AE-PoS on the one-class benchmarks.}
  \label{tab:headline}
  \begin{center}
  \setlength{\tabcolsep}{6pt}
  \begin{tabular}{l r r r}
  \toprule
  Dataset & AE & AE-PP & AE-PoS \\
  \midrule
  MNIST & 94.74$\pm$4.65 & \textbf{99.38$\pm$0.42} & 99.35$\pm$0.51 \\
  Fashion-MNIST & 91.60$\pm$5.34 & 94.83$\pm$3.95 & \textbf{94.87$\pm$4.04} \\
  CIFAR-10 & 68.96$\pm$7.29 & 71.30$\pm$6.70 & \textbf{71.73$\pm$6.60} \\
  CIFAR-100 & 65.42$\pm$9.09 & 68.56$\pm$8.89 & \textbf{68.87$\pm$8.97} \\
  Arrhythmia (F1) & 60.77$\pm$7.26 & 67.69$\pm$5.75 & \textbf{70.77$\pm$5.75} \\
  \bottomrule
  \end{tabular}
  \end{center}
\end{table}

%% file: tables/mnist.tex
\begin{table}[!htb]
  \caption{Mean and standard deviation of AUC-ROC in \% on MNIST~\citep{mnist} classes over 10 seeds. The highest value in each row is in bold.}
  \label{tab:oneclass-mnist}
  \begin{center}
  \setlength{\tabcolsep}{2pt}
  \resizebox{\columnwidth}{!}{%
  \begin{tabular}{lrrrrrrrrrrr@{\hspace{5pt}}|@{\hspace{5pt}}rr}
  \toprule
  Normal & OC-SVM & OCGAN & VAE & DCAE & IF & AnoGAN & KDE & DSVDD & MEMAE & MEMGAN & Toll & AE-PP & AE-PoS \\
  \midrule
  0 & 98.6$\pm$0.0 & 99.8 & 99.7 & 97.6$\pm$0.7 & 98.0$\pm$0.3 & 96.6$\pm$1.3 & 97.1$\pm$0.0 & 98.0$\pm$0.7 & 99.3$\pm$0.1 & 99.3$\pm$0.1 & 99.8$\pm$0.1 & \textbf{99.91$\pm$0.02} & 99.90$\pm$0.02 \\
  1 & 99.5$\pm$0.0 & 99.9 & 99.9 & 98.3$\pm$0.6 & 97.3$\pm$0.4 & 99.2$\pm$0.6 & 98.9$\pm$0.0 & 99.7$\pm$0.1 & 99.8$\pm$0.0 & 99.9$\pm$0.0 & 99.9$\pm$0.0 & 99.94$\pm$0.01 & \textbf{99.95$\pm$0.01} \\
  2 & 82.5$\pm$0.1 & 94.2 & 93.6 & 85.4$\pm$2.4 & 88.6$\pm$0.5 & 85.0$\pm$2.9 & 79.0$\pm$0.0 & 91.7$\pm$0.8 & 90.6$\pm$0.8 & 94.5$\pm$0.1 & 96.2$\pm$1.6 & 99.24$\pm$0.05 & \textbf{99.28$\pm$0.05} \\
  3 & 88.1$\pm$0.0 & 96.3 & 95.9 & 86.7$\pm$0.9 & 89.9$\pm$0.4 & 88.7$\pm$2.1 & 86.2$\pm$0.0 & 91.9$\pm$1.5 & 94.7$\pm$0.6 & 95.7$\pm$0.4 & 97.9$\pm$0.6 & 99.14$\pm$0.11 & \textbf{99.15$\pm$0.12} \\
  4 & 94.9$\pm$0.0 & 97.5 & 97.3 & 86.5$\pm$2.0 & 92.7$\pm$0.6 & 89.4$\pm$1.3 & 87.9$\pm$0.0 & 94.9$\pm$0.8 & 94.5$\pm$0.4 & 96.1$\pm$0.4 & 97.8$\pm$0.4 & \textbf{99.06$\pm$0.03} & 99.03$\pm$0.05 \\
  5 & 77.1$\pm$0.0 & 98.0 & 96.4 & 78.2$\pm$2.7 & 85.5$\pm$0.8 & 88.3$\pm$2.9 & 73.8$\pm$0.0 & 88.5$\pm$0.9 & 95.1$\pm$0.1 & 93.6$\pm$0.3 & 98.1$\pm$0.6 & 99.33$\pm$0.03 & \textbf{99.34$\pm$0.04} \\
  6 & 96.5$\pm$0.0 & 99.1 & 99.3 & 94.6$\pm$0.5 & 95.6$\pm$0.3 & 94.7$\pm$2.7 & 87.6$\pm$0.0 & 98.3$\pm$0.5 & 98.4$\pm$0.5 & 98.6$\pm$0.1 & 99.5$\pm$0.1 & \textbf{99.84$\pm$0.01} & \textbf{99.84$\pm$0.02} \\
  7 & 93.7$\pm$0.0 & 98.1 & 97.6 & 92.3$\pm$1.0 & 92.0$\pm$0.4 & 93.5$\pm$1.8 & 91.4$\pm$0.0 & 94.6$\pm$0.9 & 95.4$\pm$0.2 & 96.2$\pm$0.2 & 98.7$\pm$0.2 & \textbf{99.28$\pm$0.06} & \textbf{99.28$\pm$0.07} \\
  8 & 88.9$\pm$0.0 & 93.9 & 92.3 & 86.5$\pm$1.6 & 89.9$\pm$0.4 & 84.9$\pm$2.1 & 79.2$\pm$0.0 & 93.9$\pm$1.6 & 86.9$\pm$0.5 & 93.5$\pm$0.1 & 97.3$\pm$0.5 & \textbf{98.54$\pm$0.23} & 98.17$\pm$0.21 \\
  9 & 93.1$\pm$0.0 & 98.1 & 97.6 & 90.4$\pm$1.8 & 93.5$\pm$0.3 & 92.4$\pm$1.1 & 88.2$\pm$0.0 & 96.5$\pm$0.3 & 97.3$\pm$0.2 & 95.9$\pm$0.1 & 98.5$\pm$0.3 & 99.49$\pm$0.04 & \textbf{99.56$\pm$0.03} \\
  \midrule
  Average & 91.3 & 97.5 & 97.0 & 89.7 & 92.3 & 91.4 & 87.0 & 94.8 & 95.2 & 96.5 & 98.4 & \textbf{99.38} & 99.35 \\
  \bottomrule
  \end{tabular}}
  \end{center}
\end{table}

%% file: tables/fmnist.tex
\begin{table}[t]
\centering
\caption{One-class anomaly detection on
Fashion-MNIST~\citep{fashion_mnist}, measured by AUC-ROC (\%).
Values with uncertainty are reported as mean \(\pm\) standard deviation over
10 seeds. Bold indicates the highest mean in each row.}
\label{tab:oneclass-fmnist}

\begingroup
\footnotesize
\setlength{\tabcolsep}{2.5pt}
\renewcommand{\arraystretch}{1.05}

\resizebox{\columnwidth}{!}{%
\begin{tabular}{@{}l*{9}{c}@{}}
\toprule
&
\multicolumn{7}{c}{Baselines} &
\multicolumn{2}{c}{Ours} \\
\cmidrule(lr){2-8}
\cmidrule(lr){9-10}
Nominal
& DAGMM
& DSEBM
& ADGAN
& GANomaly
& GeoTrans
& ARNet
& TOLL
& AE-PP
& AE-PoS \\
\midrule

T-shirt/top
& 42.1
& 91.6
& 89.9
& 80.3
& \textbf{99.4}
& 92.7
& \(91.9{\pm}0.2\)
& \(95.12{\pm}0.09\)
& \(94.92{\pm}0.10\) \\

Trouser
& 55.1
& 71.8
& 81.9
& 83.0
& 97.6
& \textbf{99.3}
& \(98.7{\pm}0.0\)
& \(98.91{\pm}0.06\)
& \(99.24{\pm}0.08\) \\

Pullover
& 50.4
& 88.3
& 87.6
& 75.9
& 91.1
& 89.1
& \(91.3{\pm}0.2\)
& \(\mathbf{92.47{\pm}0.20}\)
& \(92.14{\pm}0.12\) \\

Dress
& 57.0
& 87.3
& 91.2
& 87.2
& 89.9
& 93.6
& \(92.4{\pm}0.2\)
& \(95.24{\pm}0.19\)
& \(\mathbf{95.40{\pm}0.11}\) \\

Coat
& 26.9
& 85.2
& 86.5
& 71.4
& 92.1
& 90.8
& \(91.5{\pm}0.1\)
& \(92.92{\pm}0.22\)
& \(\mathbf{93.25{\pm}0.27}\) \\

Sandal
& 70.5
& 87.1
& 89.6
& 92.7
& 93.4
& 93.1
& \(93.1{\pm}0.2\)
& \(\mathbf{94.28{\pm}0.39}\)
& \(93.07{\pm}0.46\) \\

Shirt
& 48.3
& 73.4
& 74.3
& 81.0
& 83.3
& 85.0
& \(84.6{\pm}0.3\)
& \(85.13{\pm}0.20\)
& \(\mathbf{85.37{\pm}0.14}\) \\

Sneaker
& 83.5
& 98.1
& 97.2
& 88.3
& 98.9
& 98.4
& \(99.0{\pm}0.0\)
& \(\mathbf{99.15{\pm}0.09}\)
& \(99.13{\pm}0.03\) \\

Bag
& 49.9
& 86.0
& 89.0
& 69.3
& 90.8
& \textbf{97.8}
& \(90.8{\pm}0.6\)
& \(96.56{\pm}0.48\)
& \(97.46{\pm}0.13\) \\

Ankle boot
& 34.0
& 97.1
& 97.1
& 80.3
& \textbf{99.2}
& 98.4
& \(98.9{\pm}0.1\)
& \(98.51{\pm}0.12\)
& \(98.71{\pm}0.08\) \\

\midrule
Average
& 51.8
& 86.6
& 88.4
& 80.9
& 93.5
& 93.9
& 93.2
& 94.83
& \textbf{94.87} \\
\bottomrule
\end{tabular}%
}

\endgroup
\end{table}

%% file: tables/cifar10.tex
\begin{table}[h]
\centering
\caption{One-class anomaly detection on CIFAR-10~\citep{cifar}, measured by
AUC-ROC (\%). AE-PP and AE-PoS results are reported as mean
\(\pm\) standard deviation over 10 seeds. Bold indicates the highest mean in
each row.}
\label{tab:oneclass-cifar10}

\begingroup
\footnotesize
\setlength{\tabcolsep}{2.5pt}
\renewcommand{\arraystretch}{1.05}

\resizebox{\columnwidth}{!}{%
\begin{tabular}{@{}l*{13}{r}@{}}
\toprule
&
\multicolumn{11}{c}{Baselines} &
\multicolumn{2}{c}{Ours} \\
\cmidrule(lr){2-12}
\cmidrule(lr){13-14}
Nominal
& \multicolumn{1}{c}{OCGAN}
& \multicolumn{1}{c}{VAE}
& \multicolumn{1}{c}{DSVDD}
& \multicolumn{1}{c}{DSEBM}
& \multicolumn{1}{c}{DAGMM}
& \multicolumn{1}{c}{IF}
& \multicolumn{1}{c}{AnoGAN}
& \multicolumn{1}{c}{ALAD}
& \multicolumn{1}{c}{MEMAE}
& \multicolumn{1}{c}{MEMGAN}
& \multicolumn{1}{c}{TOLL}
& \multicolumn{1}{c}{AE-PP}
& \multicolumn{1}{c}{AE-PoS} \\
\midrule

Airplane
& 75.7
& 70.0
& \(61.7{\pm}4.1\)
& \(41.4{\pm}2.3\)
& \(56.0{\pm}6.9\)
& \(60.1{\pm}0.7\)
& \(67.1{\pm}2.5\)
& \(64.7{\pm}2.6\)
& \(66.5{\pm}0.9\)
& \(73.0{\pm}0.8\)
& \(75.0{\pm}4.3\)
& \(75.53{\pm}2.18\)
& \(\mathbf{75.99{\pm}1.94}\) \\

Automobile
& 53.1
& 38.6
& \(65.9{\pm}2.1\)
& \(57.1{\pm}2.0\)
& \(48.3{\pm}1.8\)
& \(50.8{\pm}0.6\)
& \(54.7{\pm}3.4\)
& \(38.7{\pm}0.8\)
& \(46.4{\pm}0.1\)
& \(52.5{\pm}0.7\)
& \(67.6{\pm}3.2\)
& \(68.65{\pm}1.46\)
& \(\mathbf{69.10{\pm}1.32}\) \\

Bird
& 64.0
& \textbf{67.9}
& \(50.8{\pm}0.8\)
& \(61.9{\pm}0.1\)
& \(53.8{\pm}4.0\)
& \(49.2{\pm}0.4\)
& \(52.9{\pm}3.0\)
& \(67.0{\pm}0.7\)
& \(66.0{\pm}0.1\)
& \(67.2{\pm}0.1\)
& \(56.2{\pm}2.8\)
& \(59.47{\pm}1.95\)
& \(59.99{\pm}1.94\) \\

Cat
& 62.0
& 53.5
& \(59.1{\pm}1.4\)
& \(50.1{\pm}0.4\)
& \(51.2{\pm}0.8\)
& \(55.1{\pm}0.4\)
& \(54.5{\pm}1.9\)
& \(59.2{\pm}0.3\)
& \(52.9{\pm}0.1\)
& \(57.3{\pm}0.2\)
& \(\mathbf{64.6{\pm}1.6}\)
& \(63.41{\pm}0.47\)
& \(63.91{\pm}0.58\) \\

Deer
& 72.3
& \textbf{74.8}
& \(60.9{\pm}1.1\)
& \(73.3{\pm}0.2\)
& \(52.2{\pm}7.3\)
& \(49.8{\pm}0.4\)
& \(65.1{\pm}3.2\)
& \(72.7{\pm}0.6\)
& \(72.8{\pm}0.1\)
& \(73.9{\pm}0.9\)
& \(67.7{\pm}1.4\)
& \(69.47{\pm}0.81\)
& \(70.01{\pm}0.99\) \\

Dog
& 62.0
& 52.3
& \(65.7{\pm}0.8\)
& \(60.5{\pm}0.3\)
& \(49.3{\pm}3.6\)
& \(58.5{\pm}0.4\)
& \(60.3{\pm}2.6\)
& \(52.8{\pm}1.2\)
& \(52.9{\pm}0.2\)
& \(65.0{\pm}0.2\)
& \(68.1{\pm}1.6\)
& \(68.85{\pm}0.71\)
& \(\mathbf{69.36{\pm}0.43}\) \\

Frog
& 72.3
& 68.7
& \(67.7{\pm}2.6\)
& \(68.4{\pm}0.3\)
& \(64.9{\pm}1.7\)
& \(42.9{\pm}0.6\)
& \(58.5{\pm}1.4\)
& \(69.5{\pm}1.1\)
& \(63.7{\pm}0.4\)
& \(72.8{\pm}0.7\)
& \(74.7{\pm}2.5\)
& \(76.77{\pm}0.51\)
& \(\mathbf{77.20{\pm}0.75}\) \\

Horse
& 57.5
& 49.3
& \(67.3{\pm}0.9\)
& \(53.3{\pm}0.7\)
& \(55.3{\pm}0.8\)
& \(55.1{\pm}0.7\)
& \(62.5{\pm}0.8\)
& \(44.8{\pm}0.4\)
& \(45.9{\pm}0.1\)
& \(52.5{\pm}0.5\)
& \(65.8{\pm}1.3\)
& \(71.24{\pm}1.47\)
& \(\mathbf{71.71{\pm}1.42}\) \\

Ship
& 82.0
& 69.6
& \(75.9{\pm}1.2\)
& \(73.9{\pm}0.3\)
& \(51.9{\pm}2.4\)
& \(74.2{\pm}0.6\)
& \(75.8{\pm}4.1\)
& \(73.4{\pm}0.4\)
& \(70.1{\pm}0.1\)
& \(74.4{\pm}0.3\)
& \(80.0{\pm}1.1\)
& \(83.12{\pm}0.63\)
& \(\mathbf{83.33{\pm}0.62}\) \\

Truck
& 55.4
& 38.6
& \(73.1{\pm}1.2\)
& \(63.6{\pm}3.1\)
& \(54.2{\pm}5.8\)
& \(58.9{\pm}0.7\)
& \(66.5{\pm}2.8\)
& \(39.2{\pm}1.3\)
& \(48.2{\pm}0.2\)
& \(65.6{\pm}1.6\)
& \(76.2{\pm}0.5\)
& \(76.53{\pm}0.51\)
& \(\mathbf{76.69{\pm}0.60}\) \\

\midrule
Average
& 65.6
& \multicolumn{1}{c}{58.3}
& \multicolumn{1}{c}{64.8}
& \multicolumn{1}{c}{60.4}
& \multicolumn{1}{c}{54.4}
& \multicolumn{1}{c}{55.5}
& \multicolumn{1}{c}{61.8}
& \multicolumn{1}{c}{59.3}
& \multicolumn{1}{c}{58.5}
& \multicolumn{1}{c}{65.3}
& \multicolumn{1}{c}{69.6}
& \multicolumn{1}{c}{71.30}
& \multicolumn{1}{c}{\textbf{71.73}} \\
\bottomrule
\end{tabular}%
}

\endgroup
\end{table}

%% file: tables/cifar100.tex
\begin{table}[!h]
\centering
\caption{One-class anomaly detection on the CIFAR-100 coarse
superclasses~\citep{cifar}, measured by AUC-ROC (\%). Values with uncertainty
are reported as mean \(\pm\) standard deviation over 10 seeds. Bold indicates
the highest mean in each row.}
\label{tab:oneclass-cifar100}

\begingroup
\footnotesize
\setlength{\tabcolsep}{2.5pt}
\renewcommand{\arraystretch}{1.05}

\resizebox{\columnwidth}{!}{%
\begin{tabular}{@{}l*{9}{r}@{}}
\toprule
&
\multicolumn{7}{c}{Baselines} &
\multicolumn{2}{c}{Ours} \\
\cmidrule(lr){2-8}
\cmidrule(lr){9-10}
Nominal
& \multicolumn{1}{c}{DAGMM}
& \multicolumn{1}{c}{DSEBM}
& \multicolumn{1}{c}{ADGAN}
& \multicolumn{1}{c}{GANomaly}
& \multicolumn{1}{c}{GeoTrans}
& \multicolumn{1}{c}{ARNet}
& \multicolumn{1}{c}{TOLL}
& \multicolumn{1}{c}{AE-PP}
& \multicolumn{1}{c}{AE-PoS} \\
\midrule

Aquatic mammals
& 43.4
& 64.0
& 63.1
& 57.9
& 74.7
& \textbf{77.5}
& \(70.3{\pm}3.4\)
& \(69.30{\pm}1.55\)
& \(69.62{\pm}1.07\) \\

Fish
& 49.5
& 47.9
& 54.9
& 51.9
& 68.5
& 70.0
& \(66.2{\pm}1.1\)
& \(69.99{\pm}1.49\)
& \(\mathbf{70.38{\pm}1.41}\) \\

Flowers
& 66.1
& 53.7
& 41.3
& 36.0
& 74.0
& 62.4
& \(80.3{\pm}0.7\)
& \(80.13{\pm}1.06\)
& \(\mathbf{80.85{\pm}1.06}\) \\

Food containers
& 52.6
& 48.4
& 50.0
& 46.5
& \textbf{81.0}
& 76.2
& \(58.2{\pm}1.5\)
& \(66.80{\pm}3.91\)
& \(67.27{\pm}3.95\) \\

Fruit and vegetables
& 56.9
& 59.7
& 40.6
& 46.6
& 78.4
& 77.7
& \(80.1{\pm}1.3\)
& \(79.78{\pm}0.72\)
& \(\mathbf{80.36{\pm}0.62}\) \\

Household electrical devices
& 52.4
& 46.6
& 42.8
& 42.9
& 59.1
& \textbf{64.0}
& \(61.4{\pm}2.6\)
& \(61.56{\pm}1.92\)
& \(61.56{\pm}1.89\) \\

Household furniture
& 55.0
& 51.7
& 51.1
& 53.7
& 81.8
& \textbf{86.9}
& \(64.8{\pm}3.0\)
& \(69.64{\pm}1.51\)
& \(69.78{\pm}1.44\) \\

Insects
& 52.8
& 54.8
& 55.4
& 59.4
& 65.0
& 65.6
& \(70.5{\pm}1.3\)
& \(72.51{\pm}0.93\)
& \(\mathbf{73.16{\pm}1.37}\) \\

Large carnivores
& 53.2
& 66.7
& 59.2
& 63.7
& \textbf{85.5}
& 82.7
& \(62.1{\pm}2.1\)
& \(59.37{\pm}1.98\)
& \(59.59{\pm}1.97\) \\

Large man-made outdoor things
& 42.5
& 71.2
& 62.7
& 68.0
& \textbf{90.6}
& 90.2
& \(79.6{\pm}1.4\)
& \(81.53{\pm}1.19\)
& \(81.99{\pm}0.98\) \\

Large natural outdoor scenes
& 52.7
& 78.3
& 79.8
& 75.6
& \textbf{87.6}
& 85.9
& \(81.6{\pm}1.1\)
& \(83.89{\pm}1.02\)
& \(84.15{\pm}0.98\) \\

Large omnivores and herbivores
& 46.4
& 62.7
& 53.7
& 57.6
& \textbf{83.9}
& 83.5
& \(62.6{\pm}1.2\)
& \(61.04{\pm}0.84\)
& \(61.19{\pm}0.77\) \\

Medium-sized mammals
& 42.7
& 66.8
& 58.9
& 58.7
& 83.2
& \textbf{84.6}
& \(61.2{\pm}1.9\)
& \(61.21{\pm}1.06\)
& \(61.40{\pm}1.33\) \\

Non-insect invertebrates
& 45.4
& 52.6
& 57.4
& 59.9
& 58.0
& \textbf{67.6}
& \(58.6{\pm}2.0\)
& \(55.65{\pm}2.03\)
& \(55.95{\pm}2.18\) \\

People
& 57.2
& 44.0
& 39.4
& 43.9
& \textbf{92.1}
& 84.2
& \(65.4{\pm}4.5\)
& \(65.01{\pm}1.74\)
& \(65.11{\pm}1.75\) \\

Reptiles
& 48.8
& 56.8
& 55.6
& 59.9
& 68.3
& \textbf{74.1}
& \(56.1{\pm}1.5\)
& \(53.48{\pm}1.68\)
& \(53.77{\pm}1.67\) \\

Small mammals
& 54.4
& 63.1
& 63.3
& 64.4
& 73.5
& \textbf{80.3}
& \(62.3{\pm}1.1\)
& \(63.58{\pm}1.20\)
& \(63.94{\pm}1.27\) \\

Trees
& 36.4
& 73.0
& 66.7
& 71.8
& \textbf{93.8}
& 91.0
& \(80.9{\pm}1.8\)
& \(81.19{\pm}0.82\)
& \(81.31{\pm}0.69\) \\

Vehicles 1
& 52.4
& 57.7
& 44.3
& 54.9
& \textbf{90.7}
& 85.3
& \(65.2{\pm}1.4\)
& \(65.77{\pm}1.19\)
& \(66.20{\pm}1.45\) \\

Vehicles 2
& 50.3
& 55.5
& 53.0
& 56.8
& 85.0
& \textbf{85.4}
& \(63.6{\pm}1.6\)
& \(69.82{\pm}1.03\)
& \(69.85{\pm}1.04\) \\

\midrule
Average
& 50.5
& 58.8
& 54.7
& 56.5
& 78.7
& \textbf{78.8}
& 67.6
& 68.56
& 68.87 \\
\bottomrule
\end{tabular}%
}

\endgroup
\end{table}

%% file: tables/sensitivity.tex
\begin{table*}[t]
  \centering
  \caption{Effect of the Push--Pull margin scale \(\alpha\).
  AE is the reconstruction-only baseline obtained using the common TOLL
  backbone with its compactness coefficient set to \(\beta=0\).
  All AE-PP configurations use \(w=1\); \(\alpha=0\) retains Pull while
  disabling Push. Results are reported as mean \(\pm\) standard deviation.
  Bold indicates the highest mean in each row, and the dagger marks the
  value of \(\alpha\) selected on the validation set.}
  \label{tab:pp-alpha-sweep}

  \begingroup
  \scriptsize
  \setlength{\tabcolsep}{2.5pt}
  \renewcommand{\arraystretch}{1.05}

  \resizebox{\textwidth}{!}{%
  \begin{tabular}{@{}l c ccccccc@{}}
    \toprule
    &
    \multicolumn{1}{c}{Baseline} &
    \multicolumn{7}{c}{AE-PP (\(w=1\))} \\
    \cmidrule(lr){2-2}
    \cmidrule(lr){3-9}
    Dataset
    & AE (\(\beta=0\))
    & \(\alpha=0\)
    & \(\alpha=1\)
    & \(\alpha=2\)
    & \(\alpha=3\)
    & \(\alpha=4\)
    & \(\alpha=5\)
    & \(\alpha=6\) \\
    \midrule

    Arrhythmia (F1)
    & \(60.77{\pm}7.26\)
    & \(63.08{\pm}3.24\)
    & \(63.08{\pm}3.24\)
    & \(65.38{\pm}4.05\)
    & \(\mathbf{67.69{\pm}6.07}^{\dagger}\)
    & \(\mathbf{67.69{\pm}6.07}\)
    & \(\mathbf{67.69{\pm}6.07}\)
    & \(\mathbf{67.69{\pm}6.07}\) \\

    MNIST
    & \(94.74{\pm}4.65\)
    & \(98.54{\pm}1.38\)
    & \(99.04{\pm}0.77\)
    & \(99.37{\pm}0.44\)
    & \(\mathbf{99.38{\pm}0.42}^{\dagger}\)
    & \(99.32{\pm}0.45\)
    & \(99.28{\pm}0.48\)
    & \(99.26{\pm}0.50\) \\

    Fashion-MNIST
    & \(91.60{\pm}5.34\)
    & \(93.75{\pm}4.60\)
    & \(94.12{\pm}4.43\)
    & \(94.74{\pm}4.02\)
    & \(\mathbf{94.83{\pm}3.97}^{\dagger}\)
    & \(94.74{\pm}3.99\)
    & \(94.67{\pm}3.94\)
    & \(94.69{\pm}3.95\) \\

    CIFAR-10
    & \(68.96{\pm}7.29\)
    & \(69.32{\pm}7.36\)
    & \(69.48{\pm}7.13\)
    & \(\mathbf{71.30{\pm}6.73}^{\dagger}\)
    & \(71.16{\pm}6.81\)
    & \(71.05{\pm}6.78\)
    & \(71.23{\pm}6.97\)
    & \(71.29{\pm}7.06\) \\

    CIFAR-100
    & \(65.42{\pm}9.09\)
    & \(66.78{\pm}9.26\)
    & \(67.18{\pm}9.17\)
    & \(68.56{\pm}8.91^{\dagger}\)
    & \(68.45{\pm}8.98\)
    & \(68.59{\pm}8.69\)
    & \(\mathbf{68.69{\pm}8.53}\)
    & \(68.57{\pm}8.72\) \\

    \bottomrule
  \end{tabular}%
  }

  \endgroup
\end{table*}

%% file: tables/image.tex
\begin{table}[h]
\centering
\caption{Detection of unseen image degradations. The clean reference defines
the nominal class: paired clean images (Own GT) are used when available;
otherwise, a fixed set of clean DIV8K crops is used. AUC-ROC and PR-AUC are
computed using reconstruction error as the anomaly score. Bold indicates the
best result in each row and metric.}
\label{tab:four-way}

\begingroup
\small
\setlength{\tabcolsep}{3pt}
\renewcommand{\arraystretch}{1.0}

\begin{tabular}{@{}llr rr rr rr rr@{}}
\toprule
\multirow{2}{*}{Dataset} &
\multirow{2}{*}{\shortstack{Clean\\reference}} &
\multirow{2}{*}{\shortstack{\#\\Images}} &
\multicolumn{2}{c}{AE} &
\multicolumn{2}{c}{VAE} &
\multicolumn{2}{c}{AE-PP} &
\multicolumn{2}{c}{AE-PoS} \\
\cmidrule(lr){4-5}
\cmidrule(lr){6-7}
\cmidrule(lr){8-9}
\cmidrule(lr){10-11}
& & &
ROC & PR &
ROC & PR &
ROC & PR &
ROC & PR \\
\midrule

\multicolumn{11}{@{}l}{\textit{Additive noise}} \\
Kodak24 ($\sigma=25$)
& Own GT & 24
& \textbf{1.0000} & \textbf{1.0000}
& 0.9983 & 0.9983
& \textbf{1.0000} & \textbf{1.0000}
& \textbf{1.0000} & \textbf{1.0000} \\
\addlinespace[2pt]

\multicolumn{11}{@{}l}{\textit{Blur}} \\
REDS blur 0.2
& Own GT & 200
& \textbf{0.7268} & \textbf{0.6589}
& 0.3552 & 0.4024
& 0.6083 & 0.5725
& 0.6002 & 0.5766 \\

REDS blur 0.4
& Own GT & 200
& 0.6581 & 0.5683
& 0.2168 & 0.3484
& 0.7597 & 0.7239
& \textbf{0.7808} & \textbf{0.7574} \\

REDS blur 0.6
& Own GT & 200
& 0.6726 & 0.5699
& 0.1111 & 0.3245
& 0.9211 & 0.9222
& \textbf{0.9554} & \textbf{0.9567} \\

REDS blur 0.8
& Own GT & 200
& 0.7703 & 0.6713
& 0.0528 & 0.3151
& 0.9920 & 0.9930
& \textbf{0.9972} & \textbf{0.9975} \\

REDS blur 1.0
& Own GT & 200
& 0.5724 & 0.5215
& 0.0346 & 0.3119
& 0.9985 & 0.9986
& \textbf{0.9996} & \textbf{0.9996} \\

GoPro
& Own GT & 200
& 0.4114 & 0.4217
& 0.0569 & 0.3137
& 0.9471 & 0.9568
& \textbf{0.9846} & \textbf{0.9862} \\

GoPro
& DIV8K & 200
& 0.1621 & 0.3378
& 0.1018 & 0.3231
& 0.9983 & 0.9982
& \textbf{0.9997} & \textbf{0.9997} \\
\addlinespace[2pt]

\multicolumn{11}{@{}l}{\textit{Rain}} \\
Test100
& DIV8K & 98
& 0.3284 & 0.3863
& 0.7842 & 0.6766
& 0.9481 & 0.9418
& \textbf{0.9594} & \textbf{0.9470} \\

Rain100L
& DIV8K & 100
& 0.4709 & 0.4461
& 0.7467 & 0.6432
& 0.8708 & 0.8366
& \textbf{0.8925} & \textbf{0.8479} \\

Rain100H
& DIV8K & 100
& 0.8196 & 0.7307
& 0.9051 & 0.8269
& 0.9436 & 0.9131
& \textbf{0.9570} & \textbf{0.9226} \\

Test1200
& DIV8K & 200
& 0.2025 & 0.3458
& 0.7151 & 0.5944
& 0.9885 & 0.9864
& \textbf{0.9919} & \textbf{0.9881} \\

Test2800
& DIV8K & 200
& 0.3801 & 0.4016
& 0.8017 & 0.6917
& 0.9798 & 0.9764
& \textbf{0.9879} & \textbf{0.9825} \\
\addlinespace[2pt]

\multicolumn{11}{@{}l}{\textit{Underwater}} \\
Underwater high
& DIV8K & 144
& 0.3673 & 0.4010
& 0.5254 & 0.4684
& 0.9846 & 0.9848
& \textbf{0.9891} & \textbf{0.9888} \\

Underwater medium
& DIV8K & 143
& 0.3555 & 0.3974
& 0.5308 & 0.4769
& 0.9750 & 0.9753
& \textbf{0.9866} & \textbf{0.9857} \\

Underwater low
& DIV8K & 143
& 0.1630 & 0.3359
& 0.2918 & 0.3729
& 0.9412 & 0.9246
& \textbf{0.9798} & \textbf{0.9677} \\
\bottomrule
\end{tabular}

\endgroup
\end{table}

%% file: tables/image_psnr_full.tex
\begin{table}[t]
\centering
\caption{Average reconstruction PSNR (dB), computed between each input \(s\)
and its reconstruction \(\hat{s}\), for clean and degraded images. Higher PSNR
indicates lower reconstruction error. The \# Images column reports the number
of images in each evaluation setting. ``Own GT'' denotes the paired clean
ground truth associated with each degraded image, whereas ``DIV8K'' denotes
the fixed clean DIV8K evaluation subset.}
\label{tab:reconstruction_psnr_full}

\begingroup
\small
\setlength{\tabcolsep}{4.2pt}
\renewcommand{\arraystretch}{1.0}

\resizebox{\columnwidth}{!}{%
\begin{tabular}{@{}llr rr rr rr rr@{}}
\toprule
\multirow{2}{*}{Dataset}
& \multirow{2}{*}{\shortstack{Clean\\reference}}
& \multirow{2}{*}{\shortstack{\#\\Images}}
& \multicolumn{2}{c}{AE}
& \multicolumn{2}{c}{VAE}
& \multicolumn{2}{c}{AE-PP}
& \multicolumn{2}{c}{AE-PoS} \\
\cmidrule(lr){4-5}
\cmidrule(lr){6-7}
\cmidrule(lr){8-9}
\cmidrule(lr){10-11}
& & &
Clean & Degr.
& Clean & Degr.
& Clean & Degr.
& Clean & Degr. \\
\midrule

\addlinespace[1pt]
\multicolumn{11}{@{}l}{\itshape Additive noise} \\
Kodak24 (\(\sigma=25\))
& Own GT & 24
& 57.15 & 26.49
& 26.49 & 18.86
& 44.58 & 23.64
& 43.65 & 23.55 \\

\addlinespace[1pt]
\multicolumn{11}{@{}l}{\itshape Synthetic blur} \\
REDS blur 0.2
& Own GT & 200
& 63.04 & 62.06
& 27.59 & 29.11
& 43.42 & 42.55
& 43.07 & 42.38 \\
REDS blur 0.4
& Own GT & 200
& 63.04 & 62.41
& 27.59 & 30.84
& 43.42 & 41.00
& 43.07 & 40.88 \\
REDS blur 0.6
& Own GT & 200
& 63.04 & 62.37
& 27.59 & 32.85
& 43.42 & 38.12
& 43.07 & 37.97 \\
REDS blur 0.8
& Own GT & 200
& 63.04 & 61.95
& 27.59 & 35.22
& 43.42 & 33.26
& 43.07 & 33.60 \\
REDS blur 1.0
& Own GT & 200
& 63.04 & 62.72
& 27.59 & 37.01
& 43.42 & 30.51
& 43.07 & 31.79 \\

\addlinespace[1pt]
\multicolumn{11}{@{}l}{\itshape Motion blur} \\
GoPro
& Own GT & 200
& 62.66 & 63.12
& 31.15 & 39.17
& 33.80 & 28.11
& 34.50 & 28.55 \\
GoPro
& DIV8K & 200
& 59.14 & 63.12
& 29.58 & 39.17
& 44.28 & 28.11
& 43.54 & 28.55 \\

\addlinespace[1pt]
\multicolumn{11}{@{}l}{\itshape Rain} \\
Test100
& DIV8K & 98
& 58.43 & 61.11
& 29.86 & 23.51
& 44.59 & 34.35
& 43.53 & 34.58 \\
Rain100L
& DIV8K & 100
& 58.73 & 59.85
& 30.53 & 24.54
& 44.89 & 38.60
& 43.77 & 38.30 \\
Rain100H
& DIV8K & 100
& 58.73 & 54.50
& 30.53 & 20.75
& 44.89 & 37.08
& 43.77 & 36.77 \\
Test1200
& DIV8K & 200
& 59.41 & 63.13
& 30.33 & 25.31
& 44.67 & 31.72
& 43.77 & 32.80 \\
Test2800
& DIV8K & 200
& 59.44 & 60.93
& 30.22 & 23.47
& 44.58 & 32.12
& 43.69 & 32.90 \\

\addlinespace[1pt]
\multicolumn{11}{@{}l}{\itshape Underwater} \\
Underwater high
& DIV8K & 144
& 58.91 & 61.03
& 29.64 & 29.01
& 44.19 & 29.61
& 43.43 & 30.76 \\
Underwater medium
& DIV8K & 143
& 58.91 & 61.00
& 29.53 & 29.01
& 44.10 & 30.07
& 43.42 & 31.06 \\
Underwater low
& DIV8K & 143
& 58.91 & 63.02
& 29.53 & 33.71
& 44.10 & 35.30
& 43.42 & 33.94 \\
\bottomrule
\end{tabular}%
}

\endgroup
\end{table}

%% file: figs/fig_underwater.tex
\ifdefined\wsm\else
  \newlength{\wsm}\newlength{\wbg}\newlength{\wgap}\newlength{\wsh}\newlength{\wbh}
  \newlength{\wlab}\newlength{\wlabn}
  \newcommand{\wS}[1]{\includegraphics[width=\wsm]{#1}}
  \newcommand{\wB}[1]{\includegraphics[width=\wbg]{#1}}
  \newcommand{\wSv}[2]{\begin{tikzpicture}[inner sep=0]
      \node (im) {\wS{#1}};
      \node[anchor=south east, inner sep=1.2pt, fill=white, font=\footnotesize]
        at (im.south east) {#2};
    \end{tikzpicture}}
  \newcommand{\wrl}[1]{\makebox[\wlab][c]{\rotatebox{90}{\makebox[\wsh][c]{\large #1}}}}
  \newcommand{\wrb}[1]{\makebox[\wlab][c]{\rotatebox{90}{\makebox[\wbh][c]{\normalsize #1}}}}
  \newcommand{\wrw}[1]{\makebox[\wlabn][c]{\rotatebox{90}{\makebox[\wsh][c]{\normalsize #1}}}}
  \newcommand{\wrn}[1]{\makebox[\wlabn][c]{\smash{\rotatebox{90}{%
      \makebox[\dimexpr 2\wsh + \wgap\relax][c]{\normalsize #1}}}}}
  \newcommand{\wrB}[1]{\makebox[\wlab][c]{\rotatebox{90}{\makebox[\wbh][c]{\large #1}}}}
\fi
\setlength{\wgap}{2pt}
\setlength{\wlab}{1.6em}
\setlength{\wsm}{\dimexpr (\textwidth - \wlab - 6.0000\wgap) / 6\relax}
\setlength{\wbg}{\dimexpr 2\wsm + \wgap\relax}
\setlength{\wsh}{0.7478\wsm}\setlength{\wbh}{0.7478\wbg}
\begin{figure}[t]
\begin{center}
  \setlength{\tabcolsep}{0.5\wgap}\renewcommand{\arraystretch}{0}
  \begin{tabular}{@{}c cc cc cc@{}}
     & \multicolumn{2}{@{}c@{}}{\normalsize low} & \multicolumn{2}{@{}c@{}}{\normalsize medium} & \multicolumn{2}{@{}c@{}}{\normalsize high} \\[2pt]
     \wrB{$s$} & \multicolumn{2}{@{}c@{}}{\wB{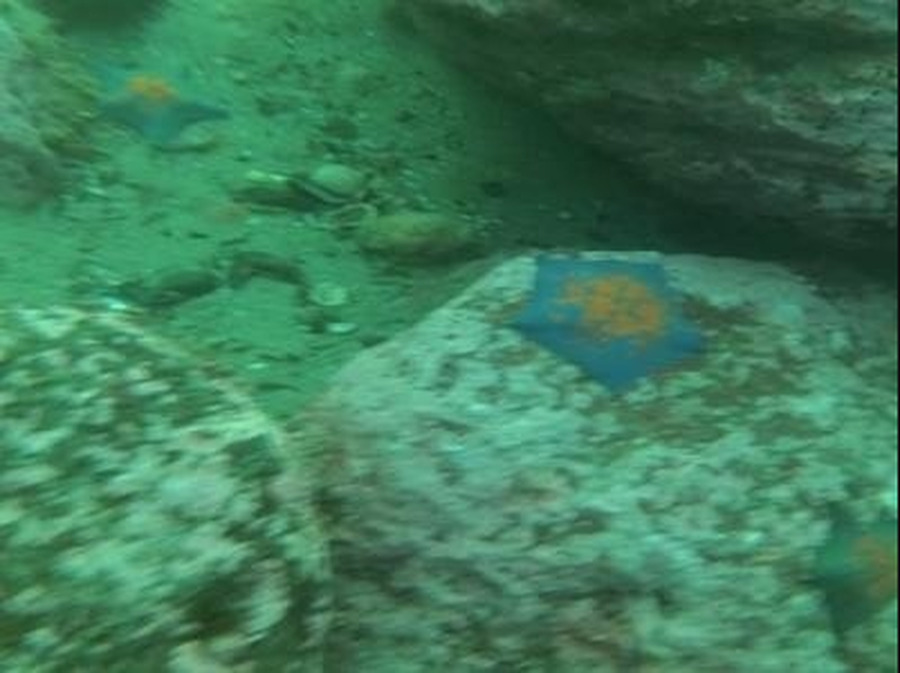}} & \multicolumn{2}{@{}c@{}}{\wB{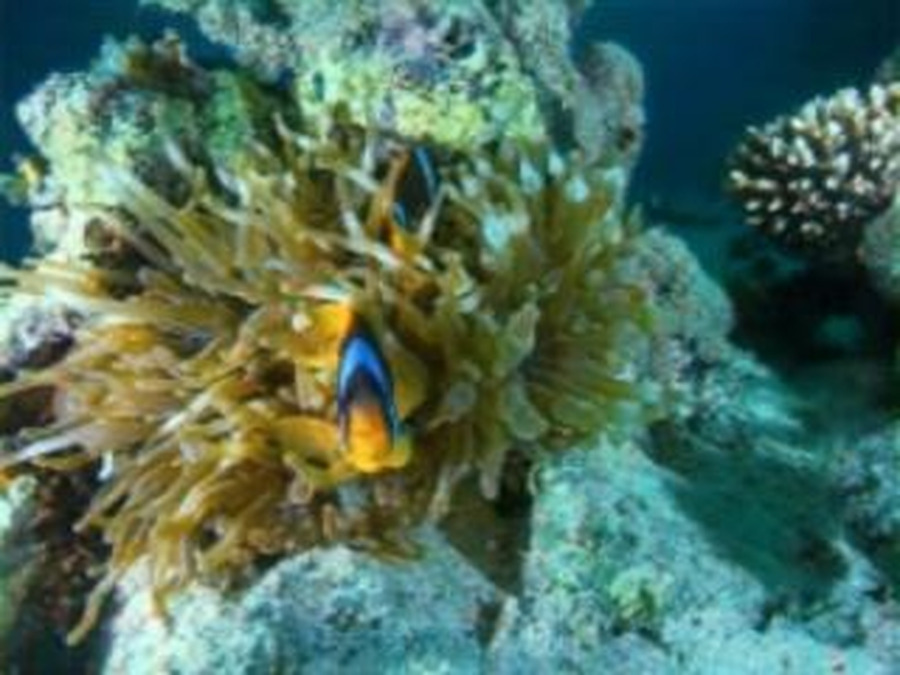}} & \multicolumn{2}{@{}c@{}}{\wB{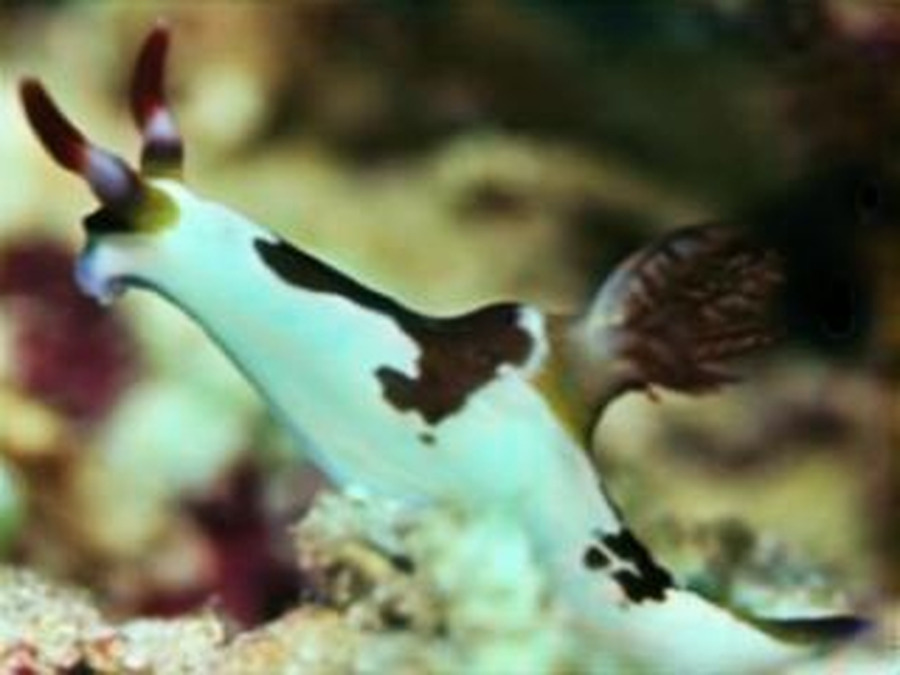}} \\[3pt]
     & {\large $\hat{s}$} & {\large $|s-\hat{s}|$} & {\large $\hat{s}$} & {\large $|s-\hat{s}|$} & {\large $\hat{s}$} & {\large $|s-\hat{s}|$} \\[6pt]
     \wrw{AE} & \wSv{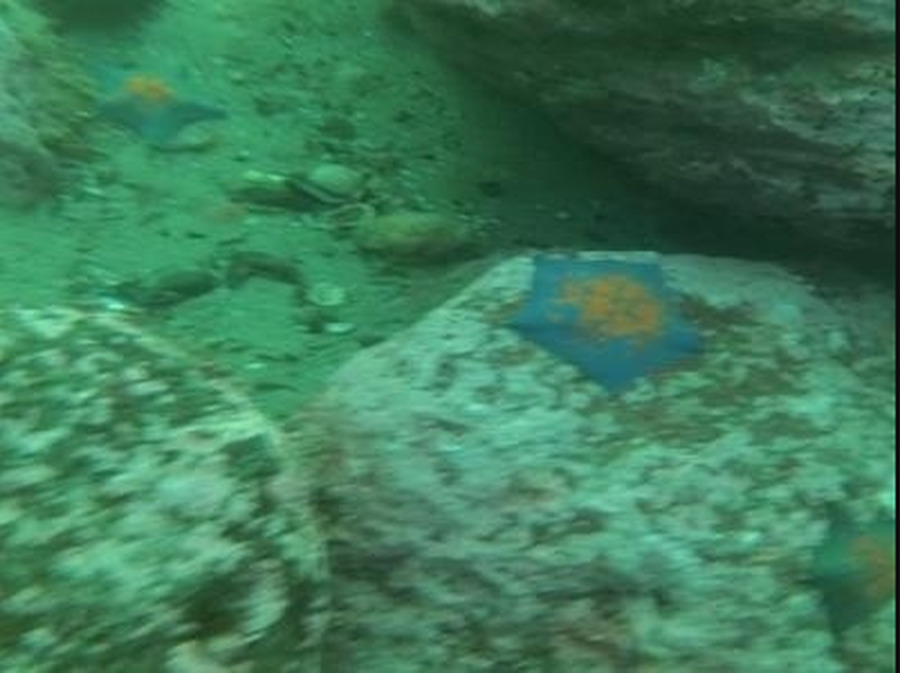}{61.8\,dB} & \wS{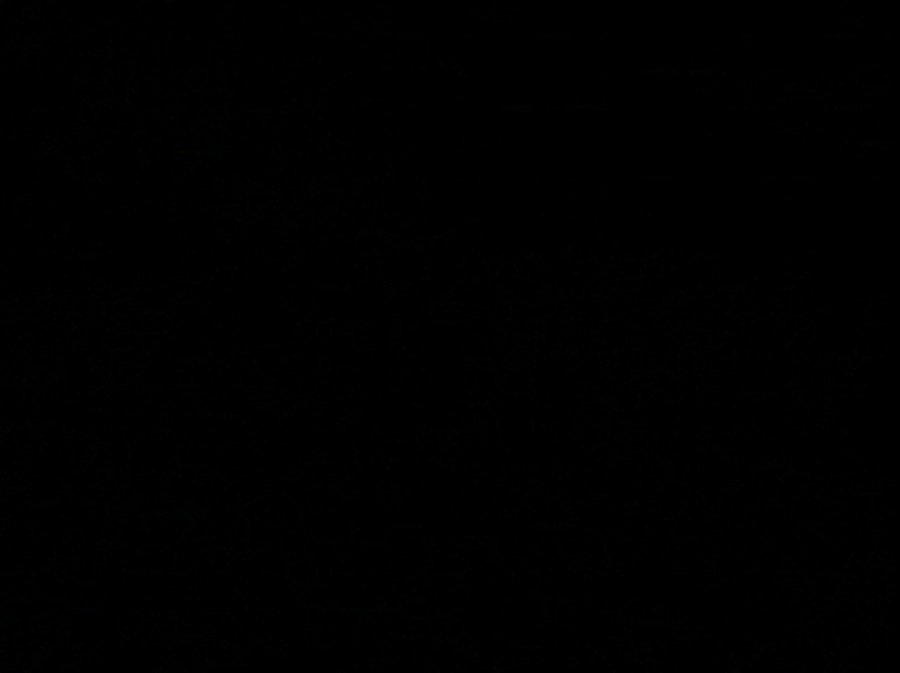} & \wSv{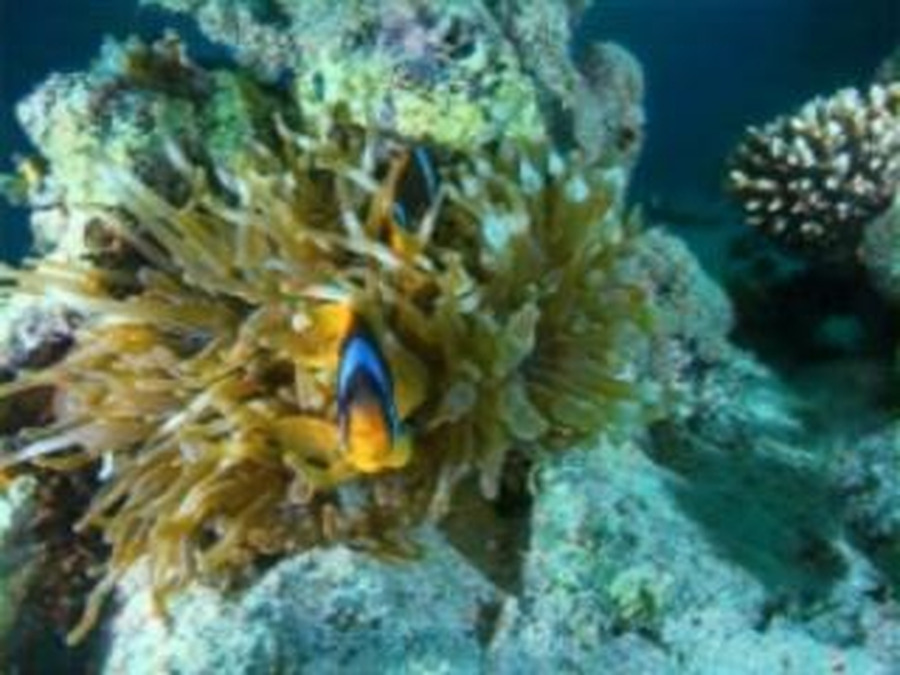}{59.8\,dB} & \wS{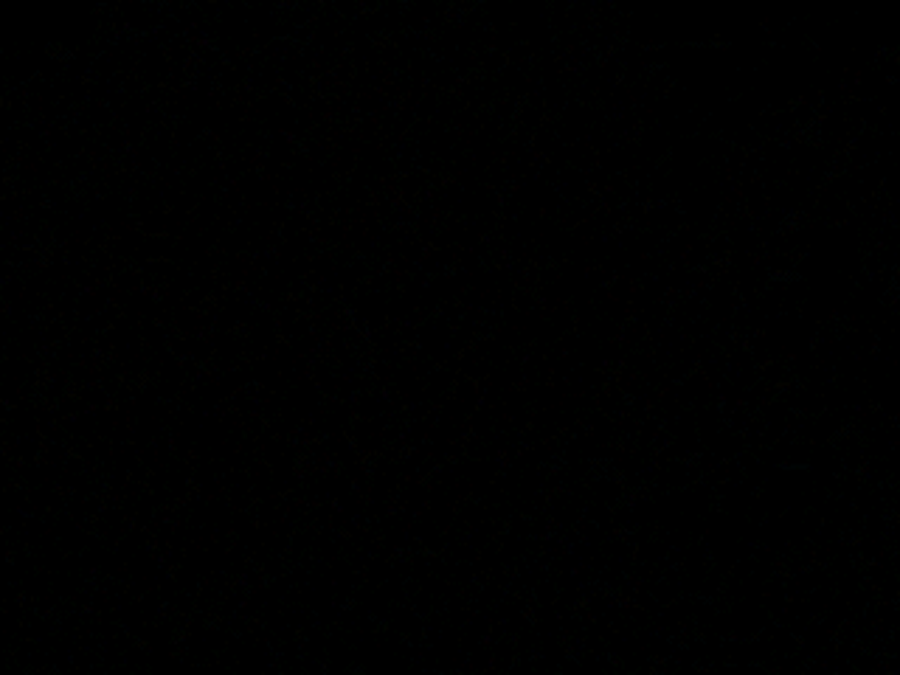} & \wSv{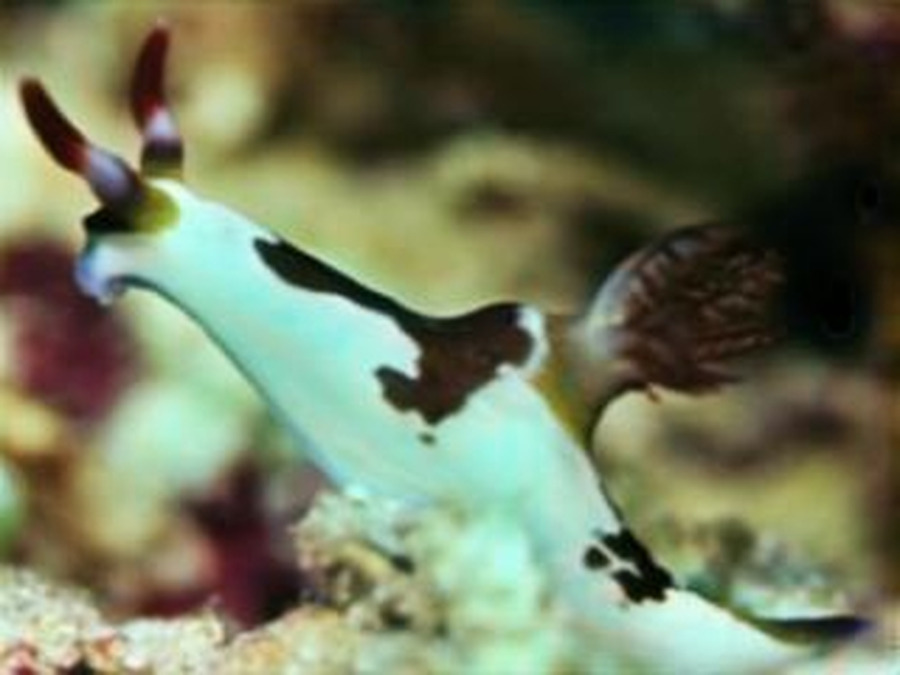}{60.6\,dB} & \wS{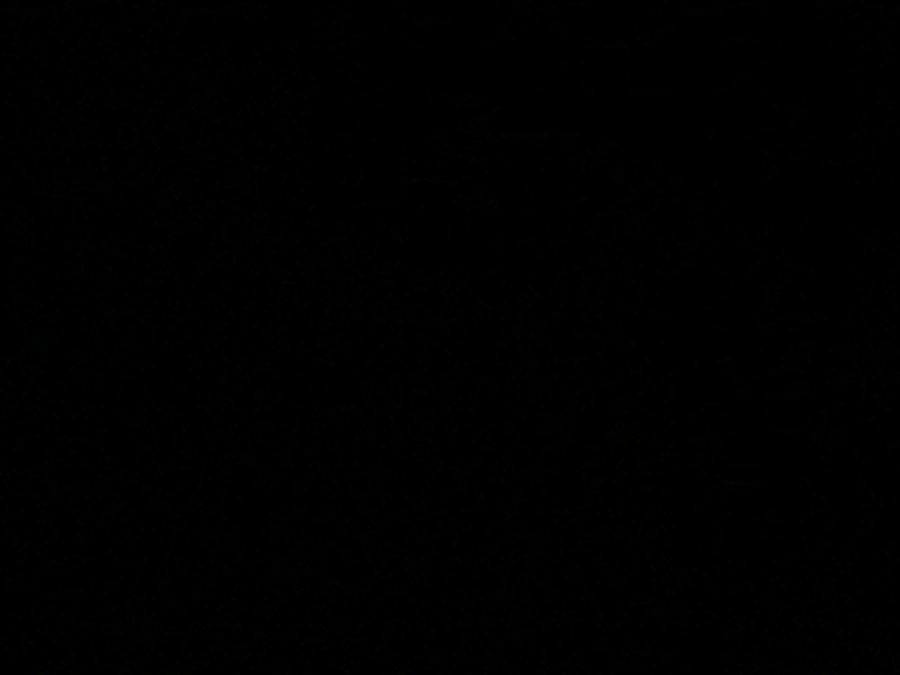} \\[\wgap]
     \wrw{VAE} & \wSv{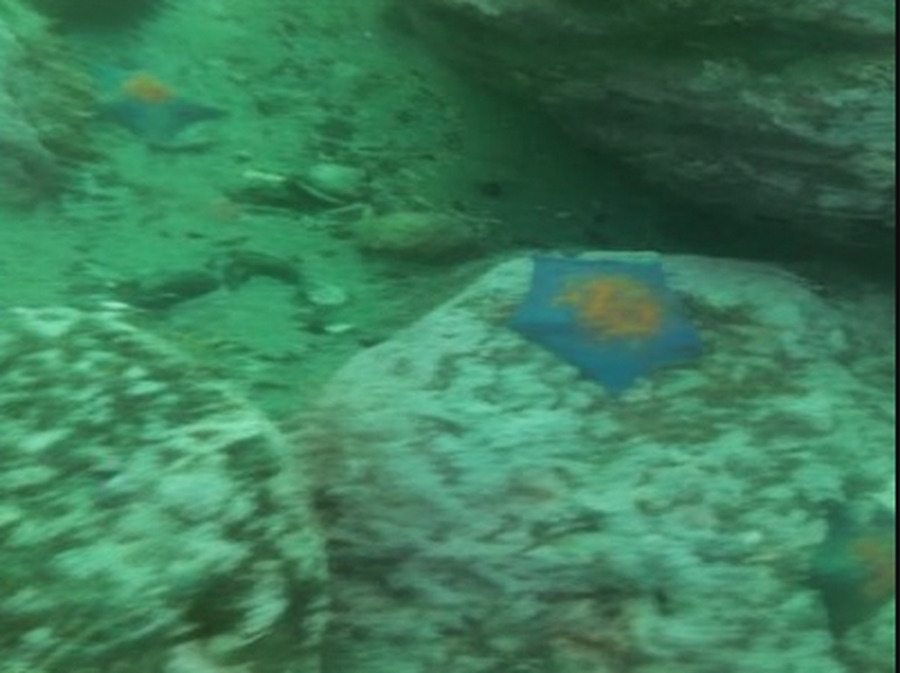}{31.4\,dB} & \wS{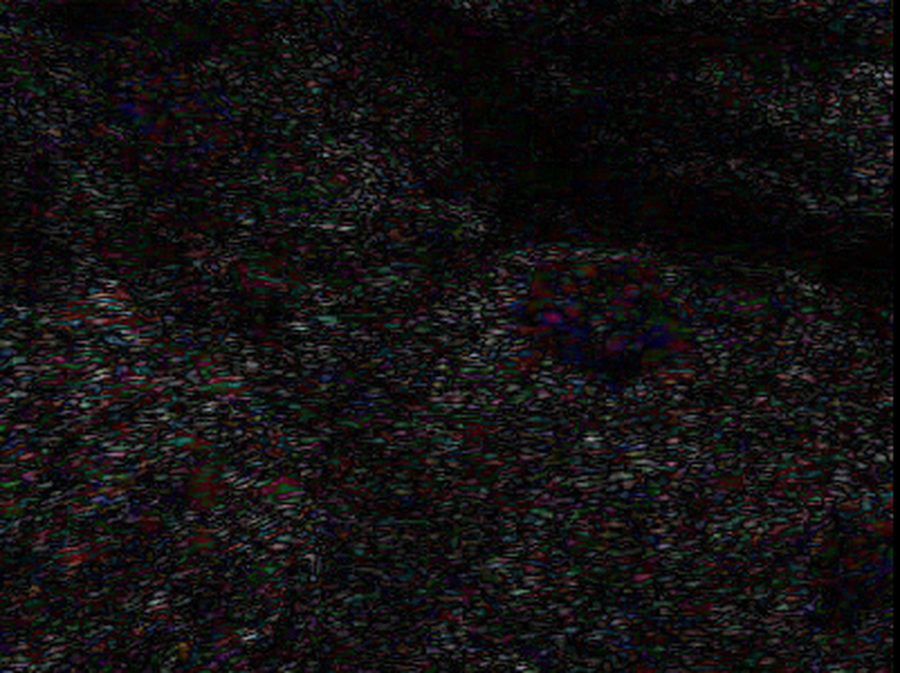} & \wSv{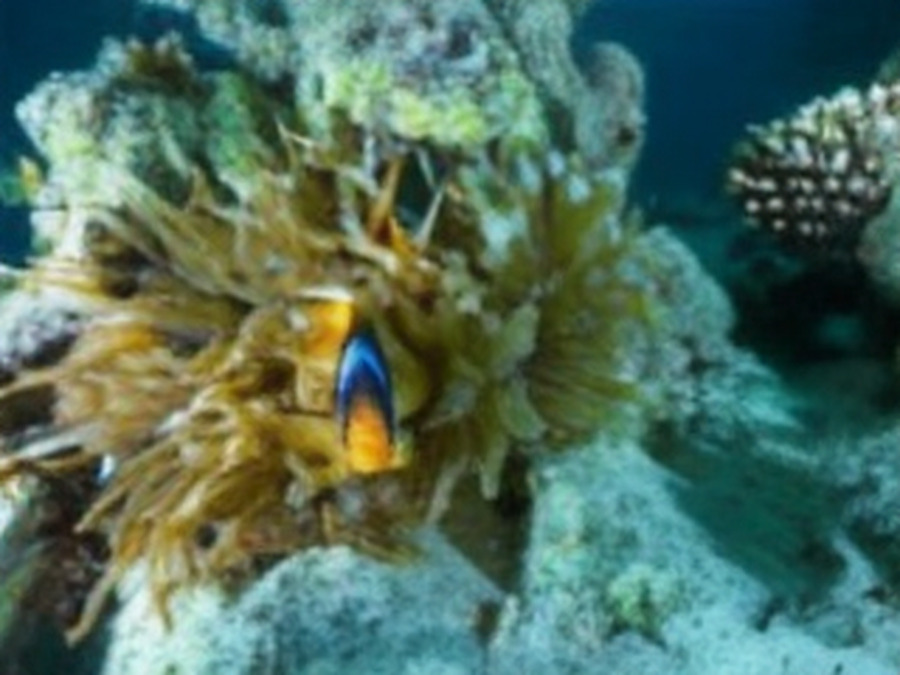}{24.7\,dB} & \wS{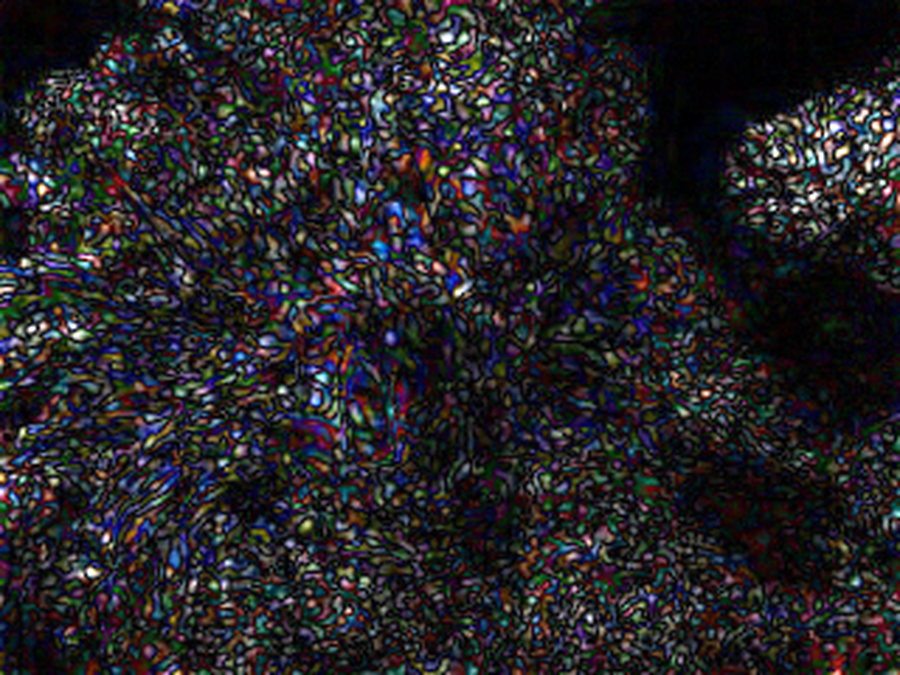} & \wSv{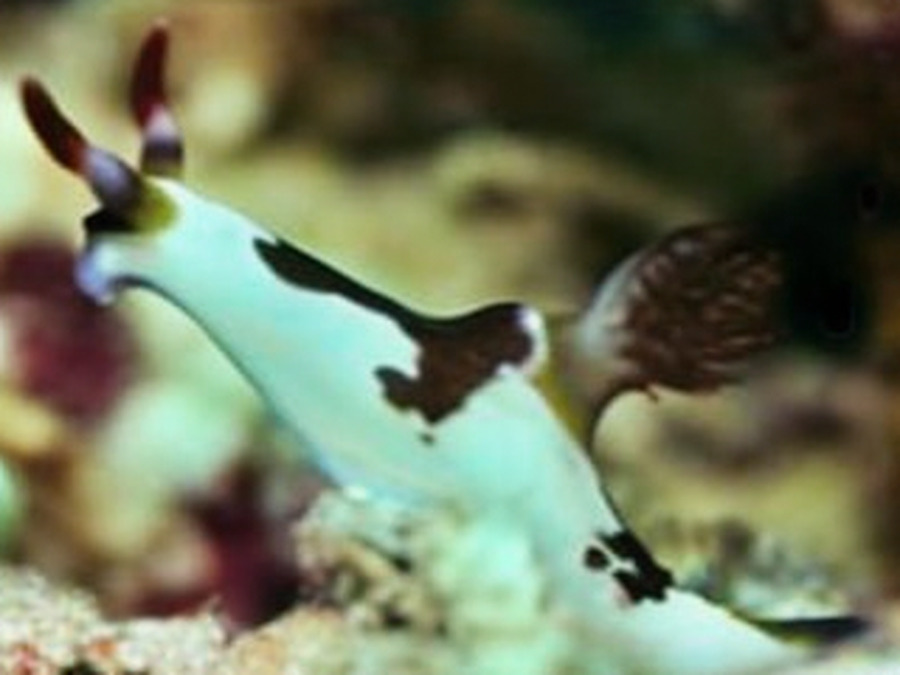}{33.3\,dB} & \wS{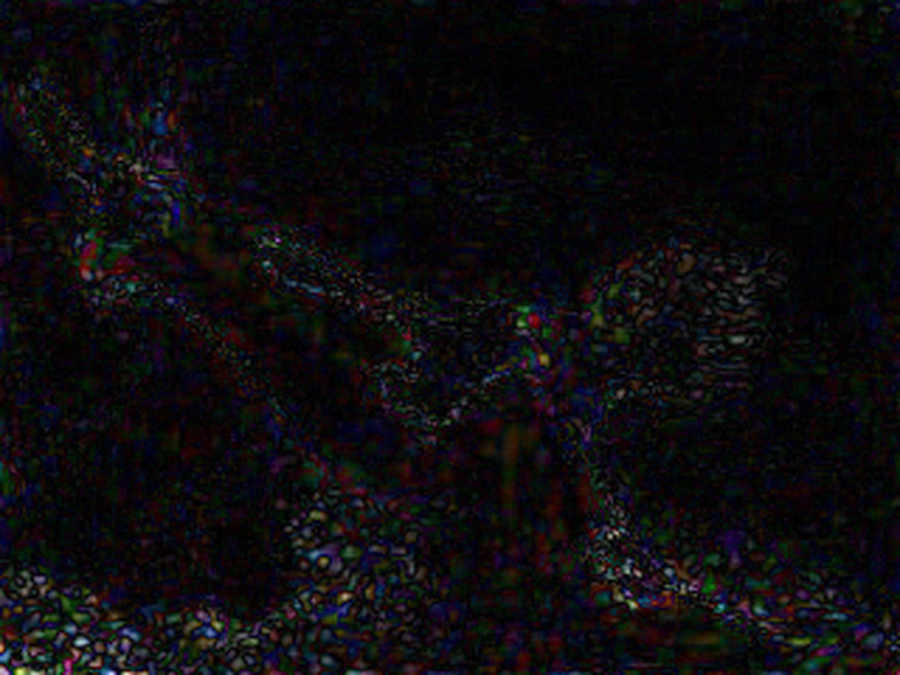} \\[\wgap]
     \wrw{AE-PoS} & \wSv{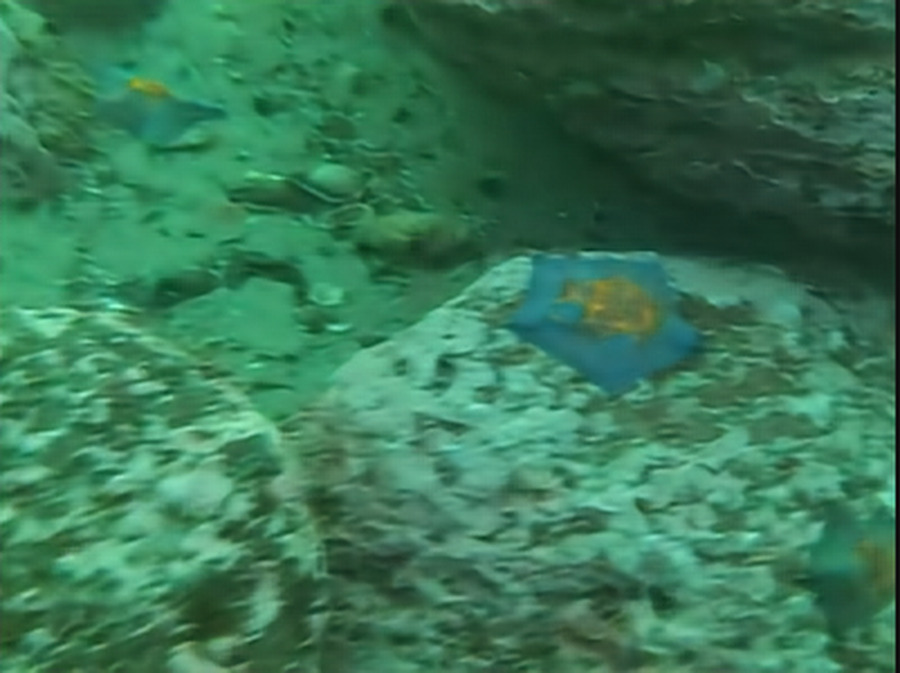}{29.7\,dB} & \wS{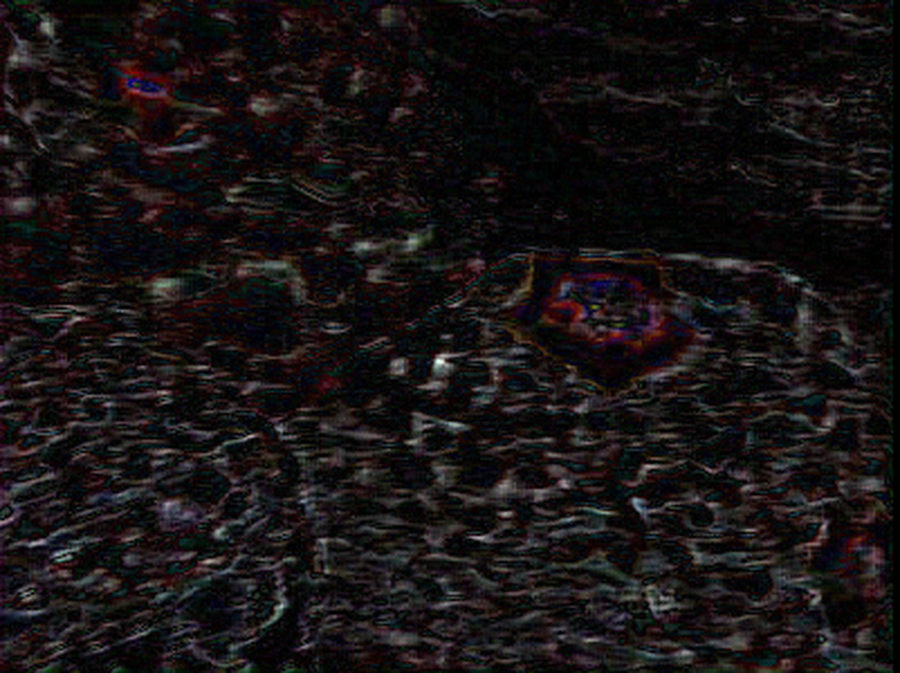} & \wSv{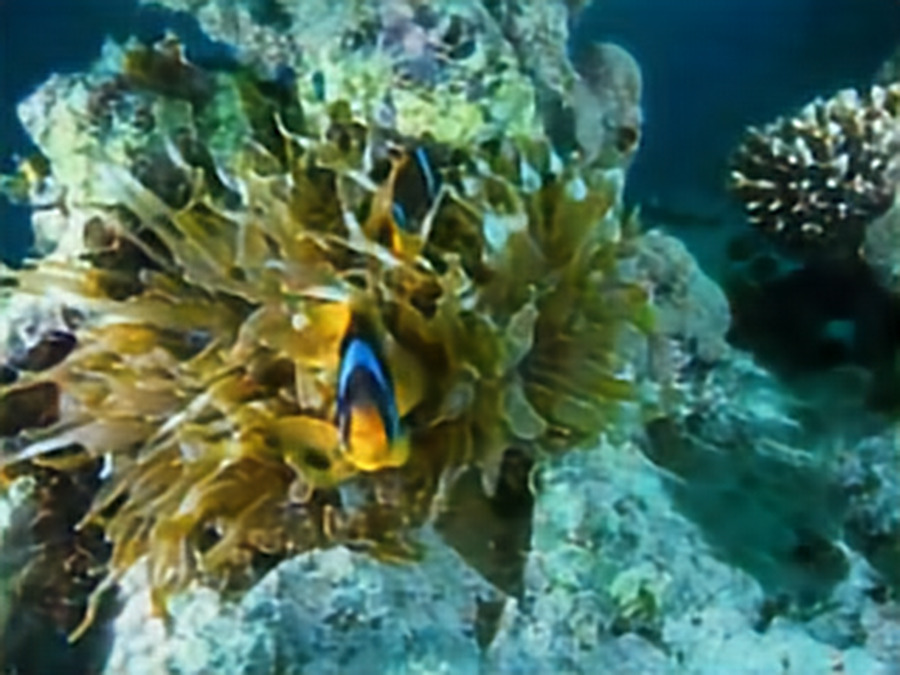}{26.8\,dB} & \wS{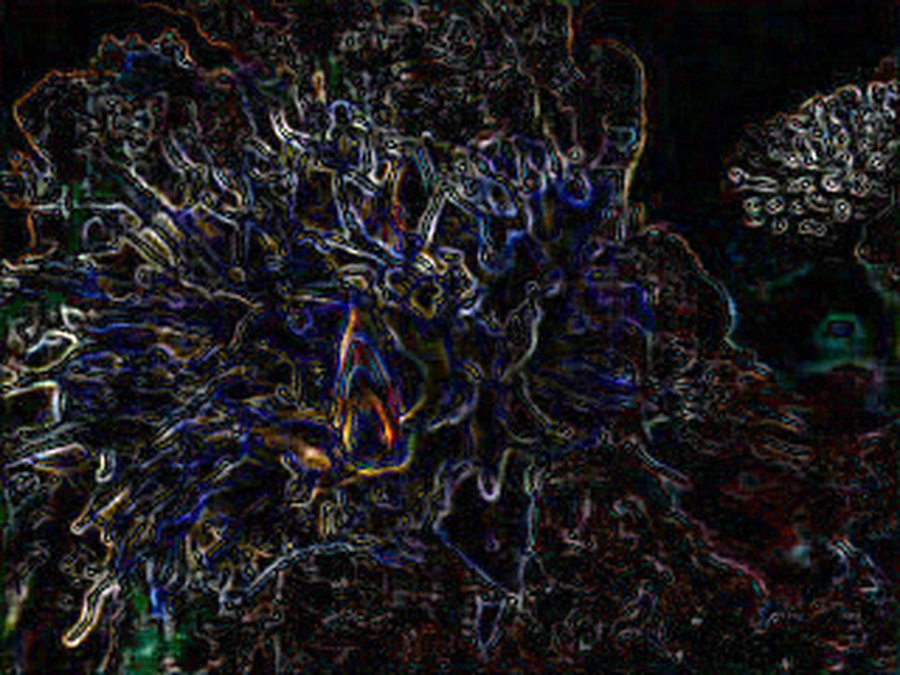} & \wSv{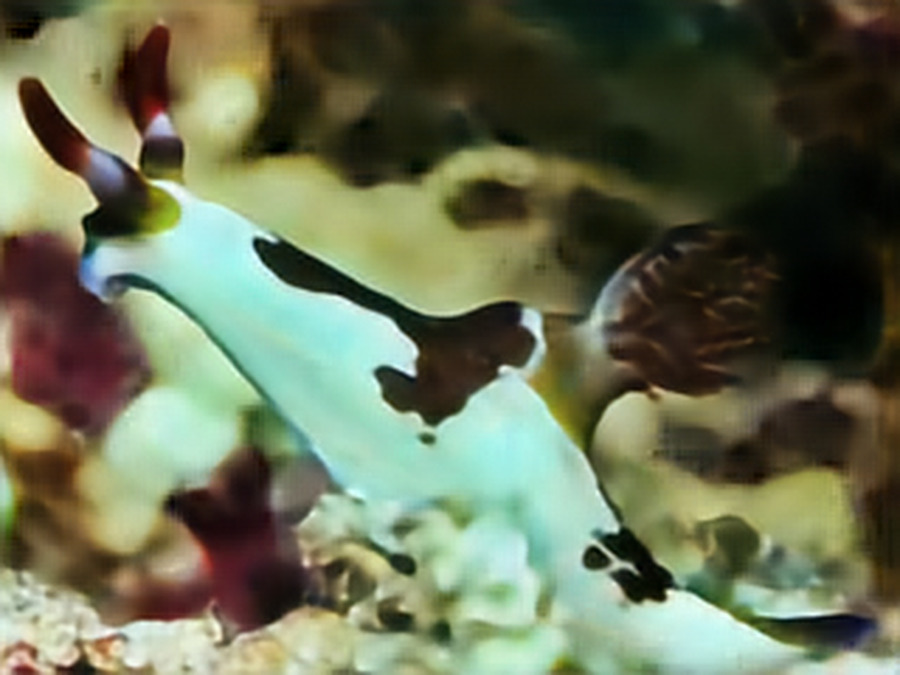}{24.1\,dB} & \wS{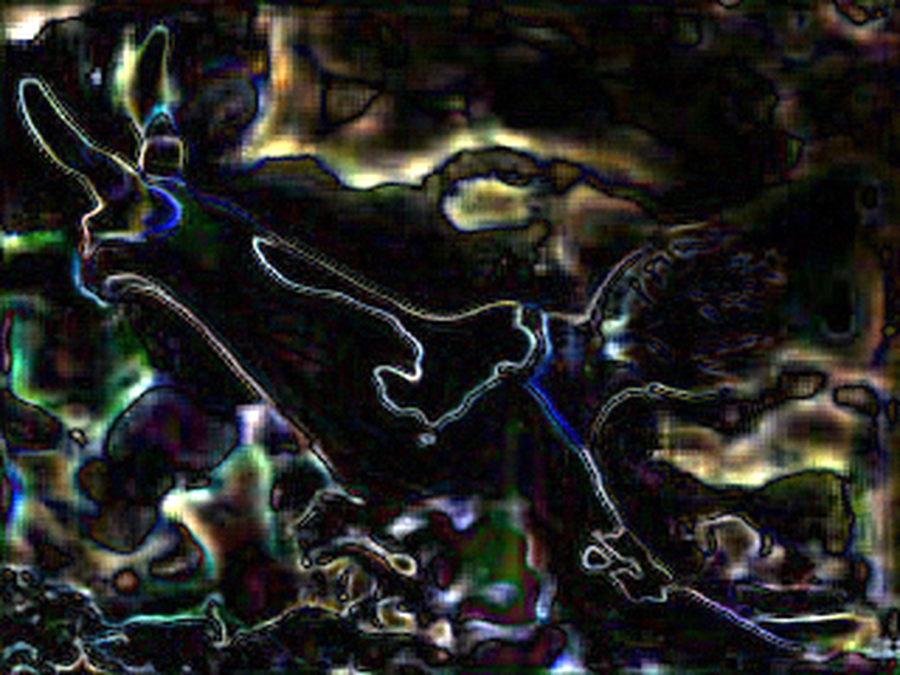} \\
  \end{tabular}
\end{center}
\caption{Underwater degradation at three severities (columns). Top row: the input $s$; below it each network's reconstruction $\hat{s}$ with its PSNR and the absolute difference $|s-\hat{s}|$. All difference panels share one intensity gain.}
\label{fig:underwater}
\end{figure}

%% file: figs/fig_blurreds.tex
\ifdefined\wsm\else
  \newlength{\wsm}\newlength{\wbg}\newlength{\wgap}\newlength{\wsh}\newlength{\wbh}
  \newlength{\wlab}\newlength{\wlabn}
  \newcommand{\wS}[1]{\includegraphics[width=\wsm]{#1}}
  \newcommand{\wB}[1]{\includegraphics[width=\wbg]{#1}}
  \newcommand{\wSv}[2]{\begin{tikzpicture}[inner sep=0]
      \node (im) {\wS{#1}};
      \node[anchor=south east, inner sep=1.2pt, fill=white, font=\footnotesize]
        at (im.south east) {#2};
    \end{tikzpicture}}
  \newcommand{\wrl}[1]{\makebox[\wlab][c]{\rotatebox{90}{\makebox[\wsh][c]{\large #1}}}}
  \newcommand{\wrb}[1]{\makebox[\wlab][c]{\rotatebox{90}{\makebox[\wbh][c]{\normalsize #1}}}}
  \newcommand{\wrw}[1]{\makebox[\wlabn][c]{\rotatebox{90}{\makebox[\wsh][c]{\normalsize #1}}}}
  \newcommand{\wrn}[1]{\makebox[\wlabn][c]{\smash{\rotatebox{90}{%
      \makebox[\dimexpr 2\wsh + \wgap\relax][c]{\normalsize #1}}}}}
  \newcommand{\wrB}[1]{\makebox[\wlab][c]{\rotatebox{90}{\makebox[\wbh][c]{\large #1}}}}
\fi
\setlength{\wgap}{2pt}
\setlength{\wlab}{1.3em}
\setlength{\wlabn}{1.4em}
\setlength{\wsm}{\dimexpr (\textwidth - \wlab - \wlabn - 5.0000\wgap) / 5\relax}
\setlength{\wsh}{0.5622\wsm}
\begin{figure}[t]
\begin{center}
  \setlength{\tabcolsep}{0.5\wgap}\renewcommand{\arraystretch}{0}
 \begin{tabular}{@{}cc ccccc@{}}
     & & {\normalsize $0.2$} & {\normalsize $0.4$} & {\normalsize $0.6$} & {\normalsize $0.8$} & {\normalsize $1.0$} \\[2pt]
     & \wrl{$s$} & \wS{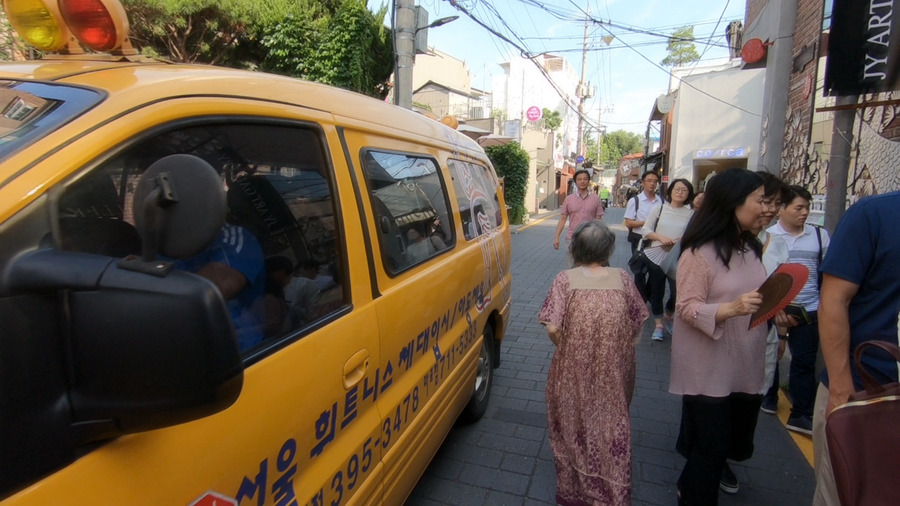} & \wS{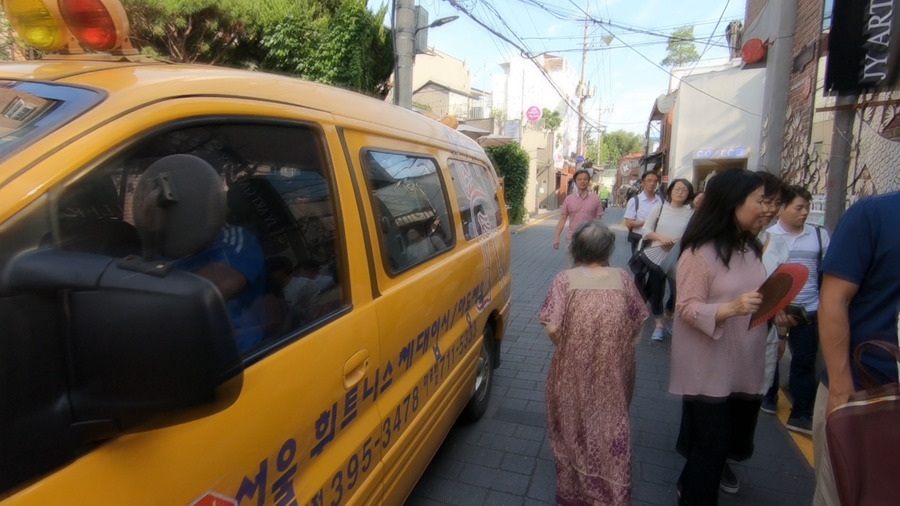} & \wS{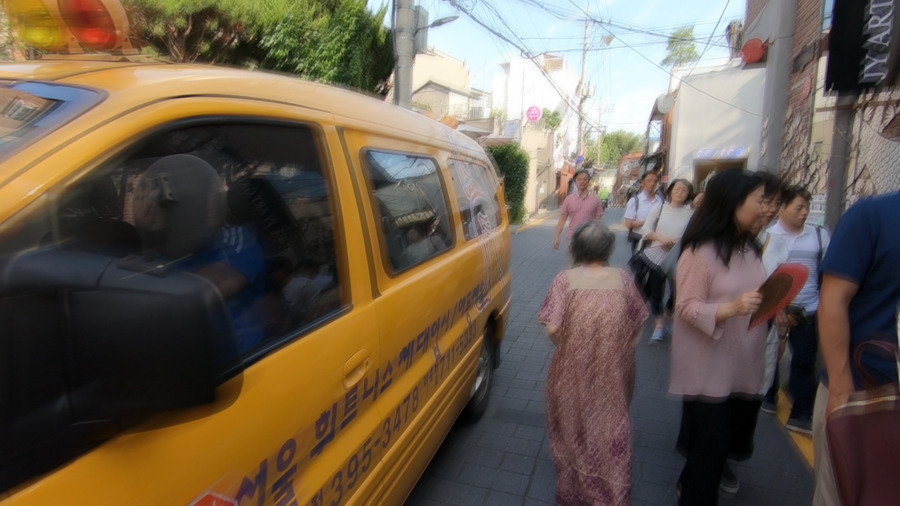} & \wS{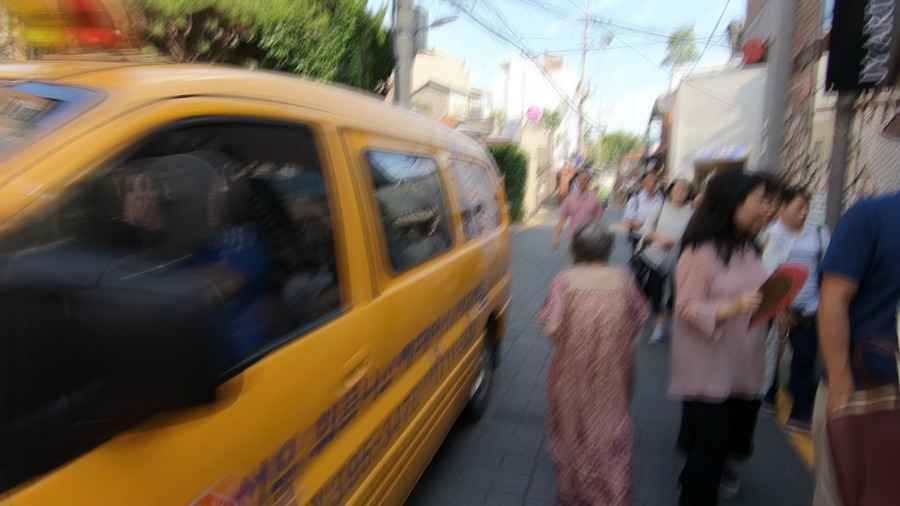} & \wS{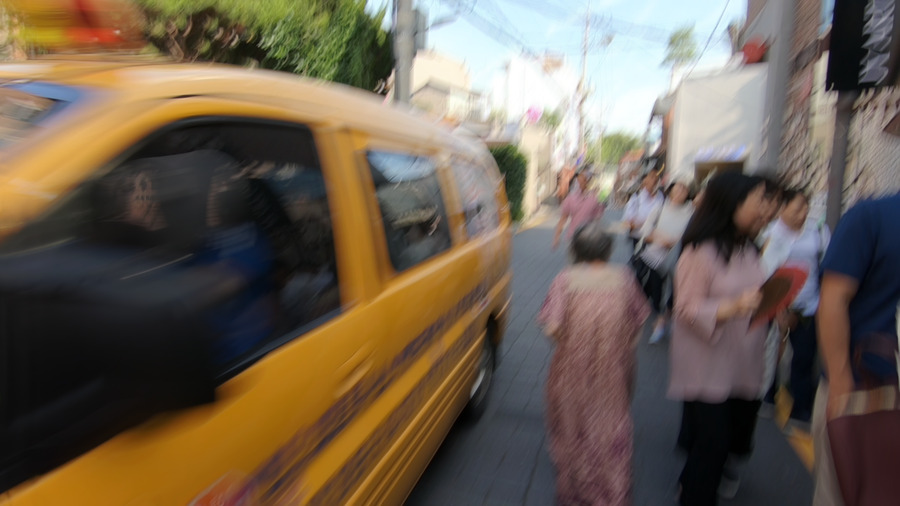} \\[\wgap]
     & \wrl{$\hat{s}$} & \wSv{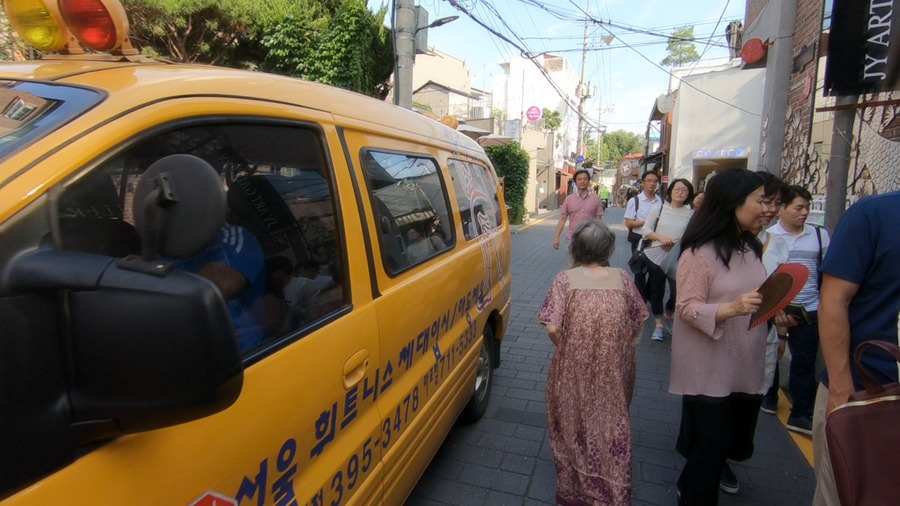}{61.8\,dB} & \wSv{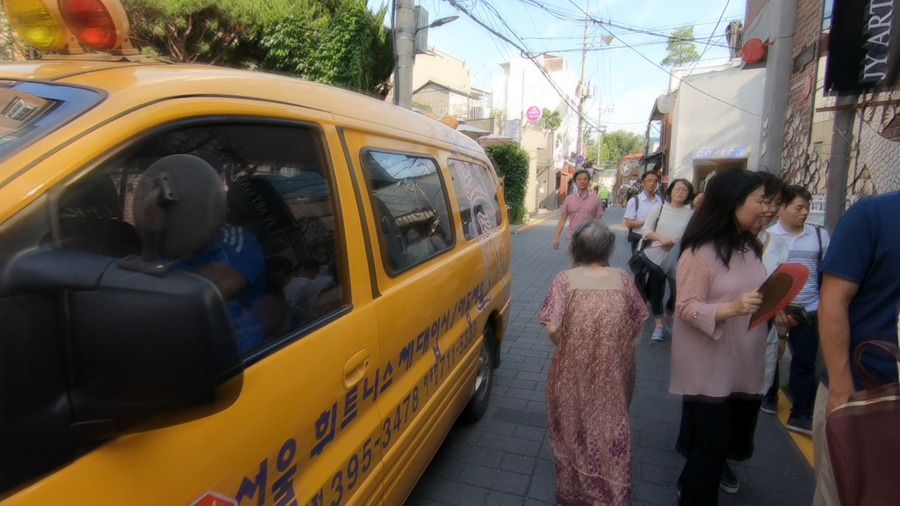}{62.3\,dB} & \wSv{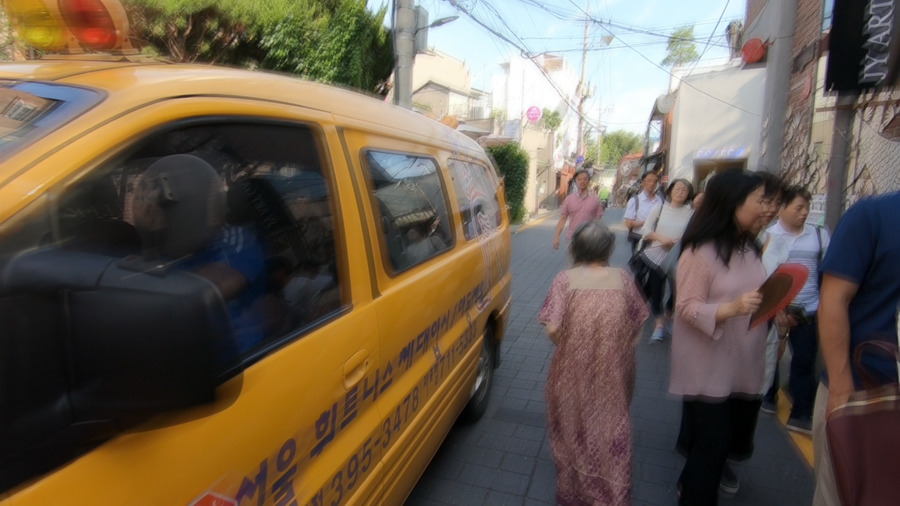}{62.4\,dB} & \wSv{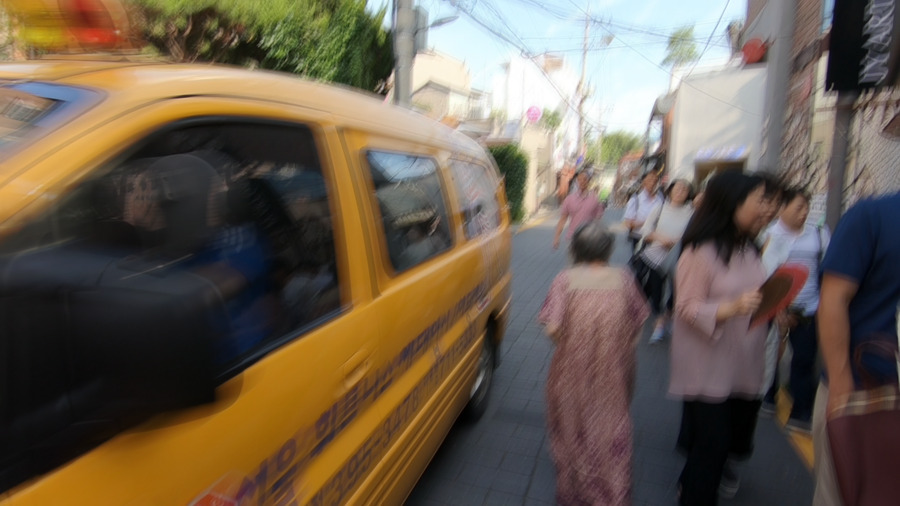}{62.0\,dB} & \wSv{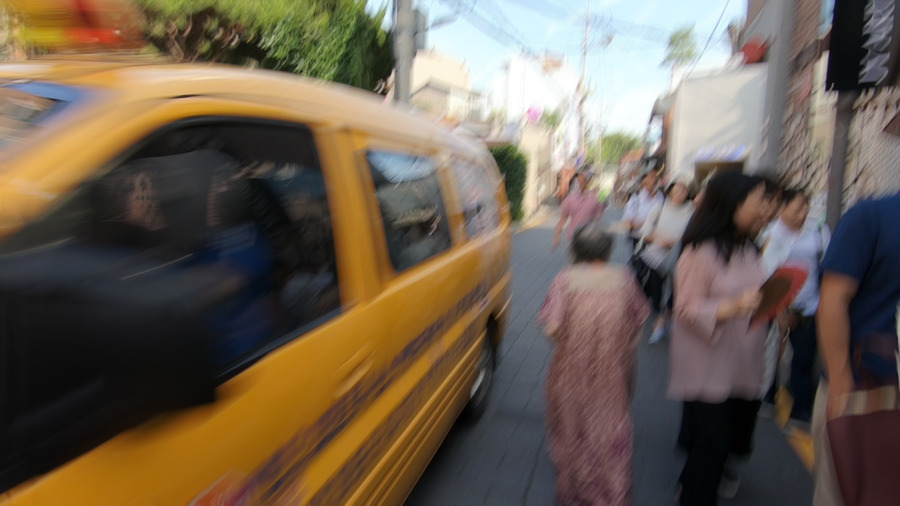}{62.8\,dB} \\[\wgap]
     \wrn{AE} & \wrl{$s-\hat{s}$} & \wS{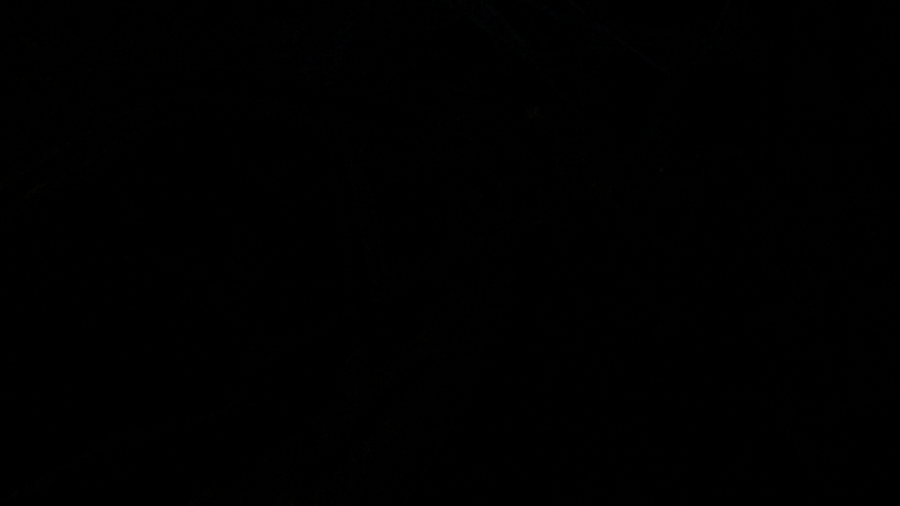} & \wS{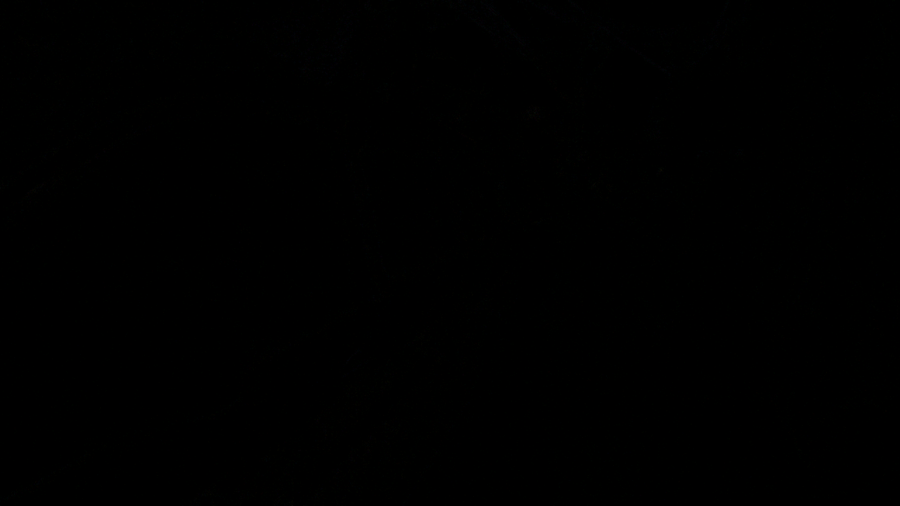} & \wS{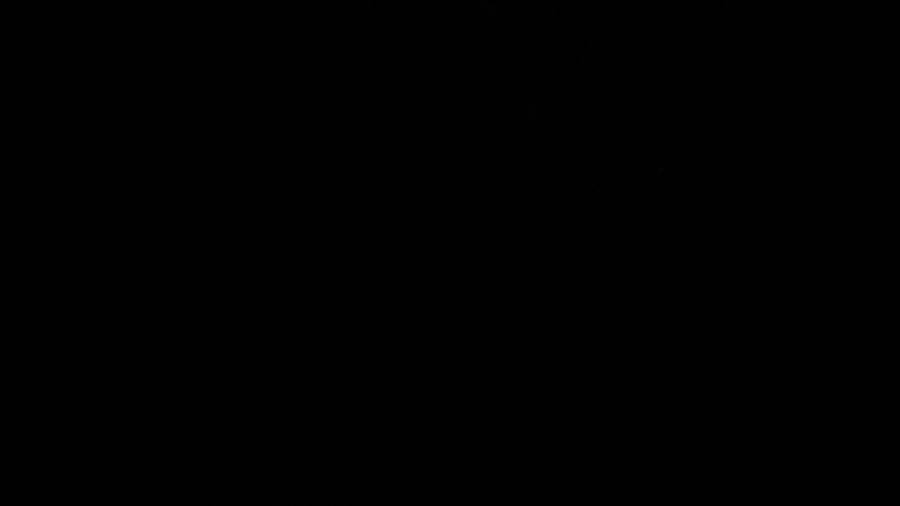} & \wS{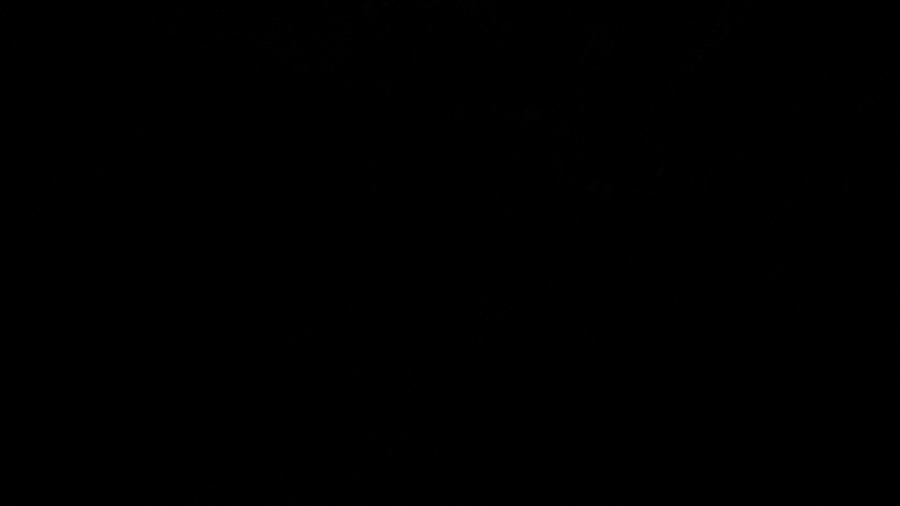} & \wS{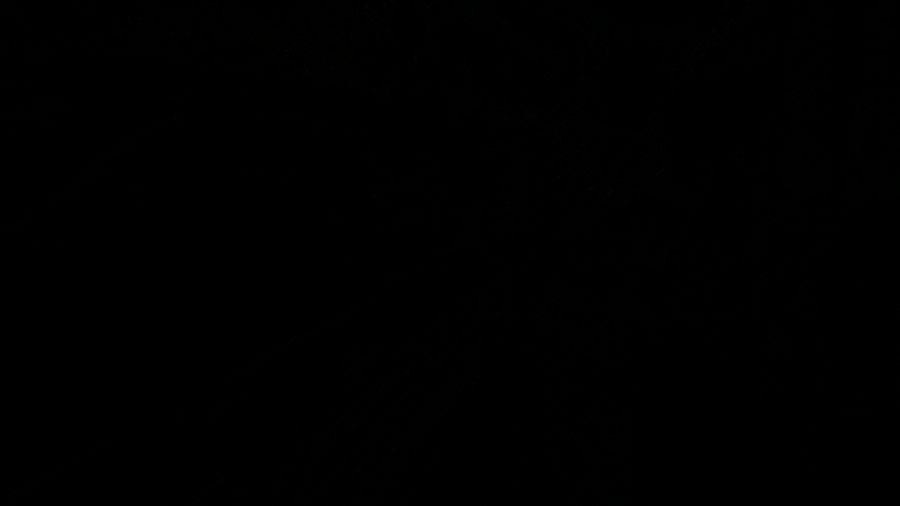} \\[\wgap]
     & \wrl{$\hat{s}$} & \wSv{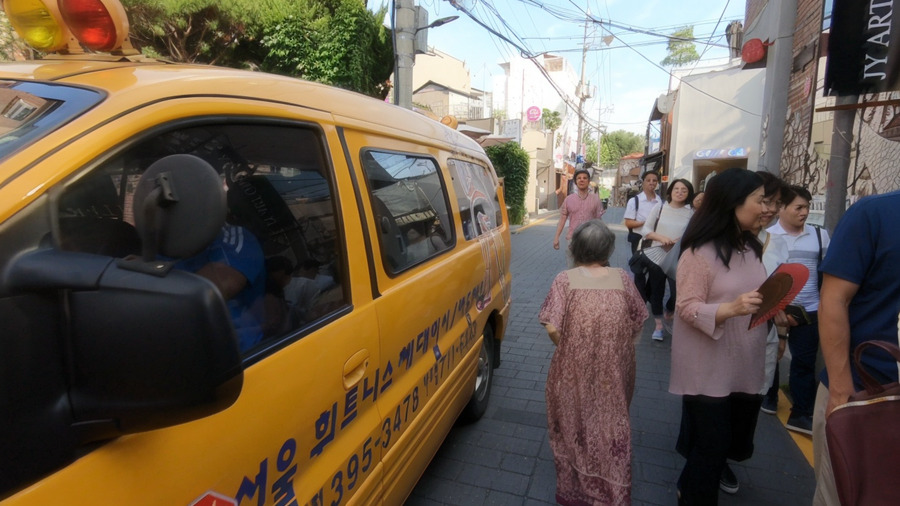}{29.2\,dB} & \wSv{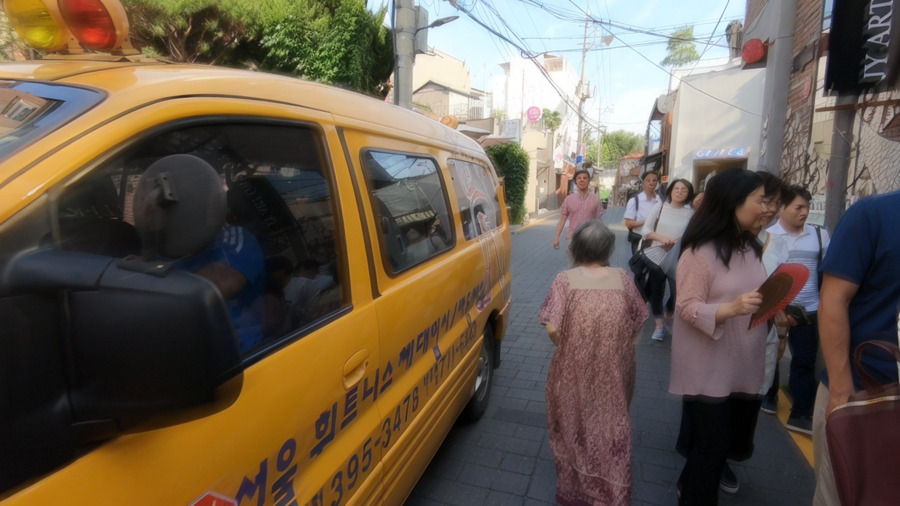}{31.1\,dB} & \wSv{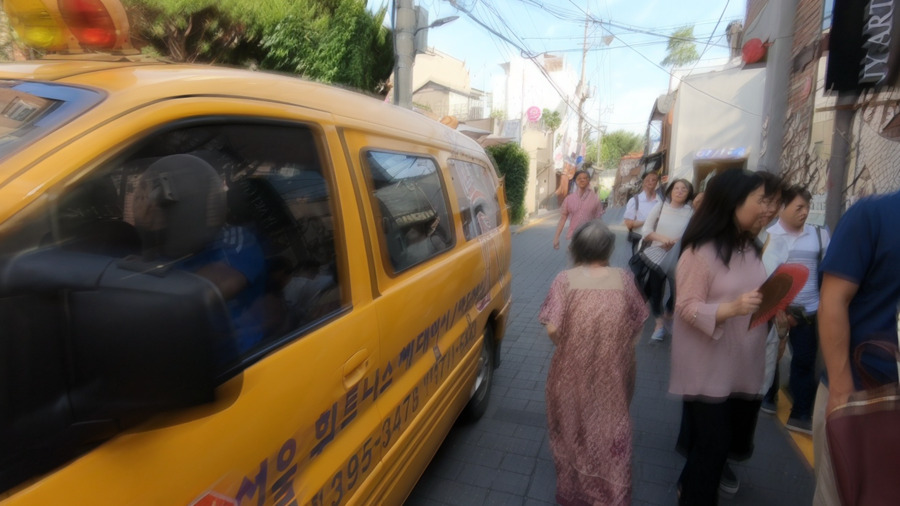}{33.2\,dB} & \wSv{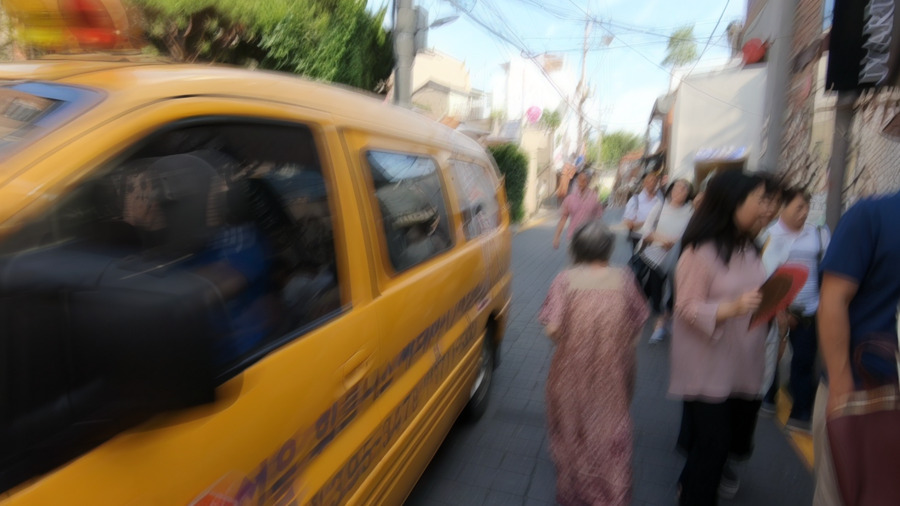}{35.7\,dB} & \wSv{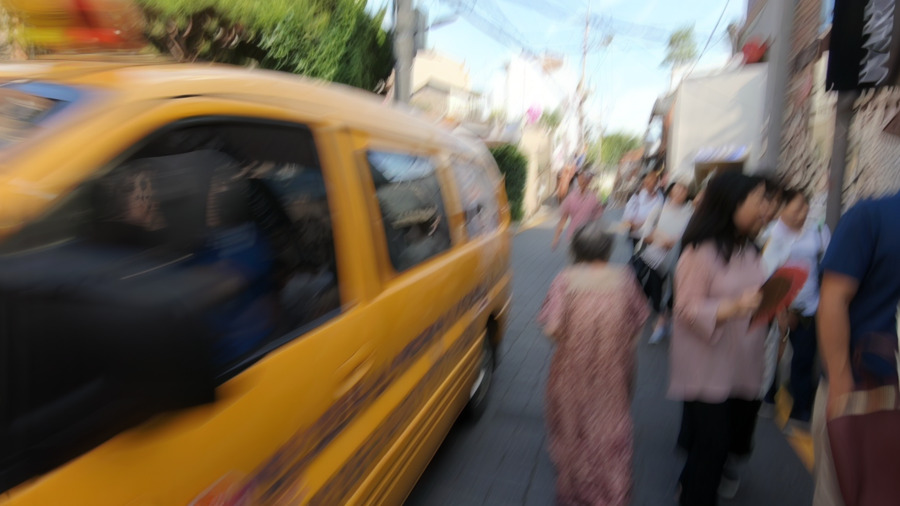}{37.3\,dB} \\[\wgap]
     \wrn{VAE} & \wrl{$s-\hat{s}$} & \wS{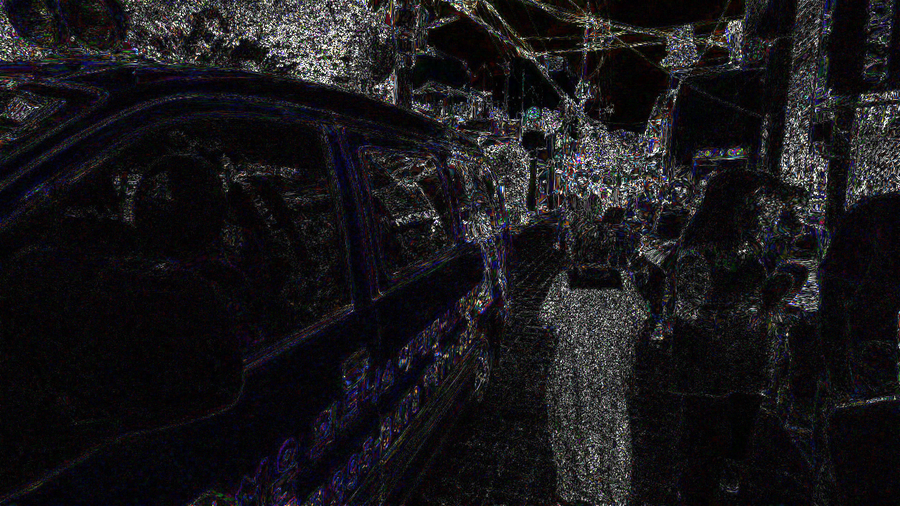} & \wS{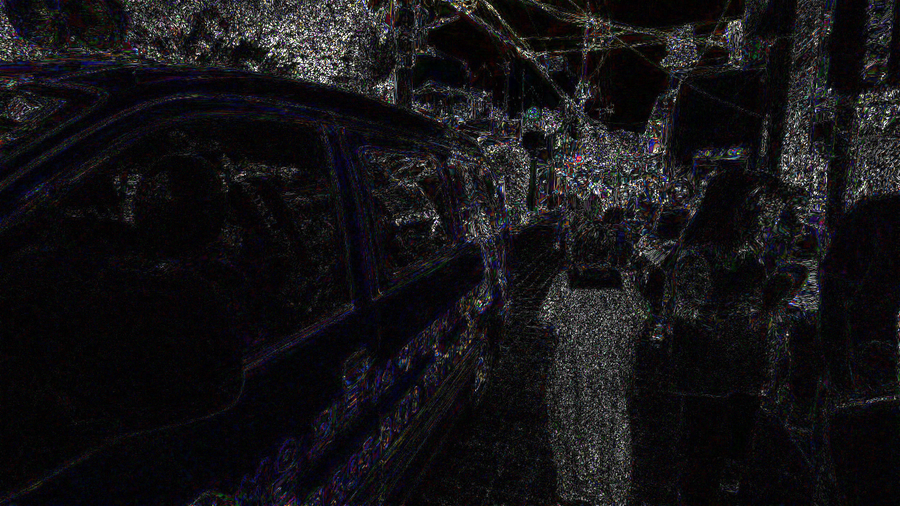} & \wS{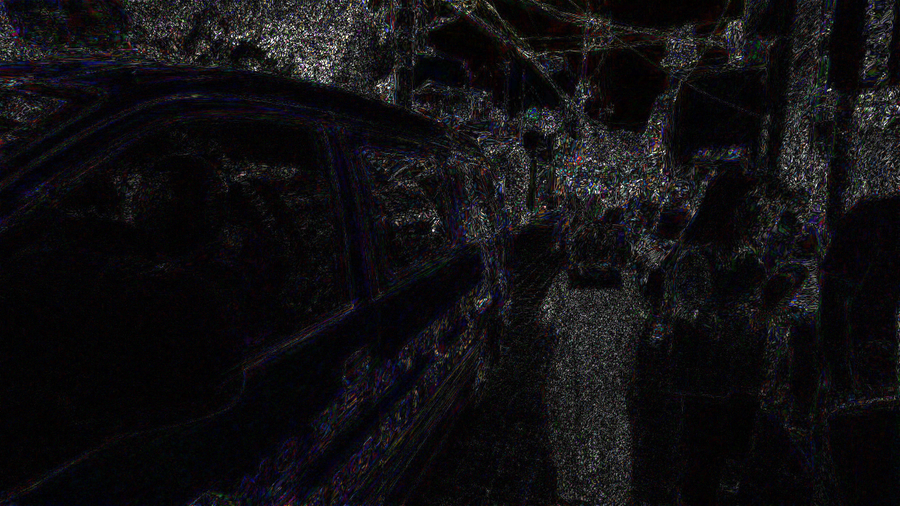} & \wS{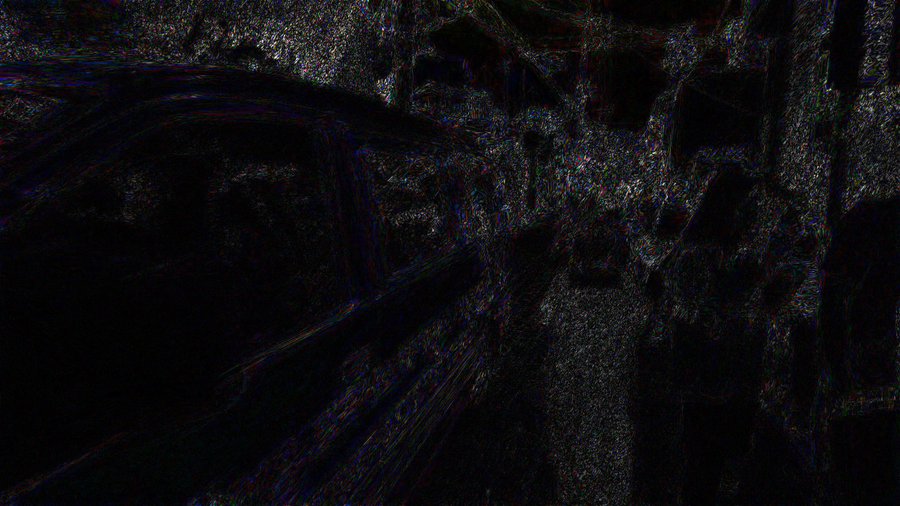} & \wS{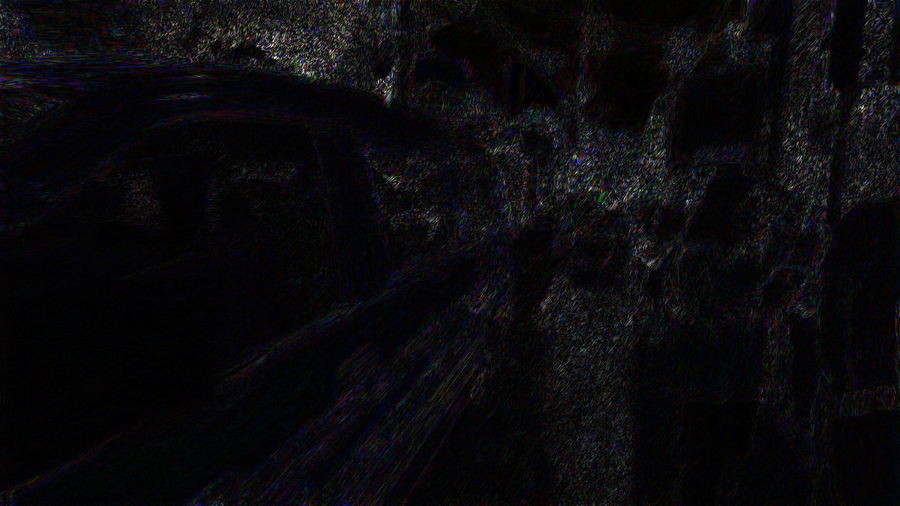} \\[\wgap]
     & \wrl{$\hat{s}$} & \wSv{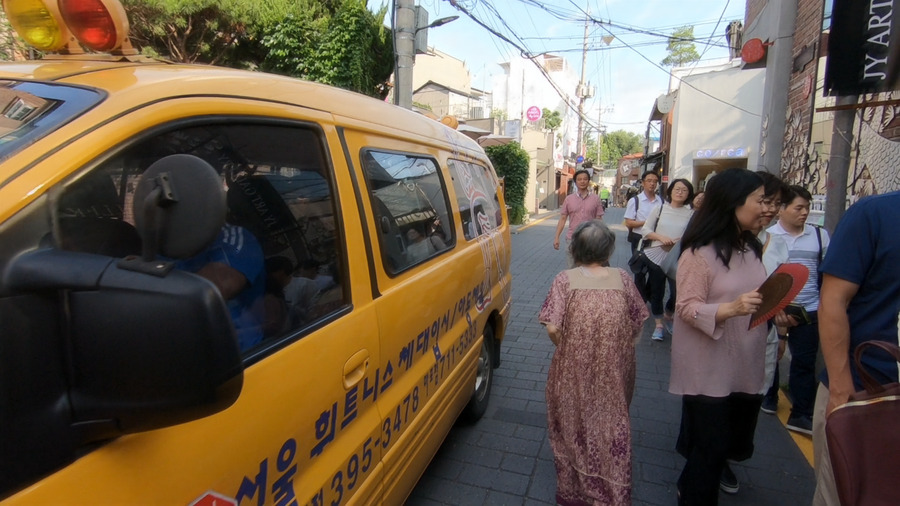}{43.5\,dB} & \wSv{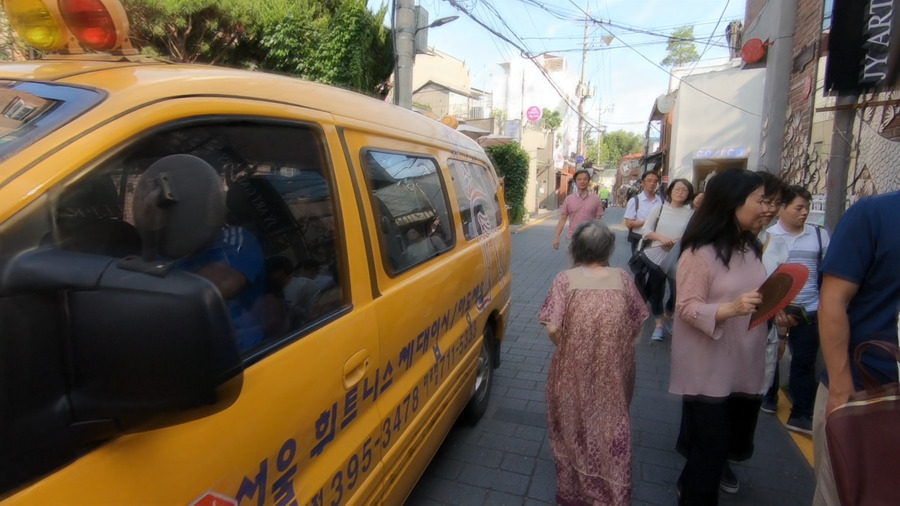}{41.7\,dB} & \wSv{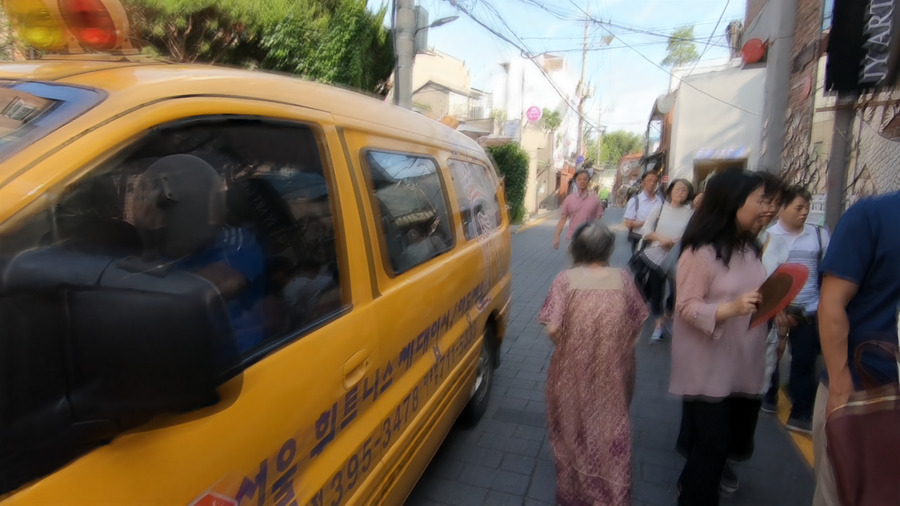}{37.8\,dB} & \wSv{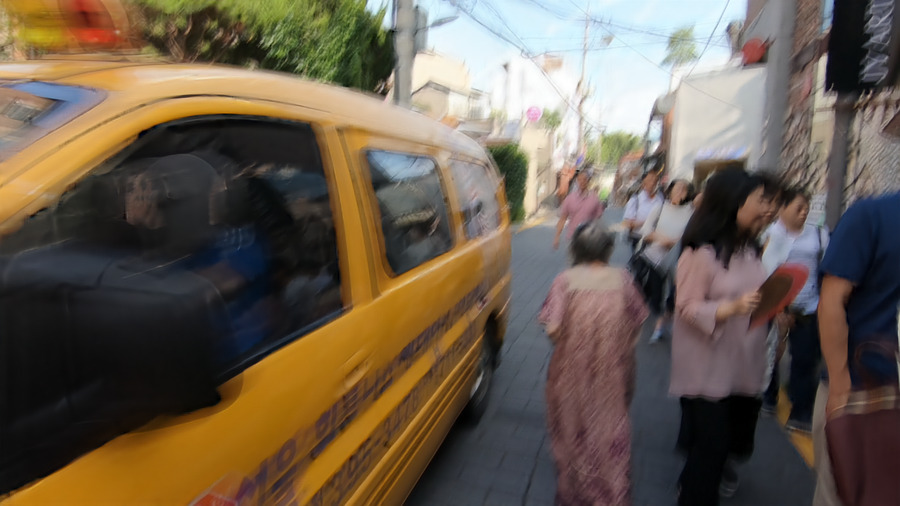}{33.1\,dB} & \wSv{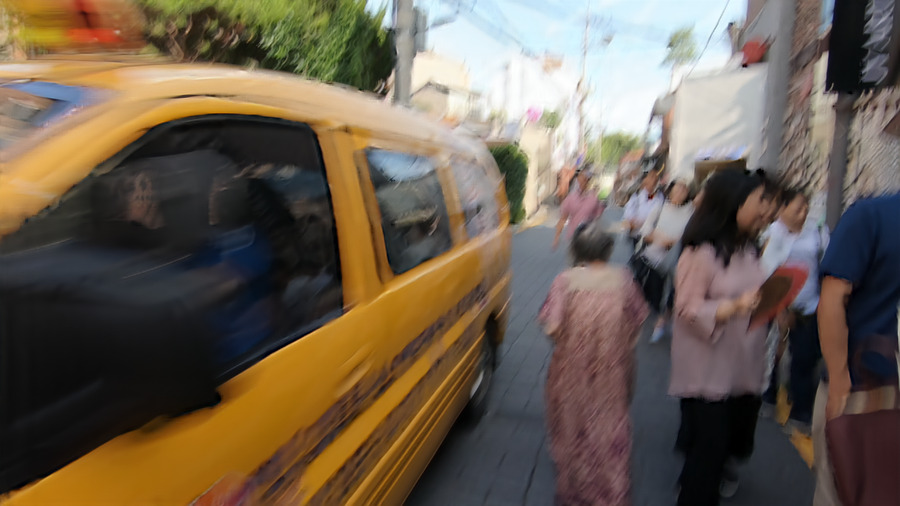}{31.1\,dB} \\[\wgap]
     \wrn{AE-PoS} & \wrl{$s-\hat{s}$} & \wS{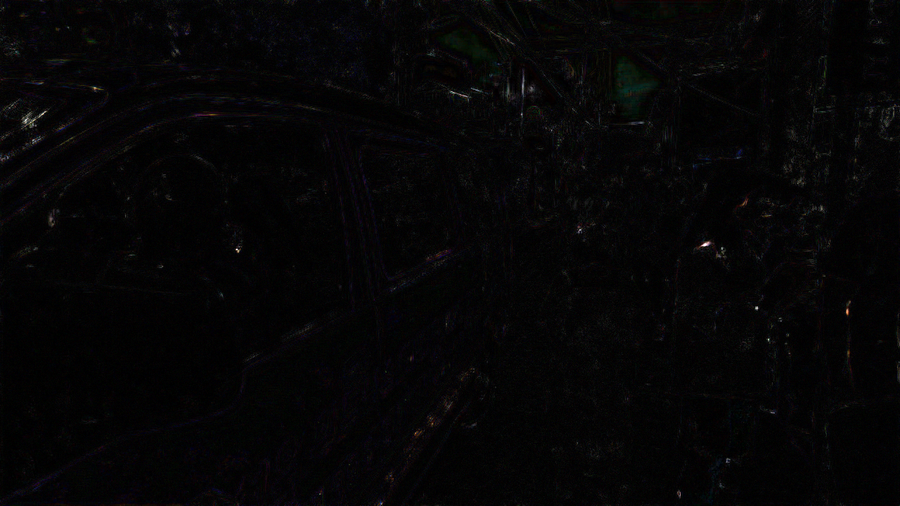} & \wS{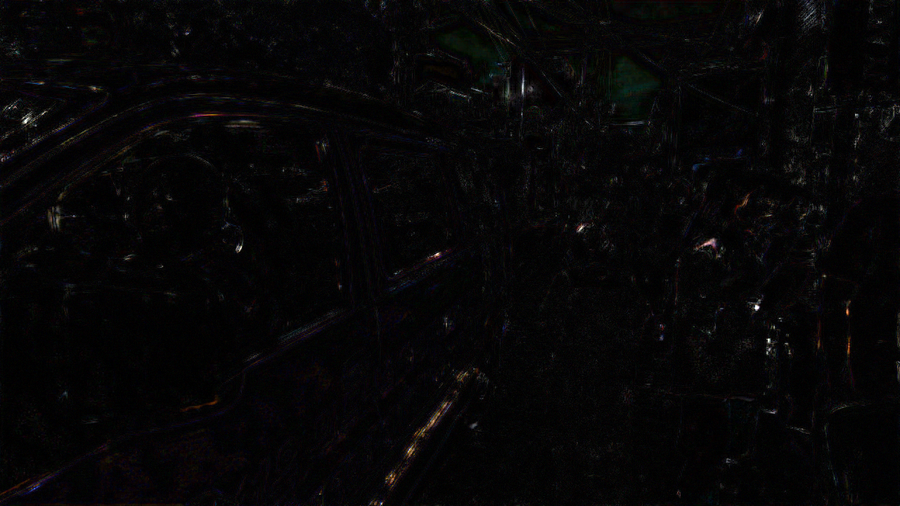} & \wS{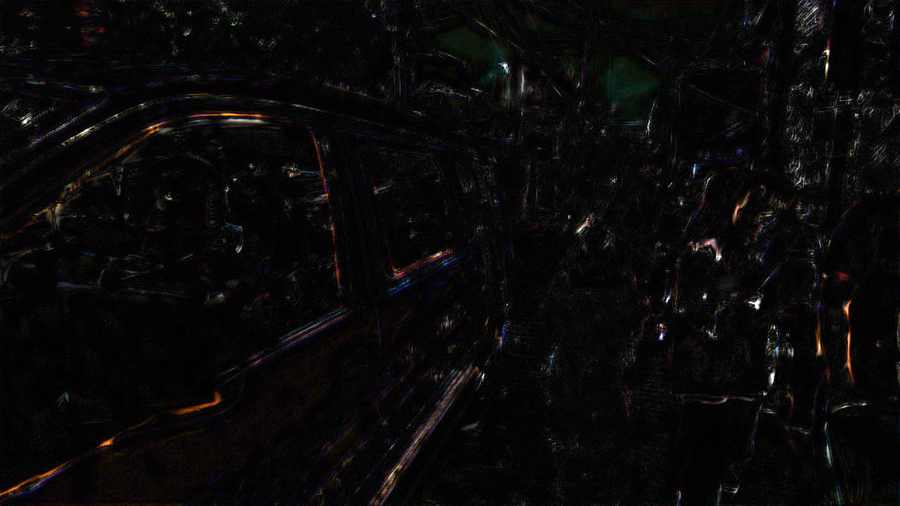} & \wS{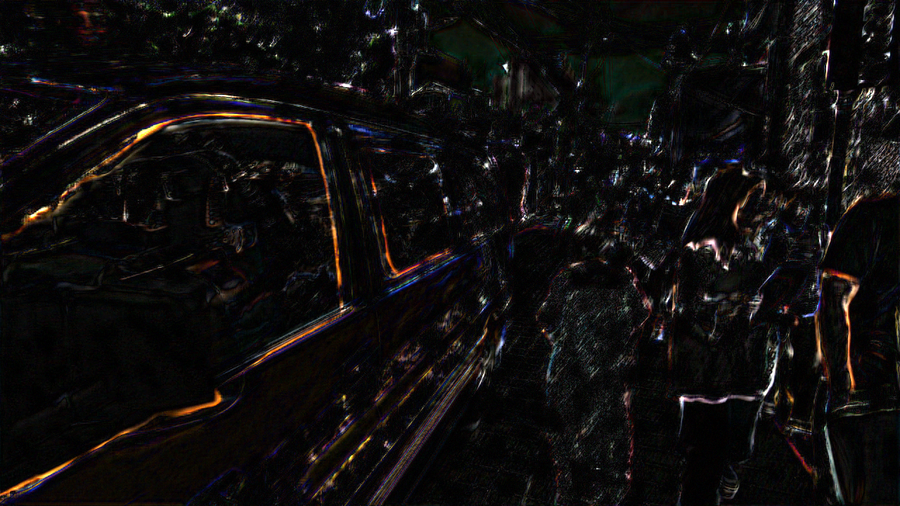} & \wS{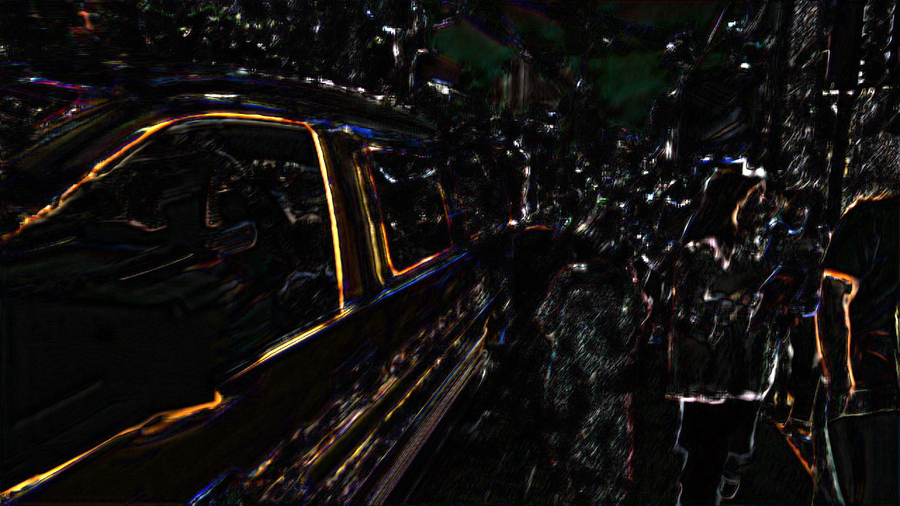} \\
  \end{tabular}
\end{center}
\caption{REDS under increasing blur. Columns show the same scene at five blur
strengths; rows show the input \(s\), each model's reconstruction
\(\hat{s}\) with its PSNR, and the absolute residual
\(\lvert s-\hat{s}\rvert\). The VAE increasingly represents the blurred
component: its reconstruction PSNR rises from \(29.2\) to \(37.3\) dB as the
blur strengthens, causing the residual to decrease. In contrast, AE-PoS
excludes the blur from its learned target set, producing increasingly visible
residuals and a PSNR decrease from \(43.5\) to \(31.1\) dB. All residual
panels use the same intensity gain. On the corresponding clean image, the
reconstruction PSNRs are \(62.4\), \(27.6\), and \(43.9\) dB for AE, VAE,
and AE-PoS, respectively.}
\label{fig:reds-blur-ramp}
\end{figure}

%% file: figs/fig_mix_images.tex
\ifdefined\vsm\else
  \newlength{\vsm}\newlength{\vbg}\newlength{\vgap}\newlength{\vsh}\newlength{\vbh}
  \newlength{\vlabt}\newlength{\vlabs}\newlength{\vpad}
  \newcommand{\vS}[1]{\includegraphics[width=\vsm]{#1}}
  \newcommand{\vB}[1]{\includegraphics[width=\vbg]{#1}}
  \newcommand{\vSv}[2]{\begin{tikzpicture}[inner sep=0]
      \node (im) {\vS{#1}};
      \node[anchor=south east, inner sep=1.2pt, fill=white, font=\footnotesize]
        at (im.south east) {#2};
    \end{tikzpicture}}
  \newcommand{\vrl}[1]{\makebox[\vlabs][c]{\rotatebox{90}{\makebox[\vsh][c]{\large #1}}}}
  \newcommand{\vrb}[1]{\makebox[\vlabt][c]{\rotatebox{90}{\makebox[\vbh][c]{\normalsize #1}}}}
\fi
\setlength{\vgap}{2.5pt}
\setlength{\vlabt}{7.2pt}\setlength{\vlabs}{9.6pt}\setlength{\vpad}{1pt}
\begin{figure}[t]
\begin{center}
  \setlength{\vsm}{\dimexpr (\textwidth - \vlabt - \vlabs - 1pt - 5.0000\vgap) / 5\relax}
  \setlength{\vbg}{\dimexpr 2\vsm + 1.0000\vgap\relax}
  \setlength{\vsh}{1.0000\vsm}\setlength{\vbh}{1.0000\vbg}
  \begin{minipage}[t]{\dimexpr \vlabt + \vgap + \vbg\relax}\vspace{0pt}\raggedright
    \makebox[\vlabt]{}\hspace{\vgap}\makebox[\vbg][c]{\large $s$}\\[2pt]
    \vrb{nominal}\hspace{\vgap}\vB{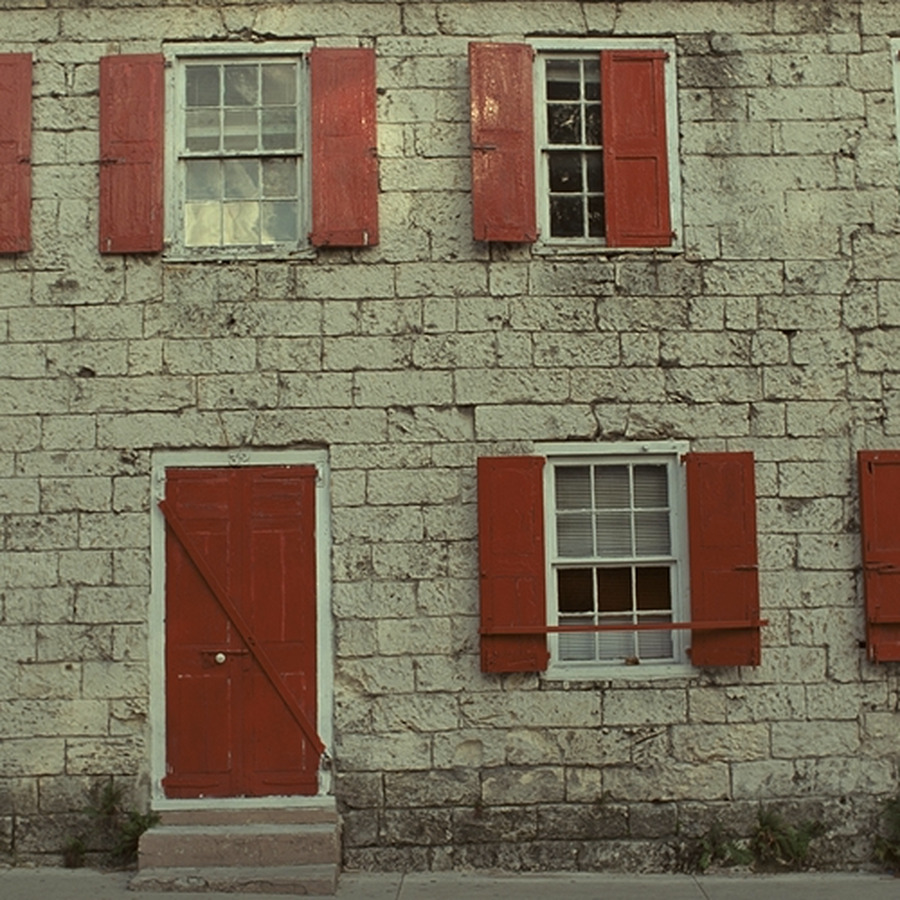}
  \end{minipage}\hfill
  \begin{minipage}[t]{\dimexpr \vlabs + 3\vsm + 3\vgap\relax}\vspace{0pt}\raggedright
    \setlength{\tabcolsep}{0.5\vgap}\renewcommand{\arraystretch}{0}
    \begin{tabular}{@{}c ccc@{}}
       & {\normalsize AE} & {\normalsize VAE} & {\normalsize AE-PoS} \\[2pt]
       \vrl{$\hat{s}$} & \vSv{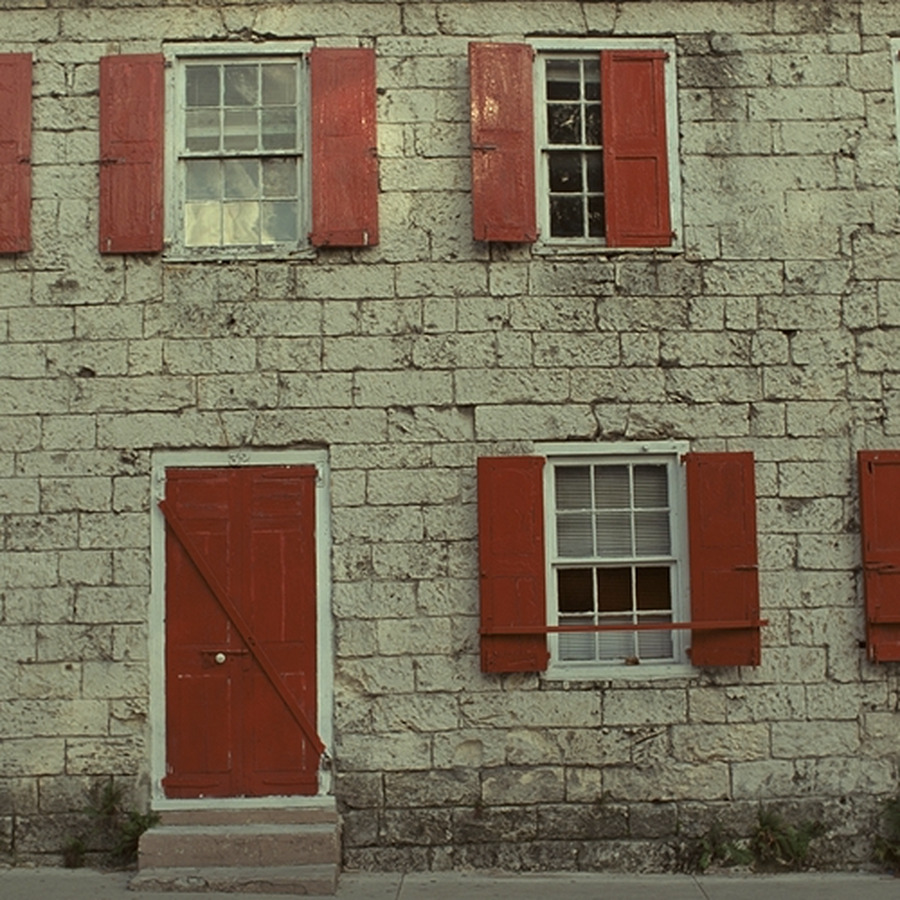}{58.3\,dB} & \vSv{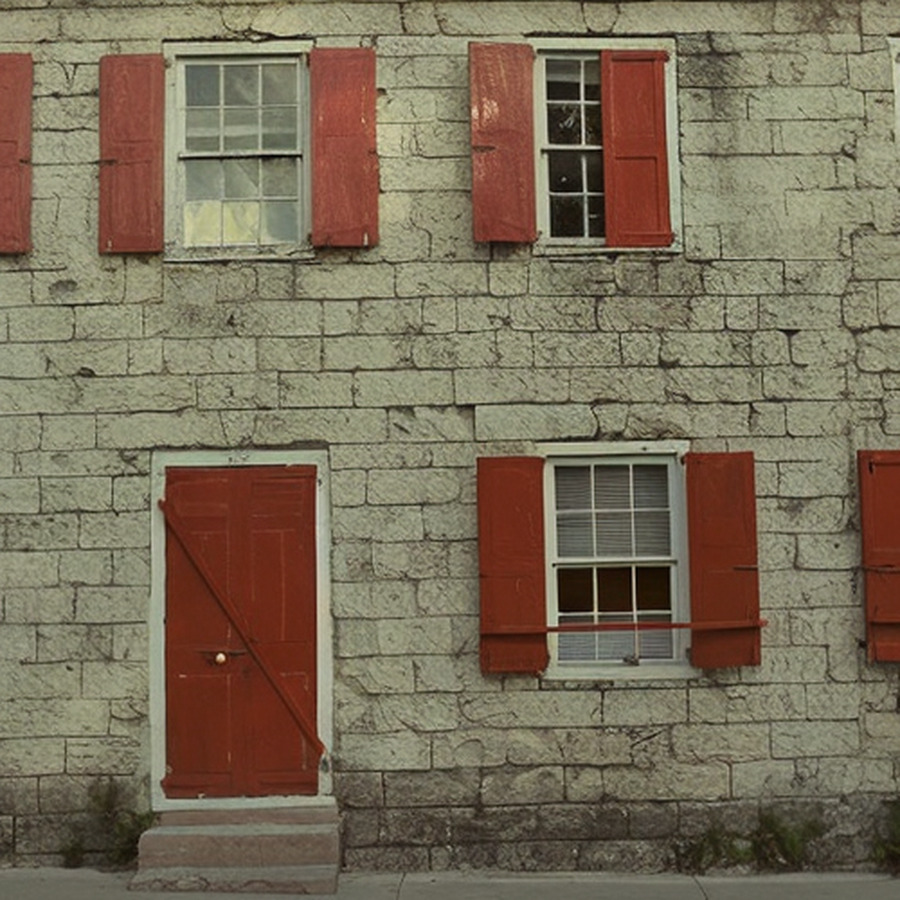}{23.6\,dB} & \vSv{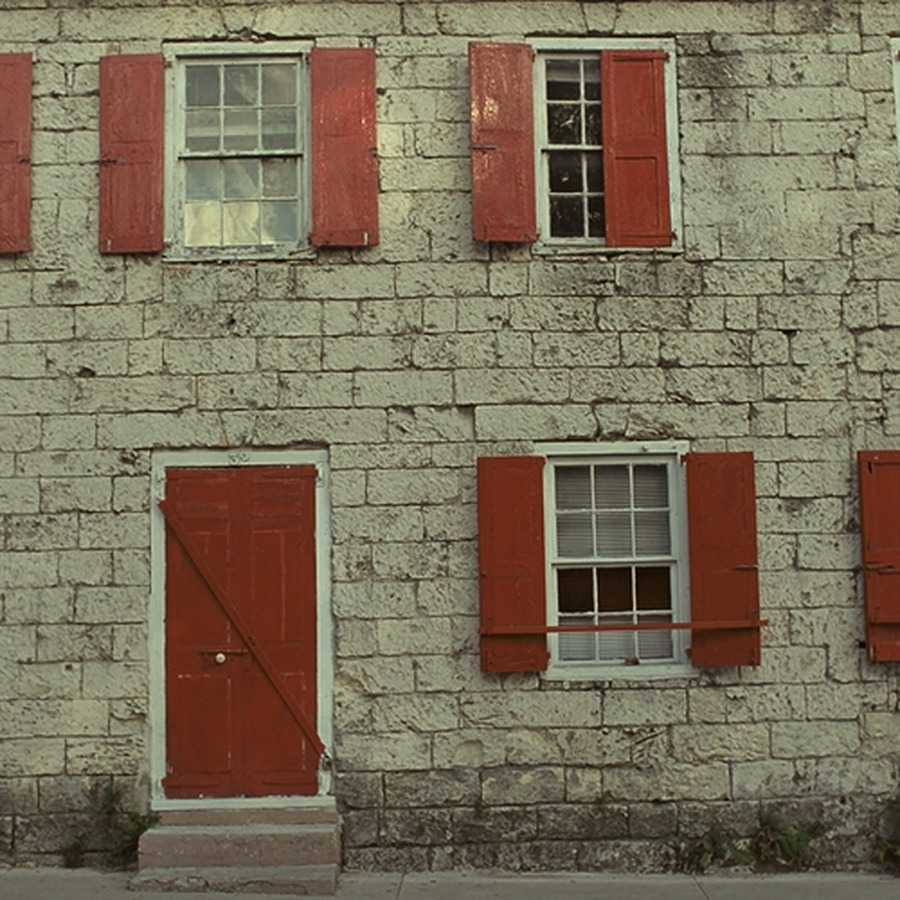}{46.0\,dB} \\[\vgap]
       \vrl{$s-\hat{s}$} & \vS{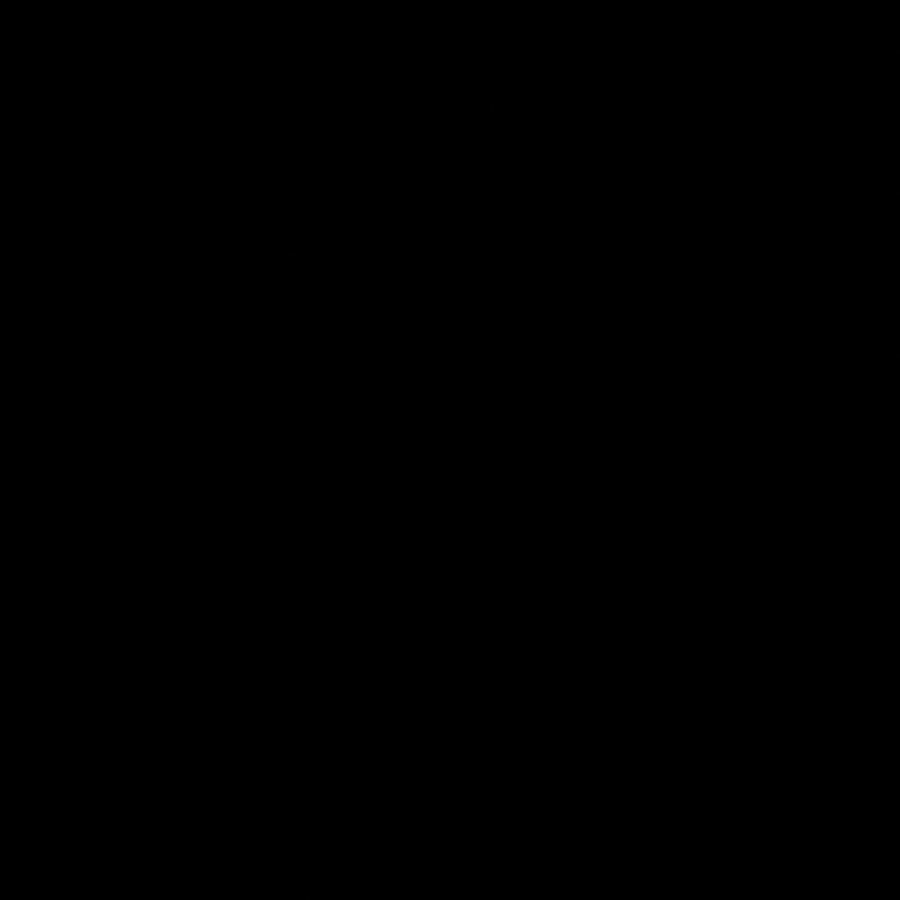} & \vS{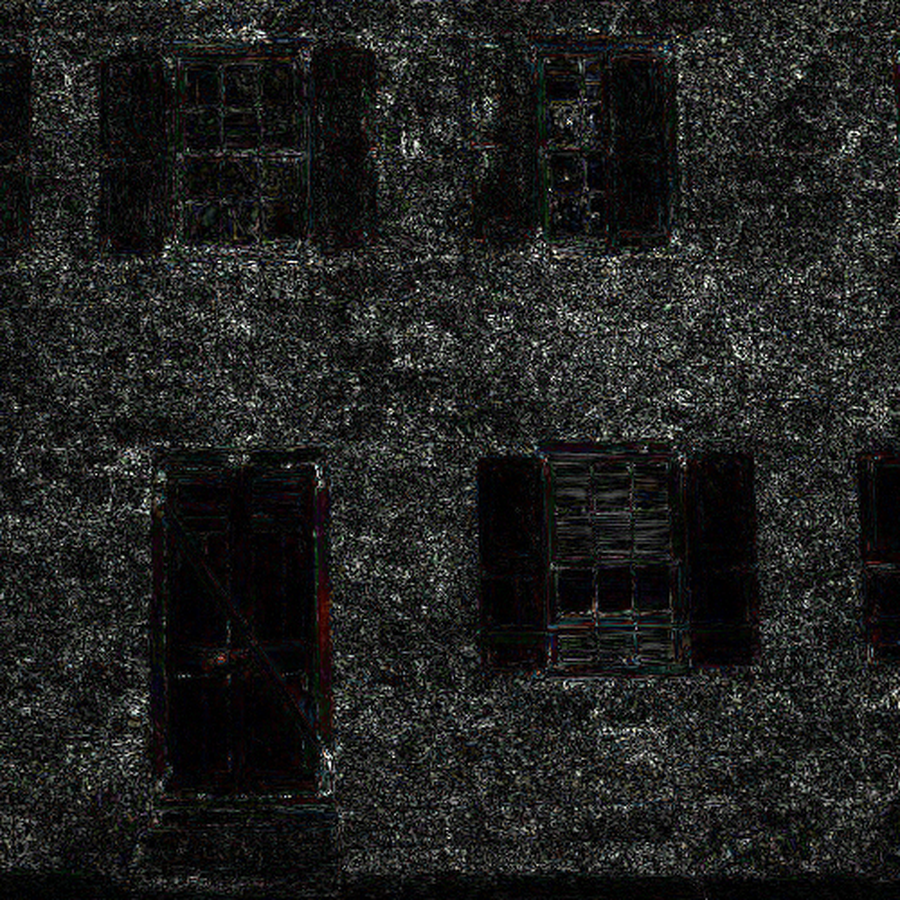} & \vS{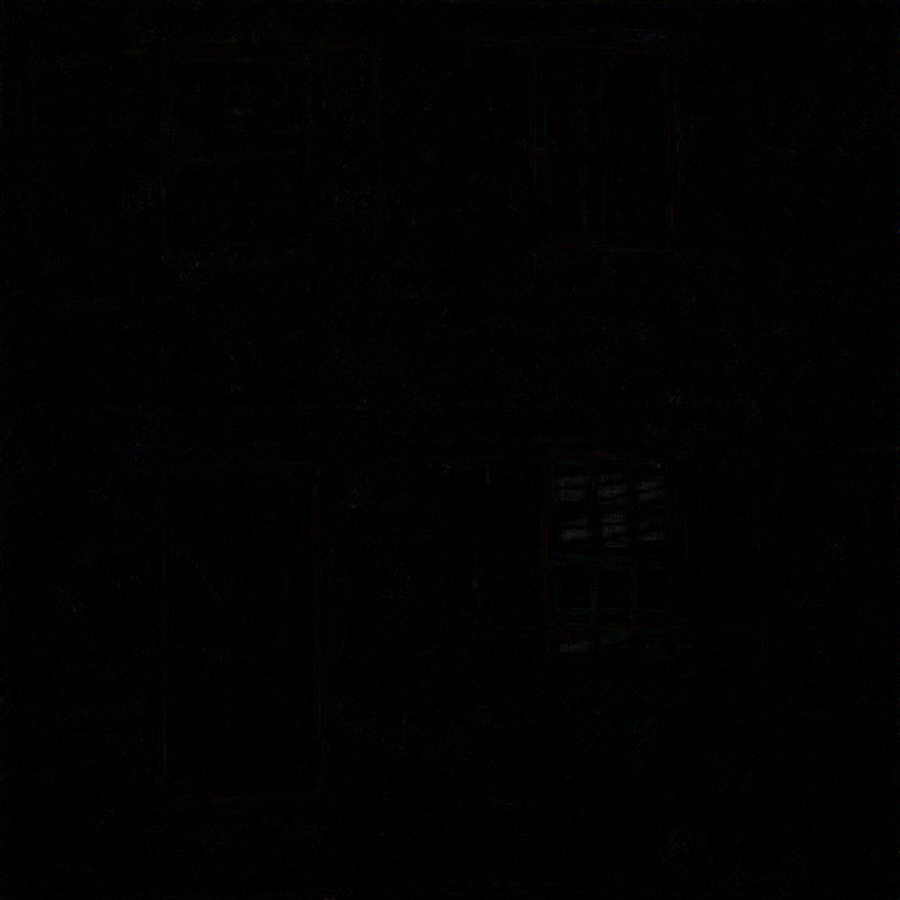} \\
    \end{tabular}
  \end{minipage}

  \vspace{2\vgap}

  \setlength{\vsm}{\dimexpr (\textwidth - \vlabt - \vlabs - 1pt - 5.7787\vgap) / 5\relax}
  \setlength{\vbg}{\dimexpr 2\vsm + 1.7787\vgap\relax}
  \setlength{\vsh}{0.5622\vsm}\setlength{\vbh}{0.5622\vbg}
  \begin{minipage}[t]{\dimexpr \vlabt + \vgap + \vbg\relax}\vspace{0pt}\raggedright
    \vrb{nominal}\hspace{\vgap}\vB{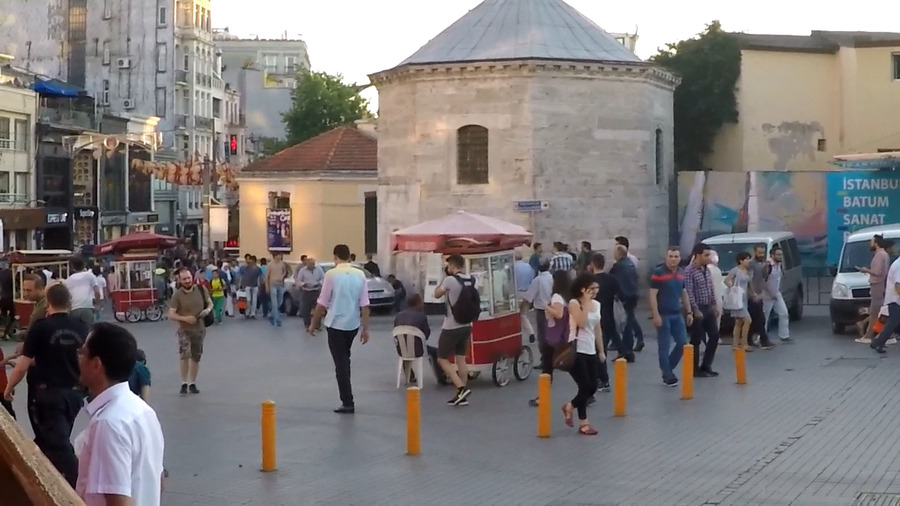}
  \end{minipage}\hfill
  \begin{minipage}[t]{\dimexpr \vlabs + 3\vsm + 3\vgap\relax}\vspace{0pt}\raggedright
    \setlength{\tabcolsep}{0.5\vgap}\renewcommand{\arraystretch}{0}
    \begin{tabular}{@{}c ccc@{}}
       \vrl{$\hat{s}$} & \vSv{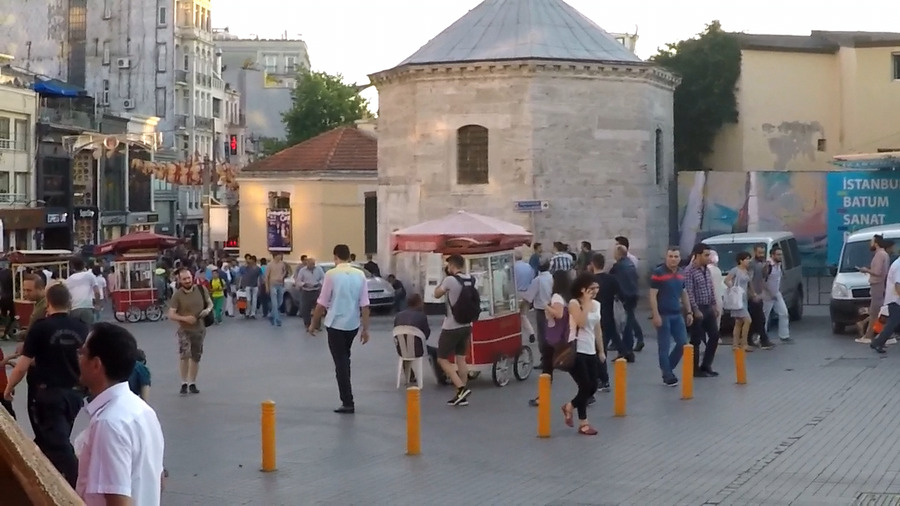}{63.3\,dB} & \vSv{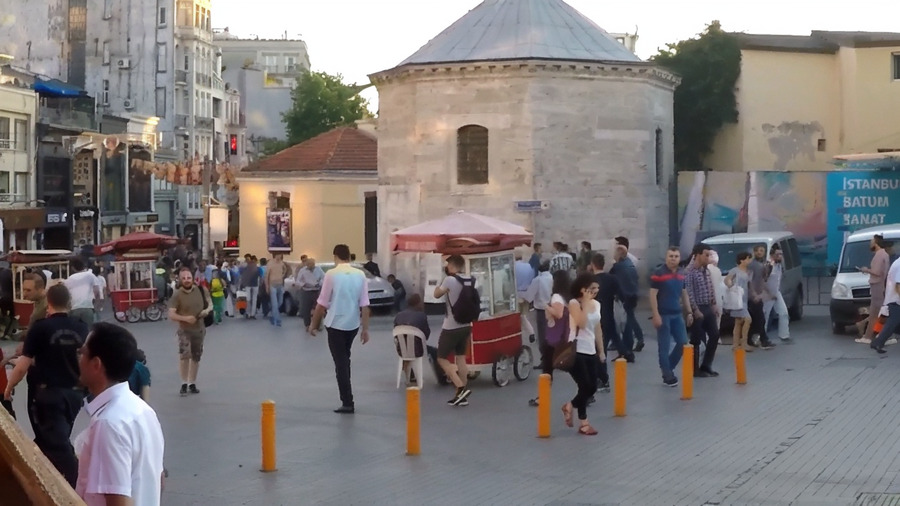}{31.1\,dB} & \vSv{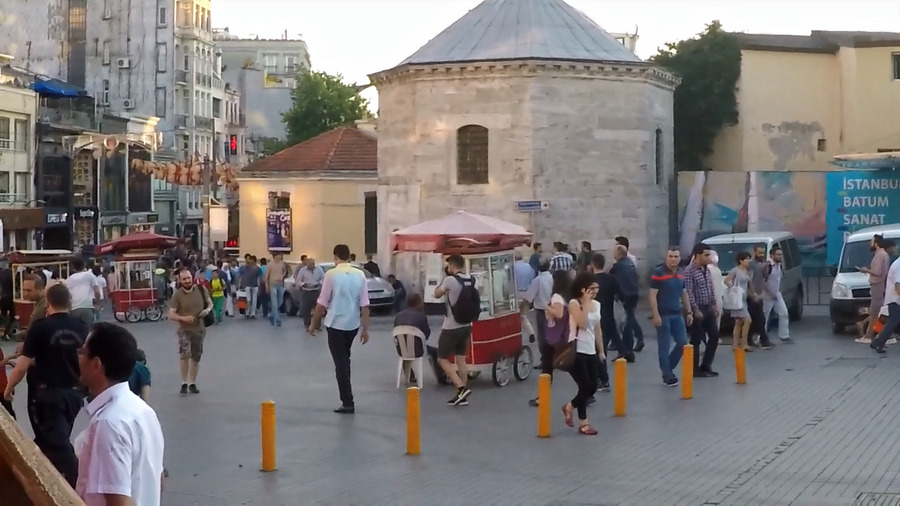}{37.8\,dB} \\[\vgap]
       \vrl{$s-\hat{s}$} & \vS{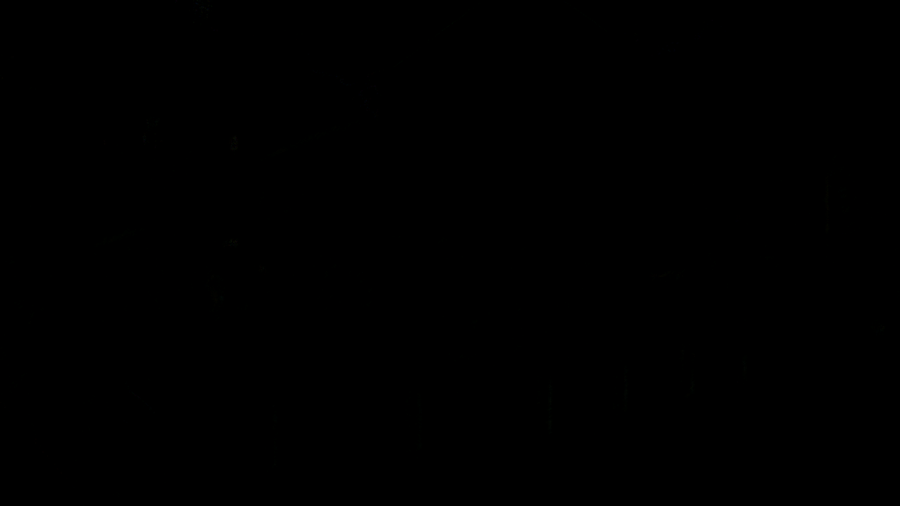} & \vS{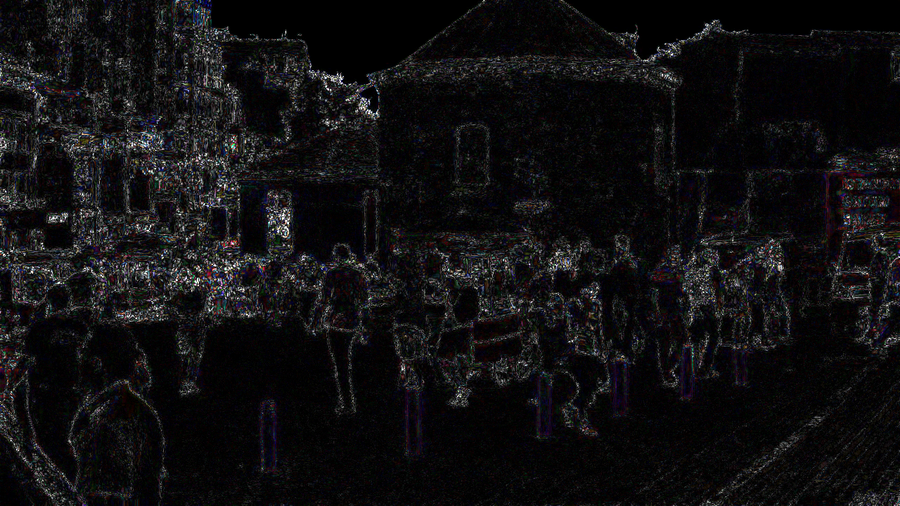} & \vS{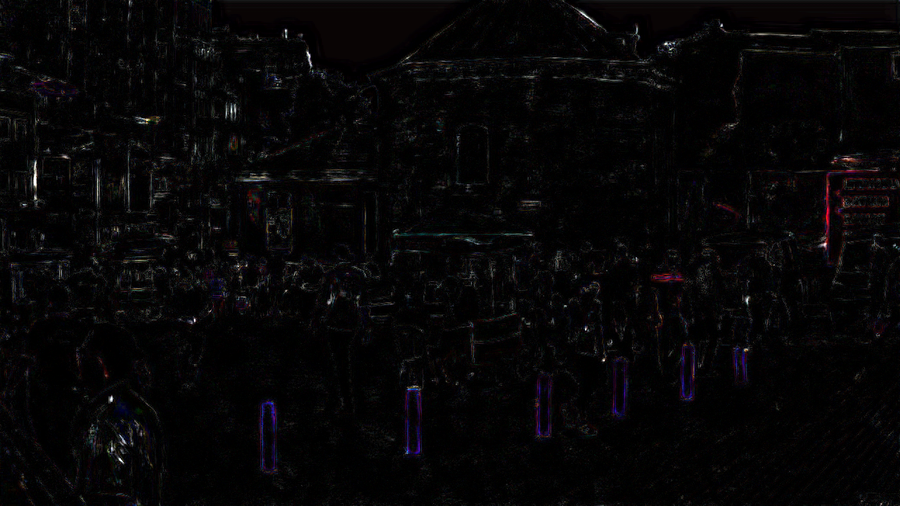} \\
    \end{tabular}
  \end{minipage}

  \vspace{2\vgap}

  \setlength{\vsm}{\dimexpr (\textwidth - \vlabt - \vlabs - 1pt - 5.7787\vgap) / 5\relax}
  \setlength{\vbg}{\dimexpr 2\vsm + 1.7787\vgap\relax}
  \setlength{\vsh}{0.5622\vsm}\setlength{\vbh}{0.5622\vbg}
  \begin{minipage}[t]{\dimexpr \vlabt + \vgap + \vbg\relax}\vspace{0pt}\raggedright
    \vrb{degraded}\hspace{\vgap}\vB{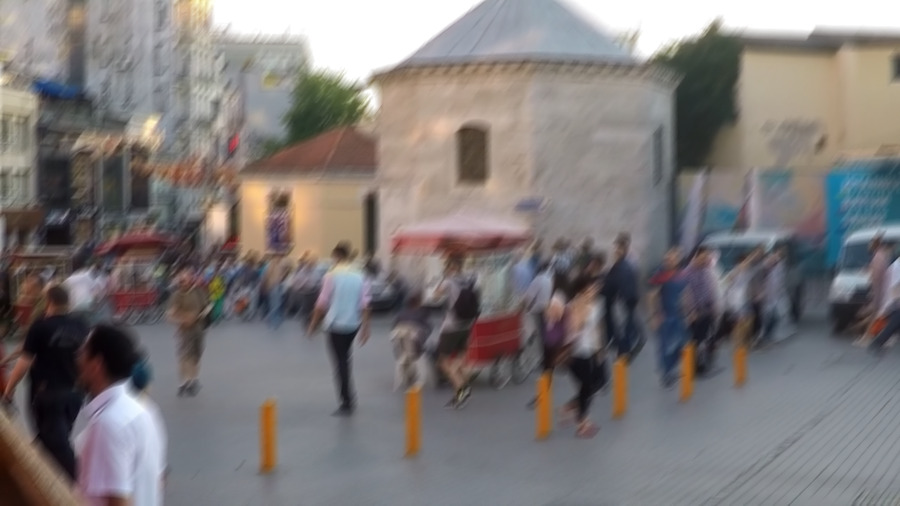}
  \end{minipage}\hfill
  \begin{minipage}[t]{\dimexpr \vlabs + 3\vsm + 3\vgap\relax}\vspace{0pt}\raggedright
    \setlength{\tabcolsep}{0.5\vgap}\renewcommand{\arraystretch}{0}
    \begin{tabular}{@{}c ccc@{}}
       \vrl{$\hat{s}$} & \vSv{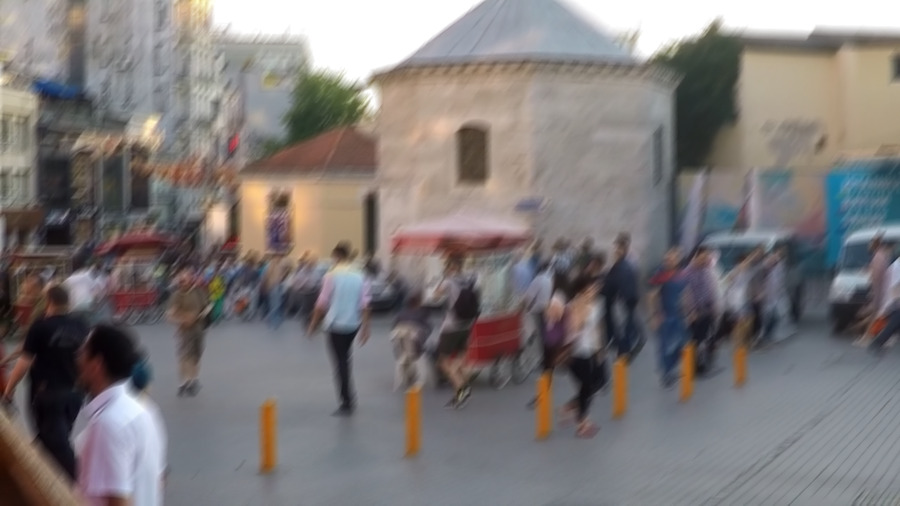}{63.7\,dB} & \vSv{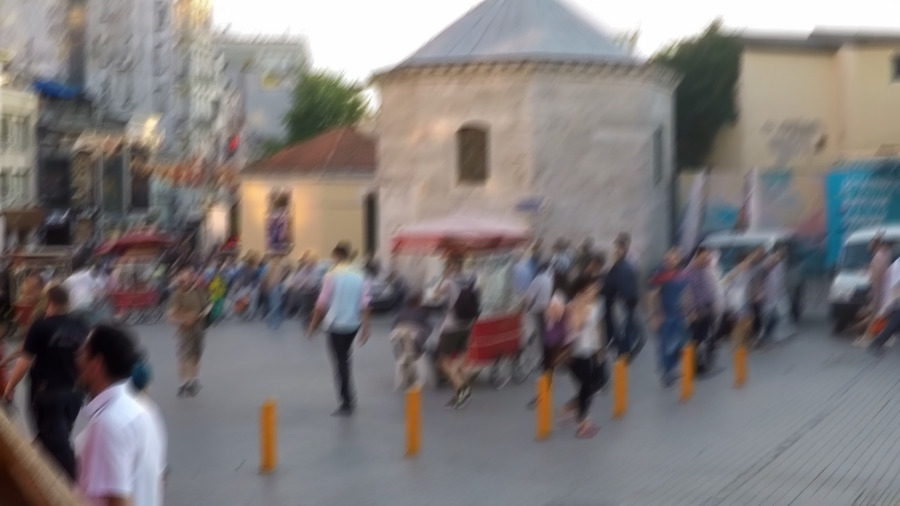}{40.1\,dB} & \vSv{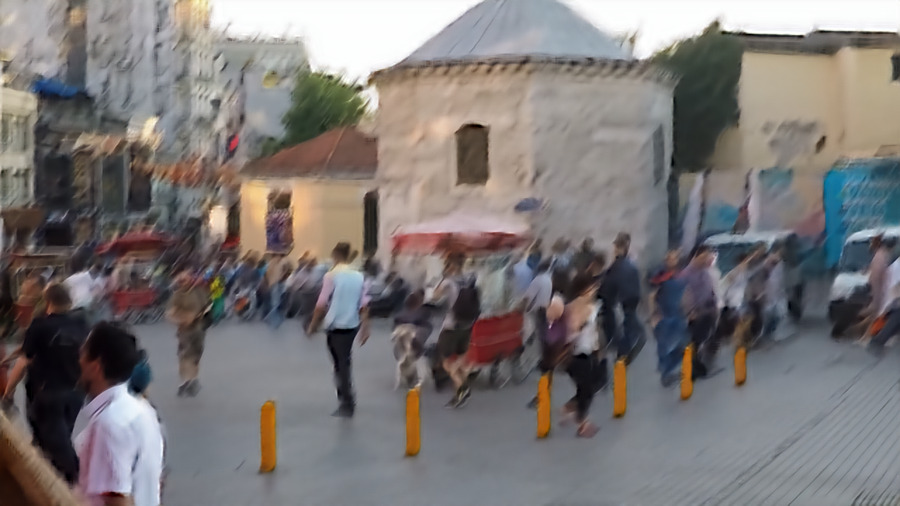}{27.5\,dB} \\[\vgap]
       \vrl{$s-\hat{s}$} & \vS{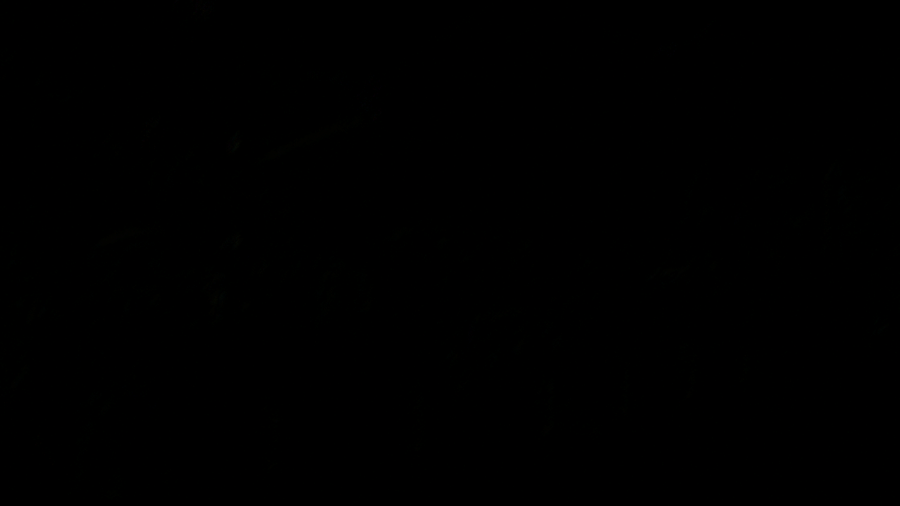} & \vS{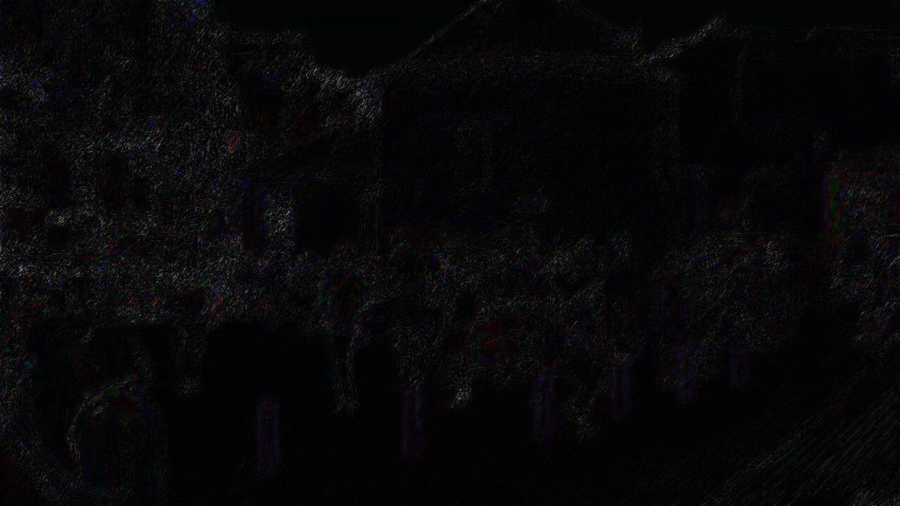} & \vS{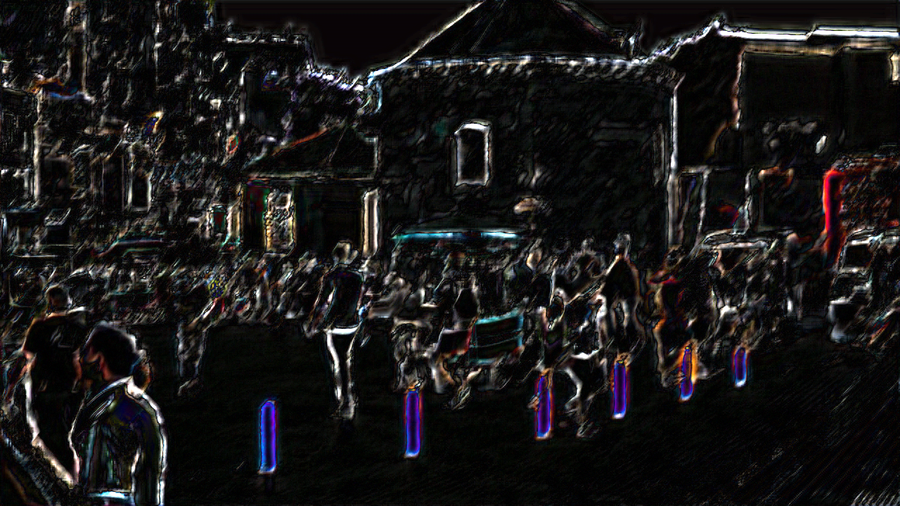} \\
    \end{tabular}
  \end{minipage}

  \vspace{2\vgap}

  \setlength{\vsm}{\dimexpr (\textwidth - \vlabt - \vlabs - 1pt - 5.3333\vgap) / 5\relax}
  \setlength{\vbg}{\dimexpr 2\vsm + 1.3333\vgap\relax}
  \setlength{\vsh}{0.7500\vsm}\setlength{\vbh}{0.7500\vbg}
  \begin{minipage}[t]{\dimexpr \vlabt + \vgap + \vbg\relax}\vspace{0pt}\raggedright
    \vrb{degraded}\hspace{\vgap}\vB{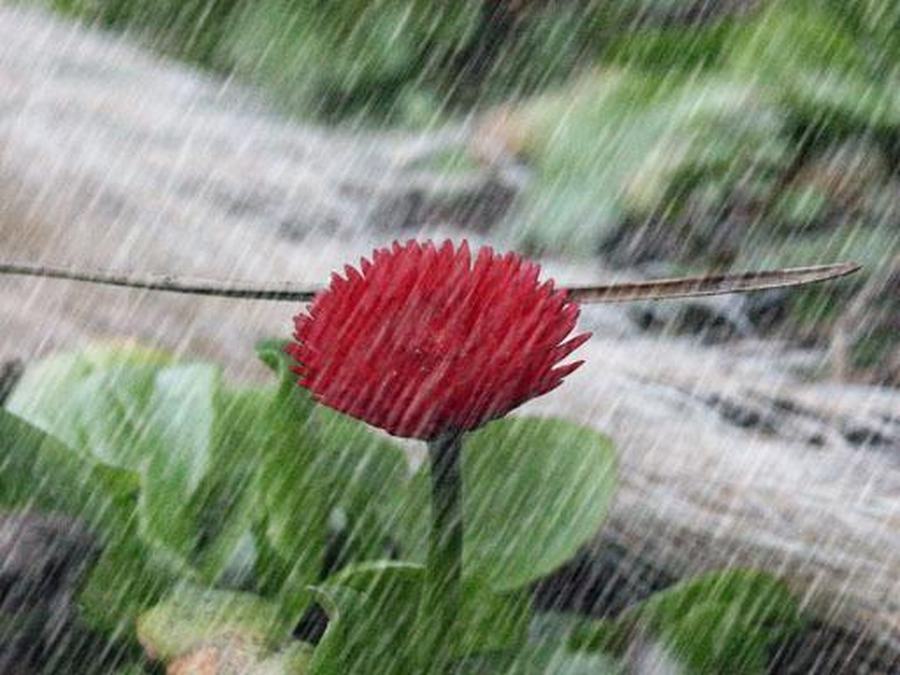}
  \end{minipage}\hfill
  \begin{minipage}[t]{\dimexpr \vlabs + 3\vsm + 3\vgap\relax}\vspace{0pt}\raggedright
    \setlength{\tabcolsep}{0.5\vgap}\renewcommand{\arraystretch}{0}
    \begin{tabular}{@{}c ccc@{}}
       \vrl{$\hat{s}$} & \vSv{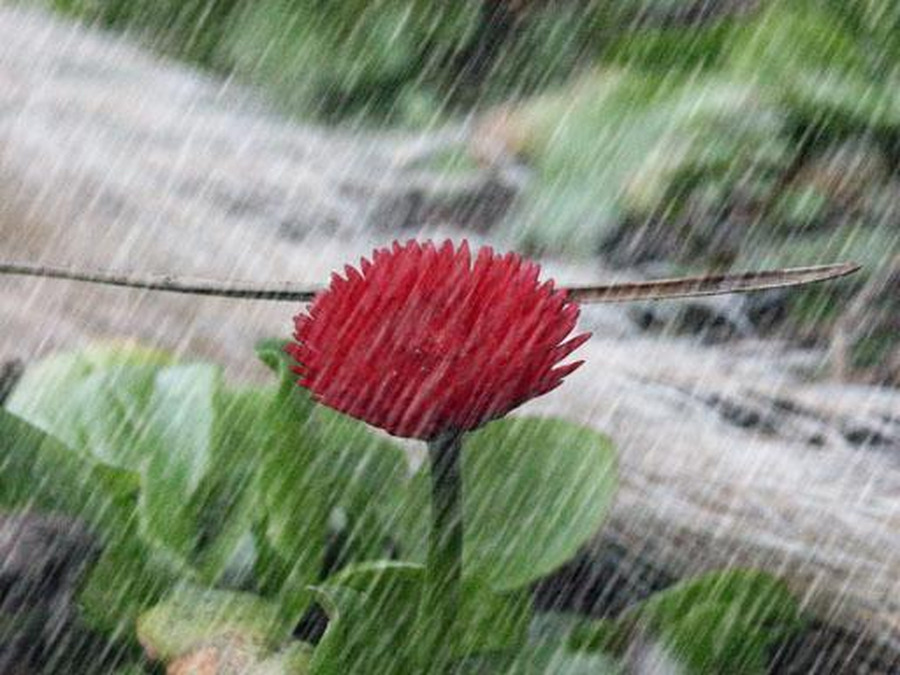}{61.7\,dB} & \vSv{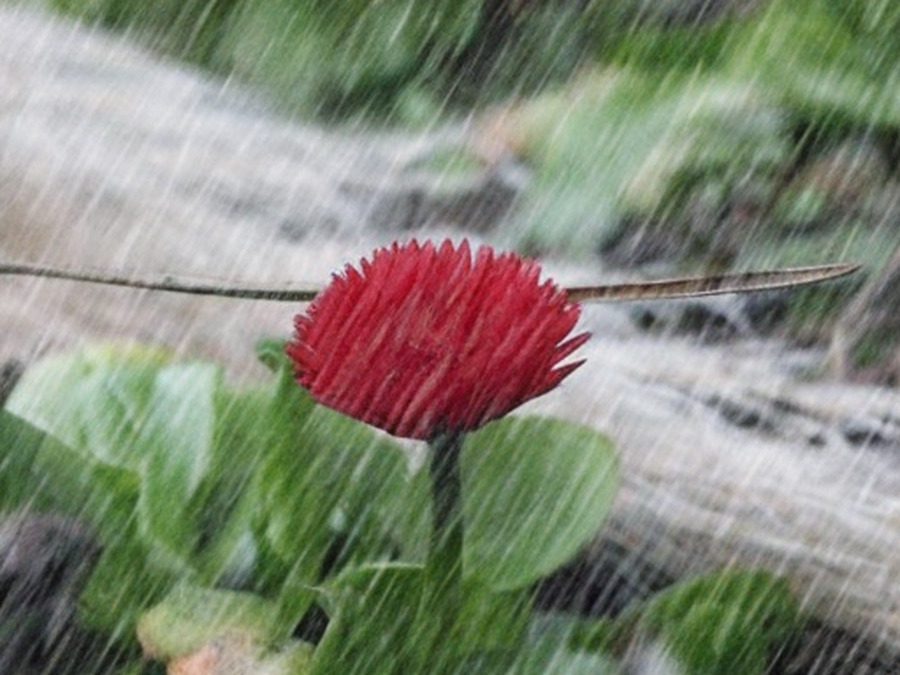}{29.0\,dB} & \vSv{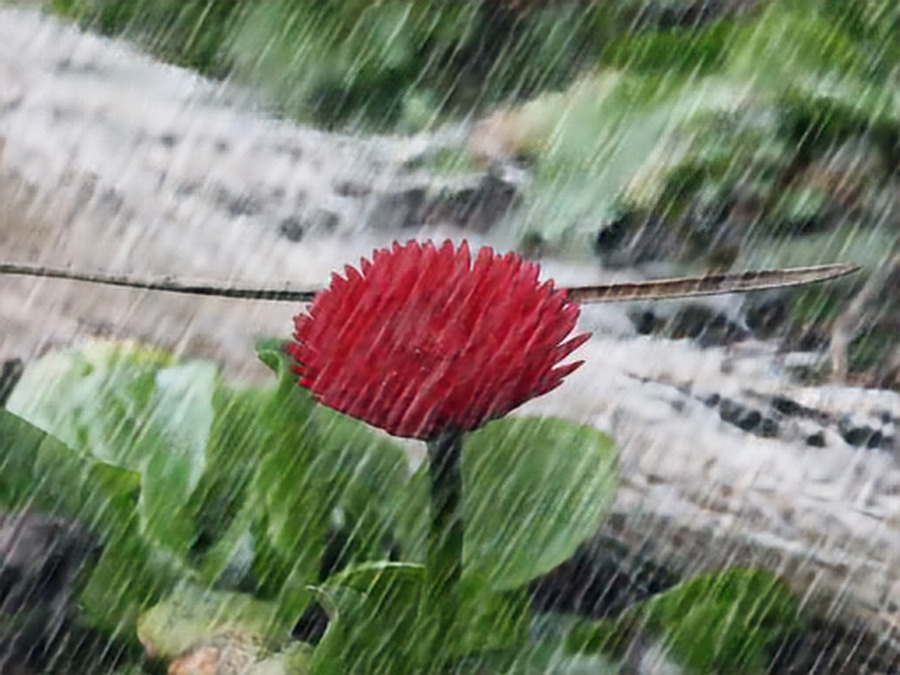}{25.1\,dB} \\[\vgap]
       \vrl{$s-\hat{s}$} & \vS{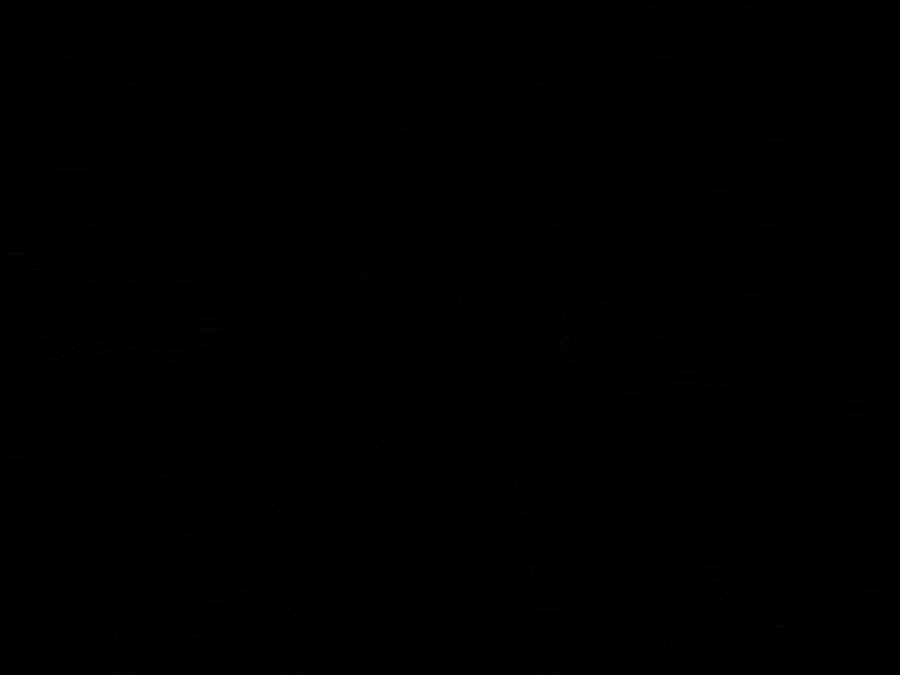} & \vS{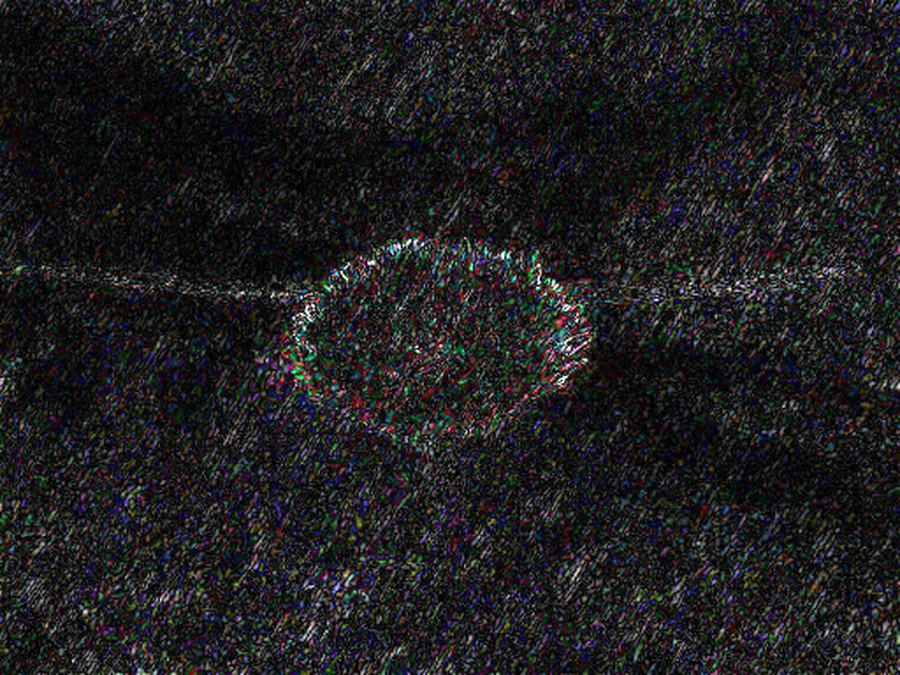} & \vS{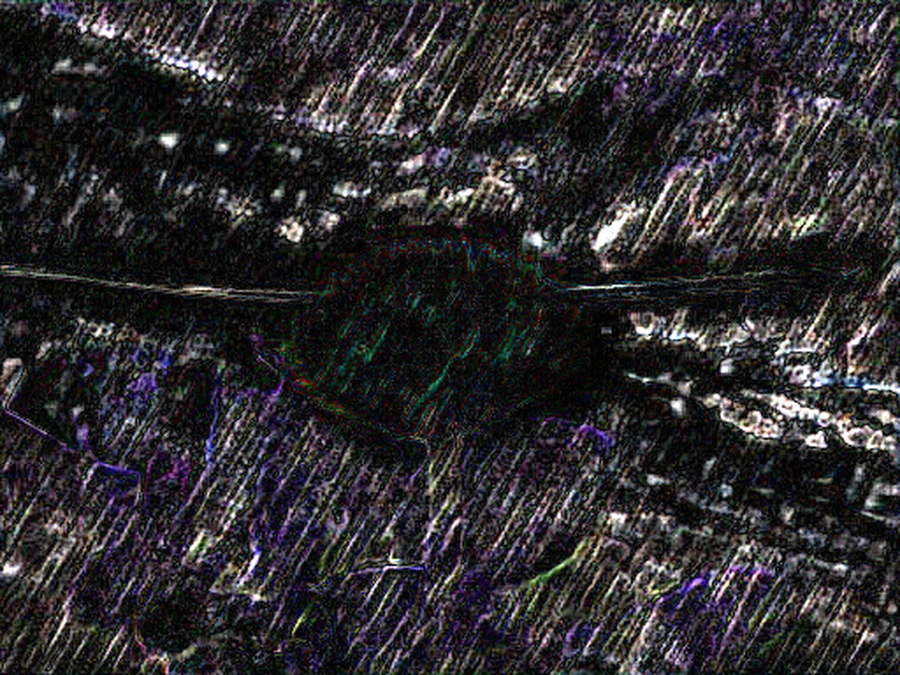} \\
    \end{tabular}
  \end{minipage}
\end{center}
\caption{\textbf{Nominal fidelity and degradation rejection.}
For each input \(s\), the figure shows the reconstruction \(\hat{s}\), its
PSNR, and the absolute residual \(\lvert s-\hat{s}\rvert\) for AE, VAE, and
AE-PoS. From top to bottom, the examples are a nominal Kodak24 image, a
nominal GoPro frame, its motion-blurred counterpart, and a rainy Test2800
image. AE reconstructs both nominal and degraded inputs almost identically,
providing high fidelity but weak anomaly separation. VAE loses substantial
detail on nominal images and represents the motion-blurred input more
accurately. AE-PoS preserves nominal image content while excluding motion
blur and rain from its learned target set, thereby exposing these degradations
in the residual. All residual panels use the same intensity gain.}
\label{fig:image-anomaly}
\end{figure}

%% file: appendices/ECG_Appendix.tex
\section{Transferable ECG Representations}
\label{apd:ecg}

\noindent\textbf{Task and baselines:}\quad
Clinical annotation of 12-lead ECG recordings requires expert interpretation, whereas large collections of unlabelled recordings are readily available. We therefore evaluate whether nested representation learning improves a self-supervised ECG encoder before downstream linear probing. We compare with the published results of ECG-CPC~\citep{AlMasud2025}, ECG-FM~\citep{McKeen2025ECGFMAO}, MERL~\citep{liuzeromerl}, and ST-MEM~\citep{na2024guiding}.

\noindent\textbf{Depth-controlled architecture:}\quad
ST-MEM learns ECG representations by reconstructing masked portions of 12-lead recordings. For a capacity-controlled comparison, Stage~0 follows the ST-MEM design but reduces its encoder from 12 to six Transformer blocks, while retaining the original four-block decoder and standard masking ratio of \(75\%\).
After pretraining, the Stage-0 encoder is frozen, and Stage~1 applies a second six-block-encoder, four-block-decoder network of the same form to its latent representation using masked latent reconstruction. The resulting nested encoder therefore has a total depth of 12 blocks, matching the original ST-MEM
backbone while distributing learning across two successive stages.

\noindent\textbf{Pretraining data:}\quad
Both stages use the same unlabelled pretraining pool: 34,808 recordings from Ningbo~\citep{PhysioNet-ecg-arrhythmia,pollard_physionet_2026,Zheng2020Optimal} and 130,164 recordings from CODE-15~\citep{Ribeiro2021CODE15AL}, giving 164,972 recordings in total. Each signal is high-pass filtered at \(0.67\,\mathrm{Hz}\), low-pass filtered
at \(40\,\mathrm{Hz}\), standardized independently, resampled to
\(250\,\mathrm{Hz}\), and cropped to 2250 samples. Stage~1 is trained only on representations of this same pretraining pool, so no additional ECG data are introduced after Stage~0.

\noindent\textbf{Pretraining:}\quad
For both stages, we use AdamW with learning rate \(3.75\times10^{-5}\), weight decay \(0.05\), and \((\beta_1,\beta_2)=(0.9,0.95)\). Training uses a batch size of \(64\), gradient clipping at norm \(1.0\), and \(150{,}000\) update steps. The learning rate is warmed up linearly for \(15{,}000\) steps and then follows cosine decay to \(10^{-6}\).

\noindent\textbf{Downstream protocol:}\quad
We use the data splits and frozen-encoder linear-probe protocol of \citet{liuzeromerl}. All pretrained parameters are frozen, and only a single linear classification layer is optimized using binary cross-entropy and Adam with learning rate \(10^{-3}\), weight decay \(10^{-4}\), batch size \(16\), and \(100\) training epochs. The learning rate is multiplied by \(0.1\) after epoch \(40\). The checkpoint with the highest validation AUROC is selected, and its test AUROC is reported. Table~\ref{tab:ecg-nested} uses the complete labelled training split and reports the mean over five seeds. The downstream tasks are the PTB-XL superclass, subclass, form, and rhythm tasks~\citep{wagner2020ptb}, together with ICBEB/CPSC-2018~\citep{PhysioNet-challenge2020}.

\noindent\textbf{Results:}\quad
As shown in Table~\ref{tab:ecg-nested}, Stage~1 improves upon Stage~0 on all five downstream tasks and obtains the highest AUROC in every column among the compared methods. Because the two-stage encoder has the same total depth as the original eight-layer ST-MEM backbone, this comparison indicates that the gain arises from successive refinement of the frozen Stage-0 representation rather than from increasing total encoder depth alone.

%% file: appendices/ExtendedLimitations.tex
\section{Expanded Limitations and Future Directions}
\label{apd:limitations}

The first limitation is that the geometry is defined relative to the reconstruction metric \(d\), which we fix to pixel \(\ell_2\) distance to isolate the representation-learning question. Consequently, anomalies that remain close to the nominal union under this metric can be difficult to distinguish. Digit 9 illustrates this limitation in the \(\{4,5,7\}\) experiment, where \(87\%\) of its samples are detected by neither score, while the gap between AE-PoS and transformation-based methods on CIFAR-100 indicates that this issue becomes more pronounced for semantically complex data. Learning task-appropriate metrics or deriving anomaly scores directly from the multilevel geometry, while preserving the metric conditions required by Lemma 1, is therefore an important direction for future work. The second limitation is that the number of nested carving levels is currently
selected manually. The nested range relation suggests a label free stopping
criterion based on the fraction of perturbed held out nominal samples whose
reconstructions continue to violate the Push margin at level \(\ell\). Whether
this quantity can reliably determine when the current range requires further
carving, without access to anomaly examples, remains the main open question
for automatically constructing the nested architecture. The present analysis characterizes representation dynamics relative to nominal data, for which nonlinear orthogonal projection onto the learned union \(\mathcal M\) provides the ideal representation. For anomalous inputs outside the reach of \(\mathcal M\), the nearest nominal point need not be unique, and direct projection onto \(\mathcal M\) therefore does not generally define an ideal representation. We conjecture that important families of structured anomalies, such as specific ECG abnormalities, may lie on or near isometric images \(\{\Phi_\gamma(\mathcal M)\}_{\gamma\in\Gamma}\) of the nominal union. Future work will investigate conditions under which an identifiable inverse transformation \(\Phi_\gamma^{-1}\) maps each anomalous input back to or within the reach of \(\mathcal M\), so that \(P_{\mathcal M}\circ\Phi_\gamma^{-1}\) provides a unique canonical representation and \(\Phi_\gamma\circ P_{\mathcal M}\circ\Phi_\gamma^{-1}\) realizes nonlinear orthogonal projection onto the corresponding structured anomalous manifold \(\Phi_\gamma(\mathcal M)\).